\pdfoutput=1

\PassOptionsToPackage{prologue,dvipsnames}{xcolor} 

\documentclass[11pt]{article}

\usepackage[]{ACL_template/acl}

\usepackage{times}
\usepackage{latexsym}

\usepackage[T1]{fontenc}
\usepackage{CJKutf8}

\usepackage[utf8]{inputenc}

\usepackage{microtype}

\usepackage{inconsolata}

\usepackage{algorithm}
\usepackage{algpseudocode}
\usepackage[inline]{enumitem}
\usepackage{algorithm}
\usepackage{algpseudocode}
\algrenewcommand\algorithmicrequire{\textbf{Input:}}
\algrenewcommand\algorithmicensure{\textbf{Output:}}
\usepackage{amsmath,amssymb,amsfonts,amsthm}
\newcommand{\bigcircledast}{\mathop{\circledast}\limits}
\newtheorem{proposition}{Proposition}
\newtheorem{definition}{Definition}
\newtheorem{lemma}{Lemma}

\usepackage{amsmath,amssymb,amsfonts,amsthm}
\usepackage{booktabs,multirow}
\usepackage{algpseudocode}
\usepackage{enumitem}
\usepackage{mathtools}

\newtheorem{theorem}{Theorem}

\newtheorem{corollary}{Corollary}
\newtheorem{remark}{Remark}

\DeclareMathOperator*{\clip}{clip}

\DeclareMathOperator{\tr}{tr}

\usepackage{multirow}

\usepackage{listings}
\usepackage{tabularx}
\usepackage{ltablex}
\usepackage{longtable}
\usepackage{graphicx}
\usepackage{colortbl}
\usepackage[dvipsnames]{xcolor}
\usepackage[table]{xcolor} 
\usepackage{fontawesome5}
\usepackage{pythonhighlight}

\usepackage{supertabular}
\usepackage{xltabular}
\usepackage{amsmath}
\usepackage{amssymb}
\usepackage{xspace}
\usepackage{mathrsfs}
\usepackage{bm,dutchcal}
\usepackage{float}
\usepackage{booktabs} 
\usepackage{makecell} 
\usepackage{amsmath} 
\usepackage{booktabs}
\usepackage{graphicx}
\usepackage{tcolorbox}
\tcbuselibrary{breakable, skins, listings}

\usepackage{marvosym}

\usepackage{adjustbox}
\usepackage{lineno}
\usepackage{array}
\usepackage{booktabs} 
\usepackage{makecell}

\usepackage[acronym]{glossaries}

\newacronym{3DVG}{3DVG}{3D Visual Grounding}
\newacronym{LLM}{LLM}{Large Language Model}
\newacronym{VLM}{VLM}{Vision Language Model}
\newacronym{LGSP}{LGSP}{Language-Guided Spatial Pruning}
\newacronym{MCDR}{MCDR}{Multi-View-Conditioned Description Reformulation}
\newacronym{LLM-Grounder}{LLM-Grounder}{Fine-tuned LLM-based Grounder}
\newacronym{MLP}{MLP}{Multilayer Perceptron}

\definecolor{custom_light_blue}{rgb}{0.85, 0.95, 1}
\definecolor{custom_light_pink}{rgb}{1, 0.85, 0.85}
\definecolor{custom_light_purple}{rgb}{0.98, 0.91, 0.973}
\definecolor{custom_light_purple_2}{rgb}{0.98, 0.91, 0.953}
\definecolor{custom_darkgreen}{rgb}{0.0, 0.5, 0.0}

\definecolor{custom_purple}{RGB}{180,114,200}   
\definecolor{custom_pink}{RGB}{230,150,190} 
\definecolor{custom_red}{RGB}{225, 30, 86}

\definecolor{gtred}{RGB}{220,20,60} 
\definecolor{m2mblue}{RGB}{0,102,204}

\usepackage[table]{xcolor}
\usepackage{worldflags}
\definecolor{softgrey}{gray}{0.85}

\definecolor{bg}{RGB}{248,248,255}
\definecolor{frame}{RGB}{80,80,180}
\definecolor{titlebg}{RGB}{60,60,150}
\definecolor{string}{RGB}{196,26,22}
\definecolor{keyword}{RGB}{0,102,204}
\definecolor{comment}{RGB}{0,140,0}
\definecolor{number}{RGB}{160,32,240}

\lstdefinestyle{mypython}{
    basicstyle=\ttfamily\footnotesize,
    keywordstyle=\color{keyword}\bfseries,
    stringstyle=\color{string},
    commentstyle=\color{comment}\itshape,
    numberstyle=\tiny\color{gray},
    numbers=left,
    stepnumber=1,
    numbersep=8pt,
    showstringspaces=false,
    breaklines=true,
    frame=none,
    columns=fullflexible
}

\DeclareRobustCommand{\RaptorTitle}{%
  \begingroup\normalfont
  \raisebox{-0.2em}{%
    \includegraphics[height=3em]{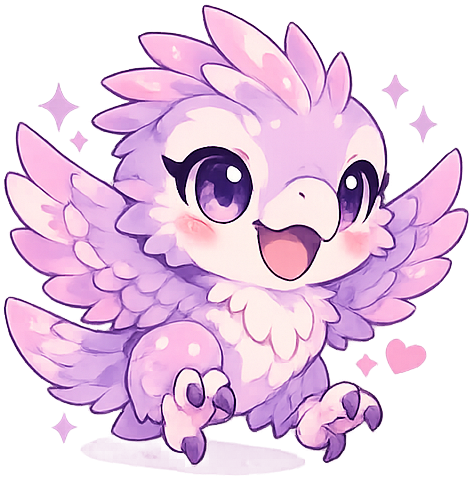}%
  }%
  \kern 0.0em 
  \textbf{\textcolor{custom_purple}{RAPT}\textcolor{custom_pink}{OR}}%
  \endgroup
}

\DeclareRobustCommand{\Raptor}{%
  \begingroup\normalfont
  \raisebox{-0.2em}{%
    \includegraphics[height=1.2em]{images/RAPTOR_icon.png}%
  }%
  \kern 0.0em 
  \textbf{\textcolor{custom_purple}{RAPT}\textcolor{custom_pink}{OR}}%
  \endgroup
}

\title{\RaptorTitle: \textcolor{custom_purple}{R}\textcolor{custom_pink}{ole-}\textcolor{custom_purple}{A}\textcolor{custom_pink}{ware} \textcolor{custom_purple}{P}\textcolor{custom_pink}{rivate} \textcolor{custom_purple}{T}\textcolor{custom_pink}{raining} \textcolor{custom_pink}{for Mixture-of-Experts}}
\author{Duc Dm$^{\text{\textcolor{custom_purple}{\faStar}} \;1}$,
Khai Le-Duc$^{\text{\textcolor{custom_purple}{\faStar}}\;2,3,4}$,
Nguyen Do$^{\text{\textcolor{custom_purple}{\faStar}}\;5}$, 
\\{\bf Minh Son Hoang$^{1}$, Florent Draye$^{6}$, Thai Hoang$^{7}$, Hoang Phuong Dam$^{1}$, Jiarui Liu$^{8}$, }
\\{\bf Chris Ngo$^{4}$, Terry Jingchen Zhang$^{9,10}$, Anh Le Duc Tran$^{11}$, Nhat Do Minh$^{12}$, Minh Ngoc Le$^{2, 3}$}
\\{\bf My T. Thai$^{5}$, Ran Xu$^{7}$, Silvio Savarese$^{7,13}$, Mona Diab$^{8}$,}
\\{\bf Bernhard Sch\"{o}lkopf$^{6,14}$, Zhijing Jin$^{6,10}$, Huy L. Nguyen$^{15}$, Daeyoung Kim$^{1}$}\\
$^1$ \resizebox{!}{0.7em}{\worldflag{KR}} KAIST
$^2$ \resizebox{!}{0.7em}{\worldflag{CA}} University of Toronto
$^3$ \resizebox{!}{0.7em}{\worldflag{CA}} Vector Institute
$^4$ \resizebox{!}{0.7em}{\worldflag{SG}} Knovel Engineering Lab\\
$^5$ \resizebox{!}{0.7em}{\worldflag{US}} University of Florida
$^6$ \resizebox{!}{0.7em}{\worldflag{DE}} MPI for Intelligent Systems, T{\"u}bingen
\\ 
$^7$ \resizebox{!}{0.7em}{\worldflag{US}} Salesforce AI Research
$^8$ \resizebox{!}{0.7em}{\worldflag{US}} Carnegie Mellon University\\
$^{9}$ \resizebox{!}{0.7em}{\worldflag{GB}} University of Oxford
$^{10}$ \resizebox{!}{0.7em}{\worldflag{CA}} Jinesis Lab, University of Toronto \& Vector Institute\\
$^{11}$ \resizebox{!}{0.7em}{\worldflag{VN}} Hanoi University of Science and Technology
$^{12}$ \resizebox{!}{0.7em}{\worldflag{VN}} Vietnam National University, Hanoi\\
$^{13}$ \resizebox{!}{0.7em}{\worldflag{US}} Stanford University
$^{14}$ \resizebox{!}{0.7em}{\worldflag{DE}} ELLIS Institute T{\"u}bingen
$^{15}$ \resizebox{!}{0.7em}{\worldflag{US}} Northeastern University 
\\$^{\text{\textcolor{custom_purple}{\faStar}}}$Co-first authors \quad
\textcolor{custom_purple}{\faEnvelope}  \texttt{ducdm200158@kaist.ac.kr } 
\textcolor{custom_purple}{\faEnvelope} \texttt{duckhai.le@mail.utoronto.ca}
\\ \Large {\faGithubSquare}   \href{https://github.com/leduckhai/RAPTOR}{leduckhai/\textbf{\textcolor{custom_purple}{RAPT}\textcolor{custom_pink}{OR}}}
}

\begin{document}
\maketitle
\begin{abstract}
Differentially private (DP) fine-tuning methods treat sparse
Mixture-of-Experts (MoE) models as a single dense block, ignoring
that shared layers see all data while experts only see routed
records. We identify and formally characterize three resulting
failure modes: global clipping suppresses expert gradients,
batch-level normalization dilutes sparse expert updates, and fixed
privacy noise degrades signal-to-noise ratio on low-load experts.
We introduce \Raptor - a \textbf{Role-Aware Private Training} framework, which alternates
shared and expert optimization and targets each failure directly,
using expert-specific clipping and noise together with a public
expected-owner denominator and a count-independent update schedule
that avoids conditioning on private, realized expert counts. We
prove the resulting mechanism satisfies $(\varepsilon,\delta)$-DP:
because each record is assigned to exactly one owner expert,
per-expert mechanisms within a layer compose in parallel, so
updating all $E$ experts costs no more, in privacy terms, than
updating one, with shared and expert streams composing sequentially
across training. We further derive a bias-variance decomposition of
the public-denominator estimator showing its bias grows predictably
with routing imbalance, yielding a privacy-free rule for selecting
which layer to protect from routing entropy measured on a small
public corpus. Experiments on Switch Transformer and OLMoE
fine-tuning across GLUE tasks, and on the DeepSeek-VL2-Tiny, show consistent gains over standard DP
baselines across several privacy levels ($\varepsilon$), with the largest margins typically at the tightest budgets.
\end{abstract}

\thispagestyle{plain}
\pagestyle{plain}

\addtocontents{toc}{\protect\setcounter{tocdepth}{-1}}

\begin{figure*}[t]
    \centering
    \includegraphics[width=0.7\textwidth]{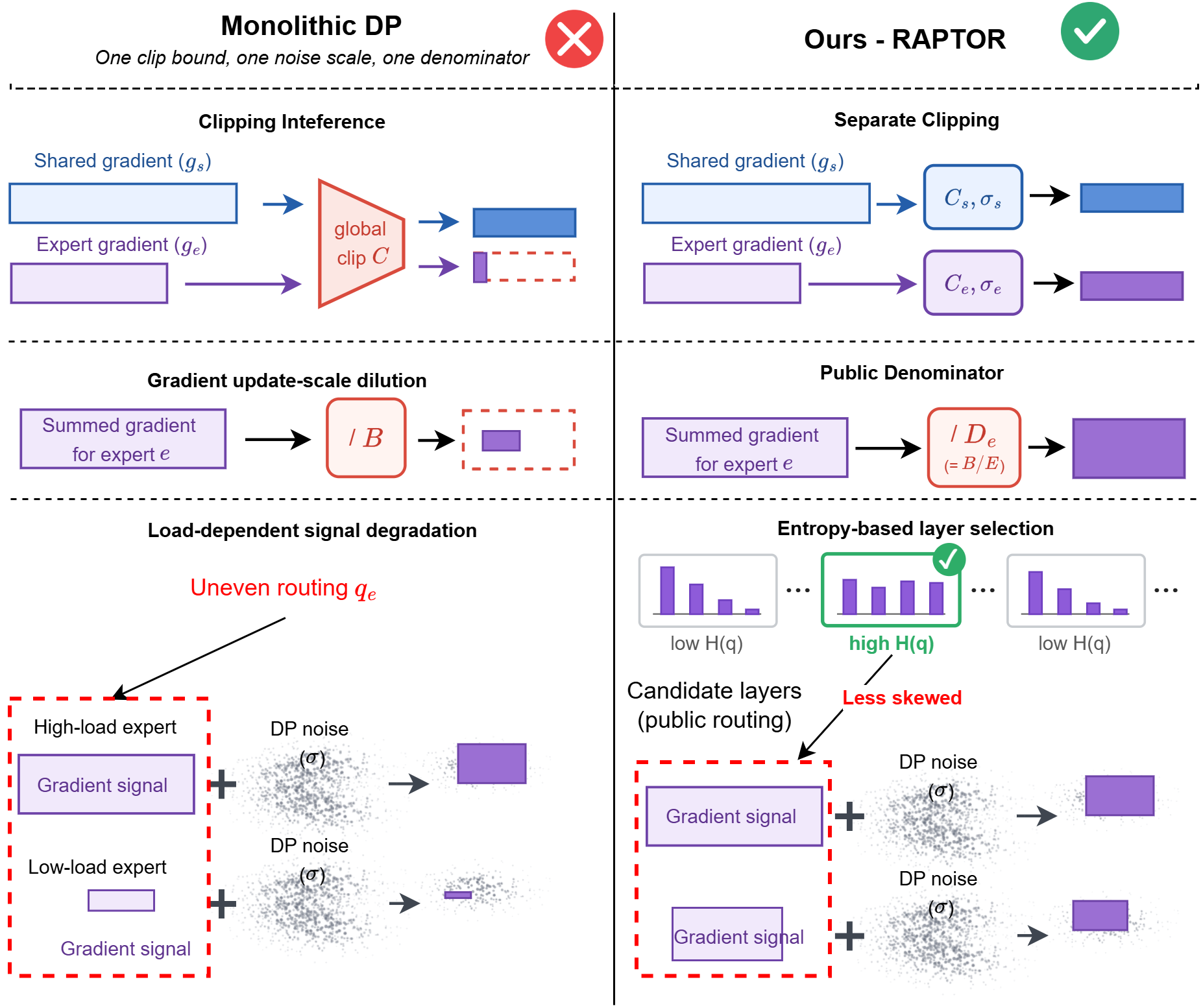}
    \caption{
        \textbf{Applying differential privacy to sparse MoE is not straightforward}:
        treating the model as dense leads to three structural failures -
        clipping interference, update-scale dilution, and load-dependent
        signal degradation. Our role-aware
        method is designed to
        address these mismatches.
    }
    \label{fig:teaser}
\end{figure*}

\section{Introduction}
Sparse Mixture-of-Experts (MoE) architectures decouple model capacity
from per-record computation, making conditional sparsity a central
strategy for scaling foundation
models~\cite{shazeer2017outrageously,fedus2022switch,jiang2024mixtral}. 
Yet many of their highest-value adaptations - on medical records,
enterprise documents, personal interactions, and other proprietary
corpora - require formal protection against memorization, for which
differential privacy (DP)~\cite{dwork2014algorithmic}, enforced via
per-record gradient clipping and Gaussian
noise~\cite{abadi2016deep}, is the prevailing standard. Combining the
two is not a straightforward application of existing tools: DP
optimizers were designed for dense models, where every parameter
receives signal from every record, so one clipping bound, one noise
scale, and one normalization denominator suffice. Sparse routing
breaks this premise - shared components still see every record, while
each expert sees only the records routed to it - and the only prior
work on private MoE training~\cite{tholoniat2024differentially}
sidesteps the asymmetry by treating the model as a single dense
block.
We show this mismatch is a structural incompatibility, not an
implementation detail (Fig.~\ref{fig:teaser}, left): monolithic DP
training suppresses expert gradients dominated by denser blocks
(\emph{clipping interference}, panel~1), dilutes sparse expert
updates through full-batch normalization (\emph{update-scale
dilution}, panel~2), and concentrates fixed privacy noise on
low-load experts (\emph{load-dependent SNR degradation}, panel~3) -
together undermining the expert specialization that makes MoEs useful
(Sec.~\ref{sec:diagnosis}).

To address these limitations, we propose \Raptor\, (\textbf{Role-Aware Private Training}), which aligns the DP
mechanism with the shared/expert role structure rather than ignoring
it (Fig.~\ref{fig:teaser}, right). The \emph{shared stream} updates
dense components on all records via standard per-record clipping and
noise. The \emph{expert stream} updates each expert only on its
owner records, with expert-specific clipping $C_e$ and noise
$\sigma_e$ (resolving clipping interference, panel~1), a public
expected-owner denominator $D_e=B/E$ that never touches realized
counts (resolving dilution without leaking routing decisions,
panel~2), a count-independent schedule that updates every expert at
every step, and privacy-free entropy-based layer selection
(resolving load-dependent SNR degradation, panel~3). Fig.~\ref{fig:architecture} gives an overview.

\begin{figure*}[t]
  \centering
  \includegraphics[width=1\textwidth]{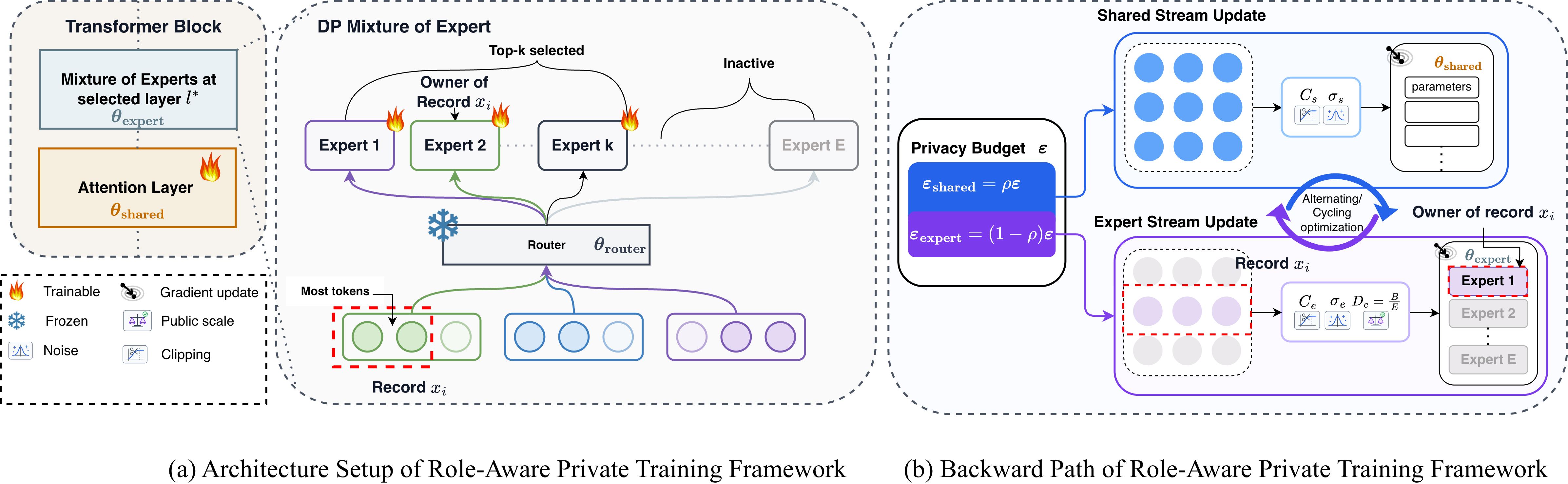}
  \caption{%
    \textbf{Overview of \Raptor\, at the selected sparse layer $l^\star$.}
    \textcolor{custom_purple}{\textbf{(a)}} The frozen router performs the base MoE model's
    standard token-level top-$k$ routing. Our framework aggregates gate
    weights across all tokens of record $x_i$ and selects one
    deterministic \emph{owner expert} $a_{l^\star}(x_i)$. The forward
    pass is unchanged: all activated experts contribute, but only the
    owner receives the expert-parameter gradient; activated non-owner
    experts use stop-gradient. This assignment partitions
    $\mathcal{D}_{\mathrm{priv}}$ into disjoint owner groups
    $\mathcal{D}_{l^\star,e}
    =\{(x_i,y_i):a_{l^\star}(x_i)=e\}$.
    \textcolor{custom_purple}{\textbf{(b)}} Each cycle draws independent Poisson subsamples for
    the shared and expert streams. Shared updates use clipping $C_s$,
    noise multiplier $\sigma_s$, and denominator $B$. Expert records
    are partitioned by owner and updated using $C_e$, $\sigma_e$, and
    the public denominator $D_e=B/E$, with noise-only updates for empty
    subsets. Expert mechanisms compose in parallel within $l^\star$,
    while shared and expert releases compose sequentially, with
    $\varepsilon_{\mathrm{sh}}=\rho\varepsilon$ and
    $\varepsilon_{\mathrm{exp}}=(1-\rho)\varepsilon$.%
  }
  \label{fig:architecture}
\end{figure*}

\paragraph{Contributions.}
\begin{enumerate}
  \item \textbf{Diagnostic analysis.} We identify and formally
    characterize three structural failure modes of monolithic DP-MoE
    training - clipping interference, update-scale dilution, and
    load-dependent SNR degradation (Sec.~\ref{sec:diagnosis}).
  \item \textbf{Role-aware optimization.} We propose Role-Aware
    Private Training, resolving all three failures via alternating
    shared/expert streams, expert-specific clipping and noise, and a
    count-independent update schedule (Sec.~\ref{sec:method}).
  \item \textbf{Privacy guarantee.} We prove
    $(\varepsilon,\delta)$-DP via parallel composition over
    routing-induced owner groups: updating all $E$ experts within a
    layer costs no more than updating one
    (Sec.~\ref{sec:privacy}).
  \item \textbf{Utility analysis and layer selection.} We bound the
    public-denominator estimator's bias by routing imbalance,
    yielding the privacy-free rule
    $l^\star = \arg\max_l H(q^{(l)})$ on a small public corpus, and
    show this bias is stable under bounded representation drift via a
    Voronoi abstraction (Sec.~\ref{sec:theory}).
  \item \textbf{Empirical validation.} Gains up to \textcolor{custom_darkgreen}{$+3.24$} points
    over monolithic and matched-scope DP baselines on Switch
    Transformer (four GLUE tasks, $\varepsilon\in\{1,4,8\}$), with
    consistent improvements on OLMoE and on the vision-language
    DeepSeek-VL2-Tiny model (Sec.~\ref{sec:experiments}).
\end{enumerate}

\section{Related Work}
Private fine-tuning of pretrained LMs - full-parameter
\cite{li2021large} or parameter-efficient via LoRA
\cite{yu2021differentially,hu2022lora} - applies one clipping bound,
noise scale, and denominator uniformly across all parameters, a
reasonable simplification for dense models where every parameter sees
every record. This breaks under MoE routing, where shared components
see every record but each expert sees only its routed subset; to our
knowledge, no DP fine-tuning method distinguishes these exposure
patterns. Separately, non-private MoE training stabilizes routing via
batch-level load-balancing losses, which do not translate to
record-level DP: their gradients break per-record sensitivity, and
they depend on realized expert counts, which are private and cannot
enter normalization or control flow without leaking routing decisions.
The only prior work on private MoE training,
\citet{tholoniat2024differentially}, sidesteps this by treating MoE as
a single dense block - feasible, but leaving shared, routing, and
expert parameters structurally indistinguishable, which suppresses
expert gradients, dilutes sparse updates, and concentrates noise on
low-load experts. We address this by freezing the router, replacing
realized counts with a public expected-owner denominator, and
analyzing expert updates via parallel composition over disjoint owner
groups - closing the gap left by dense DP methods
\cite{li2021large,yu2021differentially,hu2022lora}, which ignore
routing entirely, and by \citeauthor{tholoniat2024differentially},
which retains a monolithic optimizer despite it. Additional related
work is in Appendix~\ref{app:more_related}.
\section{Preliminaries}
\label{sec:preliminaries}
\subsection{Setup and Notation}
Let $\mathcal{D}_{\mathrm{priv}}=\{(x_i,y_i)\}_{i=1}^n$ be the private
dataset. A sparse MoE model has $E$ experts per sparse layer
($[E]=\{1,\ldots,E\}$) over layers $\mathcal{L}$, with parameters
$\theta=(\theta_{\mathrm{sh}},\theta_{\mathrm{rt}},\{\theta_{l,e}\})$.
Each expert carries a role $r(e)\in\{\textsc{routed},\textsc{shared}\}$:
shared experts (e.g., DeepSeek-MoE's always-active experts
\cite{dai2024deepseekmoe}) see every token and fold into
$\theta_{\mathrm{sh}}$; routed experts see only assigned tokens.
Unless stated otherwise, $r(e)=\textsc{routed}$ for all $e$.
At step $t$, $\mathcal{B}_t^{\mathrm{exp}}\subseteq\mathcal{D}_{\mathrm{priv}}$
is Poisson-subsampled, so its realized size is random and private; we
write $B=\mathbb{E}[|\mathcal{B}_t^{\mathrm{exp}}|]$ for the fixed,
public target size and normalize by $B$ - never the realized count -
the principle underlying $D_e=B/E$ (Sec.~\ref{sec:method}). For
gradient $g$ and radius $C{>}0$, $\mathrm{clip}(g,C)=g\cdot\min\{1,C/\|g\|_2\}$;
subscripts $s,e$ distinguish shared-/expert-stream quantities, e.g.\
$(C_s,\sigma_s,\eta_s)$ vs.\ $(C_e,\sigma_e,\eta_e)$.
\subsection{Differential Privacy}
\begin{definition}\cite{dwork2014algorithmic}
$\mathcal{A}$ is $(\varepsilon,\delta)$-DP if for all neighboring
$\mathcal{D},\mathcal{D}'$ and output sets $\mathcal{S}$,
$\Pr[\mathcal{A}(\mathcal{D})\in\mathcal{S}]\leq
e^\varepsilon\Pr[\mathcal{A}(\mathcal{D}')\in\mathcal{S}]+\delta$.
\end{definition}
DP-SGD~\cite{abadi2016deep} privatizes each step's gradient via
per-record clipping and Gaussian noise: for $\mathcal{B}_t$ of
expected size $B$,
\begin{equation}
\bar{g}_t=\frac{1}{B}\!\left(\sum_{i\in\mathcal{B}_t}
\mathrm{clip}(g_{i,t},C)+\mathcal{N}(0,\sigma^2C^2I)\right).
\end{equation}
We privatize every release this way, but optimize with Adam - $\bar{g}_t$
feeds Adam's moment estimates in place of the raw gradient, giving
\emph{DP-Adam}. This is post-processing of the already-released
sequence and adds no privacy cost beyond the $T$-fold composition
already charged (Sec.~\ref{sec:composition-alternating}).
\subsection{Sparse MoE Models and Load Balancing}
\label{sec:moe-background}
Sparse MoE Transformers~\cite{shazeer2017outrageously,fedus2022switch}
replace dense FFNs with $E$ experts and a router assigning each token
to $k$ of them; shared experts fold into the shared stream, routed
experts into the expert stream, and sparse activation causes load
imbalance. Our method assigns each record deterministically to one
routed expert per layer, its \emph{owner expert} - the expert
receiving the most of the record's tokens across the $k$ assignments,
reducing to the standard case at $k{=}1$ - the structure our
parallel-composition analysis relies on (Sec.~\ref{sec:privacy}).
Auxiliary load-balancing losses $\mathcal{L}_{\mathrm{aux}}$ depend on
batch-level routing statistics, giving per-record gradients that
reflect other records' decisions and break the per-record sensitivity
DP-SGD requires; we therefore freeze the router and omit
$\mathcal{L}_{\mathrm{aux}}$.
\section{Diagnosing Monolithic Private Optimization Failures}
\label{sec:diagnosis}

Monolithic private optimization applies one clipping bound, one noise scale,
and one normalization denominator to all MoE parameters - as if the model
were dense. This ignores that shared layers, routers, and task heads receive
signal from every record, while each expert receives signal only from its
routed records. Under DP, this asymmetry causes three structural failures,
each a property of the mechanism definition rather than any hyperparameter
choice. Appendix~\ref{app:diagnostics} formalizes this claim algebraically
and reports empirical role-mismatch and per-expert SNR diagnostics
(Tab.~\ref{tab:rq2_diagnostics_tasks}) that corroborate it directionally.
\subsection{Clipping Interference}
\label{sec:clipping-suppression}

The per-record gradient $g_{i,t}^{\mathrm{full}}=
[g^{\mathrm{sh}},g^{\mathrm{rt}},g^{\mathrm{head}},g^{\mathrm{exp}}]_{i,t}$
is clipped by a single factor
$\alpha_{i,t}^{\mathrm{full}}=\min(1,C/\|g_{i,t}^{\mathrm{full}}\|_2)$, dominated by the three dense
blocks, so it suppresses the expert component even when
$\|g_{i,t}^{\mathrm{exp}}\|_2$ alone is within $C$. Relaxing $C$ does not
resolve the coupling: noise scales with $C$, so any value permissive enough
for experts inflates dense-block noise. The natural resolution is a separate
clipping bound for the expert stream, as in our method.

\subsection{Update-Scale Dilution}
\label{sec:update-dilution}

Let $m_{t,l,e}=|\mathcal{B}_{t,l,e}|$ be the number of records routed to
expert $e$ in batch $\mathcal{B}_t^{\mathrm{exp}}$. The monolithic update
\begin{equation}
\widehat{g}_{t,l,e}^{\mathrm{mono}}
=\frac{1}{B}\!\left(\sum_{i\in\mathcal{B}_{t,l,e}}
\widetilde{g}_{i,t,l,e}^{\mathrm{exp}}+\xi_{t,l,e}\right)
\label{eq:mono-expert-update}
\end{equation}
normalizes by $B$ despite only $m_{t,l,e}$ records contributing signal,
rescaling expert gradients toward zero by $m_{t,l,e}/B$. Normalizing by the
realized count $m_{t,l,e}$ is not DP-free: it is a private quantity whose
use in normalization or control flow leaks routing information. We instead
use the public expected-owner denominator $D_e=B/E$, which is
data-independent and unbiased under uniform routing.

\subsection{Load-Dependent Signal Degradation}
\label{sec:load-imbalance}

Privacy noise per expert is fixed regardless of load; useful signal scales
with $m_{t,l,e}$. Uneven routing therefore concentrates noise on low-load
experts. We track this with the routing entropy:
\begin{align}
H_{t,l} &= -\frac{1}{\log E}
  \sum_{e=1}^{E}q_{t,l,e}\log q_{t,l,e},
  \label{eq:H}
\end{align}
where $q_{t,l,e}=m_{t,l,e}/B$,  $\tau=10^{-8}$ denotes a small constant added for numerical stability in ratio denominators.
Sec.~\ref{sec:theory} connects $H_{t,l}$ and the entropy deficit
$\log E - H_{t,l}$ to the imbalance bias of the public-denominator expert
estimator, 
motivating routing entropy as a candidate, privacy-free
layer-selection diagnostic, whose relationship to Adam-based training
utility Sec.~\ref{sec:theory} examines directly (Remark~\ref{rem:adam-invariance});

\section{Role-Aware Private Training}
\label{sec:method}

\subsection{Public Layer Selection}
\label{sec:public-layer-selection}
We protect a single sparse layer $l^\star$ per stage. This choice keeps
simple the parallel-composition argument we prove later
(Lemma~\ref{lem:owner_parallel}, Sec.~\ref{sec:privacy}):
updating all $E$ experts within a protected layer costs no more, in
privacy terms, than updating one, provided the \emph{owner map} - the
deterministic assignment of each private record to a single owner expert at
the protected layer, formalized in Sec.~\ref{sec:ownership}
(Eq.~\ref{eq:topk-owner}) - is fixed before training and does not change
during the stage. This condition holds trivially when only one layer is
trained. Multi-layer extension would require a staged scheme - train and
freeze one layer's owner map before the next, composing privacy costs
sequentially - left to future work; the main paper scopes to
$L=\{l^\star\}$.
Before private training, we select one sparse layer using external
public unlabeled data $\mathcal{C}_{\mathrm{pub}}$ (disjoint from
$\mathcal{D}_{\mathrm{priv}}$, format-matched to the task, e.g.\
unlabeled IMDb reviews for sentiment, SQuAD contexts for entailment).
For each candidate layer $l$, we apply the frozen router to
$\mathcal{C}_{\mathrm{pub}}$, compute the empirical owner distribution
$q^{(l)}$, and set
\begin{equation}
l^\star = \arg\max_l H\!\left(q^{(l)}\right).
\label{eq:layer-selection}
\end{equation}
Layer selection (Eq.~\ref{eq:layer-selection}) requires a disjoint,
format-matched public corpus; without one, we default to the last
sparse layer. Tab.~\ref{tab:switch_main_results} shows first-/last-sparse trailing
entropy-selection on less balanced tasks - but in one setting lacking
a public corpus it still improved over the monolithic DP baseline,
suggesting the core mechanism does not strictly require entropy-based
selection.
Since $\mathcal{C}_{\mathrm{pub}}$ is public and the router is frozen,
this consumes \emph{no privacy budget} and does not touch
$\mathcal{D}_{\mathrm{priv}}$. Sec.~\ref{sec:theory} shows entropy
provably shrinks the public-denominator estimator's bias $\beta_e$
(Theorem~\ref{thm:agg-bias}); Remark~\ref{rem:adam-invariance}
clarifies this does not extend to a guarantee on per-expert
signal-to-noise ratio, whose benefit for Adam-based training we
support empirically (Sec.~\ref{sec:experiments}) rather than by an
analogous bound.

\subsection{Record-Level Expert Ownership}
\label{sec:ownership}
At layer $l$, the frozen router gives each token $t\in\mathrm{Tok}(x_i)$
a top-$k$ expert set with gate weights $g_{l,e}(x_i,t)\ge 0$
($g_{l,e}=0$ off the top-$k$, $\sum_e g_{l,e}=1$). We lift this to a
deterministic record-level owner map by gate-weighted plurality,
\begin{equation}
a_l(x_i)=\operatorname*{arg\,max}_{e\in[E]}
\sum_{t\in\mathrm{Tok}(x_i)} g_{l,e}(x_i,t),
\label{eq:topk-owner}
\end{equation}
ties broken by lowest index; replacing $g_{l,e}$ by
$\mathbb{1}[e\in\text{top-}k]$ recovers the majority-vote rule at
$k=1$. This induces a deterministic partition of $\mathcal{D}_{\mathrm{priv}}$
into \emph{owner groups} $\mathcal{D}_{l,e}$, one per expert
$e \in [E]$, with $\bigcup_e \mathcal{D}_{l,e} = \mathcal{D}_{\mathrm{priv}}$
for every $k \ge 1$.
The forward pass uses standard top-$k$ routing, but the \emph{expert
stream} deposits gradient only into the owner expert's parameters via a
stop-gradient $\mathrm{sg}(\cdot)$ on all non-owner experts:
\begin{equation}
\begin{aligned}
a_i &\coloneqq a_l(x_i),\\
g^{\exp}_{i,l}
&=
\nabla_{\theta_{l,a_i}}
\ell_{\mathrm{exp}}\!\left(
x_i;\theta_{l,a_i},
\theta^{\mathrm{sg}}_{l,\neg a_i}
\right),
\end{aligned}
\label{eq:owner-mask}
\end{equation}
so each record contributes to exactly one owner expert; only the
attribution of private expert updates is restricted. Under
add/remove adjacency, one record changes only group
$\mathcal{D}_{l,a_l(x_0)}$, by a single clipped contribution of norm
$\le C_e$. Hence $\Delta_{l,e}=C_e/D_e=EC_e/B$, independent of
$|\mathcal{D}_{l,e}|$ and $k$, and conditions~(i)-(iii) of
Lemma~\ref{lem:owner_parallel} (Sec.~\ref{sec:parallel-composition})
hold verbatim for any $k$. The owner map and realized counts $|\mathcal{D}_{l,e}|$ are never
released; $a_{l^\star}$ is computed once from the frozen router before
training and fixed thereafter.

\begin{remark}[Why masking, not just joint clipping]
\label{rem:masking}
Joint-clipping a record's concatenated per-expert gradients to one radius $C$
would also give sensitivity $C$ without exclusivity, avoiding
Lemma~\ref{lem:owner_parallel}'s condition~(ii) cost. We reject this for two
reasons independent of composition cost: it reintroduces clipping interference
(Sec.~\ref{sec:clipping-suppression}) \emph{among} experts, and it makes
expert $e$'s clip factor depend on other experts' gradient norms, breaking the
per-group independence Theorem~\ref{thm:bias-var} and the entropy criterion
rely on.
\end{remark}

\subsection{Shared and Expert Stream Updates}
\label{sec:stream-updates}
\label{sec:shared-stream}
\label{sec:expert-stream}

\paragraph{Shared stream.}
Every record contributes to the shared loss, so the shared stream
computes a standard privatized estimate over its own subsample:
\begin{equation}
\bar{g}_t^{\mathrm{sh}}
=\frac{1}{B}\!\left(
\sum_{i\in\mathcal{B}_t^{\mathrm{sh}}}\mathrm{clip}(g_{i,t}^{\mathrm{sh}},C_s)
+\mathcal{N}(0,\sigma_s^2C_s^2I)\right).
\label{eq:shared-update}
\end{equation}

\paragraph{Residual objective.}
Let $z_i^{\mathrm{sh}}\in\mathbb{R}^{|\mathcal{Y}|}$ be the output
logits for record~$i$ under the current shared parameters
$\theta_{\mathrm{sh}}$ (expert parameters $\{\theta_{l^\star,e}\}$
held fixed, no gradient this pass). Define the residual weight
\begin{equation}
  w_i = 1 - \operatorname{softmax}(z_i^{\mathrm{sh}})_{y_i}
  \;\in[0,1],
  \label{eq:residual-weight}
\end{equation}
the shared stream's predicted error probability on the true class:
$w_i\approx0$ when confidently correct, $w_i\approx1$ when not. The
expert loss $\ell_{\mathrm{exp}}=w_i\,\ell_{\mathrm{cls}}$
concentrates expert capacity on records the shared stream has not yet
classified confidently. Since $w_i$ enters before clipping to $C_e$,
sensitivity is unaffected.

\paragraph{Privatized expert update.}
For owner subset
$\mathcal{B}_{t,l,e}=\{i\in\mathcal{B}_t^{\mathrm{exp}}:a_l(x_i)=e\}$
and public denominator $D_e=B/E$, the \emph{expert stream} computes
its own privatized estimate over the owner subset:
\begin{equation}
\begin{aligned}
\bar{g}_{t,l,e}^{\mathrm{exp}}
= \frac{1}{D_e}\Bigg(
&\sum_{i\in\mathcal{B}_{t,l,e}}
\mathrm{clip}\!\left(g_{i,t,l,e}^{\mathrm{exp}}, C_e\right) \\
&+ \mathcal{N}(0,\sigma_e^2 C_e^2 I)
\Bigg)
\end{aligned}
\label{eq:expert-update}
\end{equation}
Three properties follow: (i)~expert clipping decouples from the
shared stream, resolving the suppression of
Sec.~\ref{sec:clipping-suppression}; (ii)~$D_e$ depends only on
the public $B,E$, so realized counts are never revealed; (iii)~the
normalized sensitivity $C_e/D_e=EC_e/B$ feeds directly into the
accountant of Sec.~\ref{sec:privacy}.
\paragraph{Empty owner subsets.}
If $\mathcal{B}_{t,l,e}=\emptyset$, the update reduces to noise alone,
$\bar{g}_{t,l,e}^{\mathrm{exp}}=D_e^{-1}\mathcal{N}(0,\sigma_e^2C_e^2I)$:
every expert updates at every step with no count-dependent skipping
or logging, which would leak realized counts and require additional
accounting.
\subsection{Alternating Optimization and Layer Scope}
\label{sec:alternating}
Each outer cycle - one iteration of the \textbf{for} $t$ loop in
Algorithm~\ref{alg:role-aware} (lines~\ref{line:poisson-sh}-\ref{line:poisson-exp}
and the nested expert loop that follows) - draws two independent Poisson
subsamples, $\mathcal{B}_t^{\mathrm{sh}}$ and $\mathcal{B}_t^{\mathrm{exp}}$
- rather than reusing one sample for both streams, which costs no
additional data access since the expert stream already needs its own
forward pass to compute $w_i$ - then performs one shared-stream update
over $\mathcal{B}_t^{\mathrm{sh}}$ followed by one expert-stream update
over $\mathcal{B}_t^{\mathrm{exp}}$ for all $E$ experts. Alternating lets
shared representations and expert residual corrections co-adapt: each
shared step updates the $w_i$ seen by the next expert step, and each
expert step changes the residual for the next shared step. Privacy
accounting for the two streams remains independent of the downstream
optimizer, since the independent subsamples make their per-step
releases independent random mechanisms (Sec.~\ref{sec:privacy}).

\begin{algorithm*}[t]
\small
\caption{\Raptor}
\label{alg:role-aware}

\begin{algorithmic}[1]

\Require $\mathcal{D}_{\mathrm{priv}}$ ($n$ records), $\mathcal{C}_{\mathrm{pub}}$, frozen router, candidate layers, $B,E,p{=}B/n,C_s,C_e,\sigma_s,\sigma_e,\rho$

\State $l^\star \leftarrow \arg\max_l H(q^{(l)})$ on $\mathcal{C}_{\mathrm{pub}}$
\Comment{\textcolor{custom_purple}{no privacy cost}}

\State $a_{l^\star}(x_i) \leftarrow \arg\max_{e\in[E]} \sum_{t\in\mathrm{Tok}(x_i)} g_{l^\star,e}(x_i,t)$, $\forall x_i\in\mathcal{D}_{\mathrm{priv}}$
\Comment{\textcolor{custom_purple}{once; never released}}

\For{$t=1,\dots,T$}

\State\label{line:poisson-sh} $\mathcal{B}_t^{\mathrm{sh}} \leftarrow \{i: u_i^{\mathrm{sh}}{=}1\}$, $u_i^{\mathrm{sh}}\overset{\text{iid}}{\sim}\mathrm{Bern}(p)$
\Comment{\textcolor{custom_purple}{shared subsample}}

\State $\bar{g}_t^{\mathrm{sh}} \leftarrow \tfrac{1}{B}\big(\sum_{i\in\mathcal{B}_t^{\mathrm{sh}}}\mathrm{clip}(g_{i,t}^{\mathrm{sh}},C_s) + \mathcal{N}(0,\sigma_s^2C_s^2I)\big)$; DP-Adam on $\theta_{\mathrm{sh}}$

\State\label{line:poisson-exp} $\mathcal{B}_t^{\mathrm{exp}} \leftarrow \{i: u_i^{\mathrm{exp}}{=}1\}$, $u_i^{\mathrm{exp}}\overset{\text{iid}}{\sim}\mathrm{Bern}(p)$, $\{u_i^{\mathrm{exp}}\}\perp\{u_i^{\mathrm{sh}}\}$
\Comment{\textcolor{custom_purple}{independent subsample}}

\For{$e\in[E]$}

\State $\mathcal{B}_{t,e} \leftarrow \{i\in\mathcal{B}_t^{\mathrm{exp}} : a_{l^\star}(x_i){=}e\}$

\State $\bar{g}_{t,e}^{\mathrm{exp}} \leftarrow \tfrac{1}{D_e}\big(\sum_{i\in\mathcal{B}_{t,e}}\mathrm{clip}(g_{i,t,e}^{\mathrm{exp}},C_e) + \mathcal{N}(0,\sigma_e^2C_e^2I)\big)$; DP-Adam on $\theta_{l^\star,e}$
\Comment{\textcolor{custom_purple}{$D_e{=}B/E$; noise-only if empty}}

\EndFor
\EndFor

\Ensure $\theta_T$

\end{algorithmic}
\end{algorithm*}

\section{Privacy Analysis}
\label{sec:privacy}

We analyze our framework under $(\varepsilon,\delta)$-DP
with add/remove adjacency; only the final model is released; owner
assignments, expert counts, residual weights, and intermediate
gradients are never exposed.
\subsection{Parallel Composition Across Experts}
\label{sec:parallel-composition}
\begin{lemma}[Parallel composition across owner groups]
\label{lem:owner_parallel}
Fix layer $l$. Suppose \textup{(i)}~$a_l$ is deterministic and
stage-fixed; \textup{(ii)}~every record belongs to exactly one owner
group $\mathcal{D}_{l,e}$ (the partition cell of records assigned to
expert $e$ by $a_l$, Sec.~\ref{sec:ownership}); and
\textup{(iii)}~$M_{l,e}$ depends only on $\mathcal{D}_{l,e}$ and
public quantities. If each $M_{l,e}$ is $(\varepsilon_{l,e},\delta_{l,e})$-DP
on its owner group, the joint mechanism
$M_l=(M_{l,1},\dots,M_{l,E})$ is
$(\max_e\varepsilon_{l,e},\,\max_e\delta_{l,e})$-DP.
\end{lemma}
All three conditions hold by construction: $a_l$ is stage-fixed
(Sec.~\ref{sec:ownership}); the public denominator $D_e=B/E$ and
count-independent schedule make $M_{l,e}$ depend only on
$\mathcal{D}_{l,e}$ and public quantities; and per-group DP follows
from subsampled-Gaussian accounting with sensitivity
$\Delta_{l,e}$ (Eq.~\eqref{eq:sensitivity-expert}). Full proofs are in
Appendix~\ref{app:owner_parallel_proof}.

\subsection{Composition Across Streams and Layers}
\label{sec:composition-alternating}
Interleaved shared/expert updates compose adaptively under the PRV
framework~\cite{gopi2021numerical}: since $\mathcal{B}_t^{\mathrm{sh}}$
and $\mathcal{B}_t^{\mathrm{exp}}$ are independent Poisson subsamples
drawn afresh each cycle, per-step mechanisms are conditionally
independent given $\theta_t$, so standard PRV convolution applies
without a joint privacy-loss variable. Writing
$\omega^{\circledast T}$ for $T$-fold self-convolution: the shared
stream has per-step PRV $\omega_{\mathrm{sh}}$ set by $(C_s,\sigma_s,\gamma)$;
each layer's expert mechanisms, by
Lemma~\ref{lem:owner_parallel}, compose in \emph{parallel} (disjoint
owner groups, no realized counts released), giving a single
layer-level PRV $\omega_l = \omega_{l,e}$ (homogeneous $C_e,\sigma_e,D_e{=}B/E$
across experts) rather than an $E$-fold convolution. Sequential
composition across shared updates and trained layers then gives
\begin{equation}
\omega_{\mathrm{tot}} = \omega_{\mathrm{sh}}^{\circledast T_{\mathrm{sh}}}
\circledast \bigcircledast_{l\in\mathcal{L}_{\mathrm{train}}}
\omega_l^{\circledast T_{\mathrm{exp},l}},
\label{eq:prv-composition}
\end{equation}
with $\mathcal{L}_{\mathrm{train}}\subseteq\mathcal{L}$ the
privately-updated layers and
$\varepsilon(\delta)=\inf\{\varepsilon\ge0:\Pr[\omega_{\mathrm{tot}}>\varepsilon]\le\delta\}$.
\paragraph{Budget split.} A fraction $\rho\in(0,1)$ of $\varepsilon$
goes to the shared stream, $1{-}\rho$ to the experts; we solve for
$\sigma_s,\sigma_e$ jointly via binary search over the PRV accountant,
using $\rho=0.9$ at $E=8$ (Appendix~\ref{app:budget-allocation});
sensitivity to $\rho$ and larger $E$ is in Appendix~\ref{sec:hparam}.
\section{Utility Analysis of the Expert Estimator}
\label{sec:theory}
\subsection{Bias-Variance Decomposition and Layer Selection}
\label{sec:bias-variance}

\paragraph{Setup.}
Fix one selected layer (suppress layer index). The owner map assigns
each record to exactly one owner expert via gate-weighted plurality over its
tokens (Eq.~\eqref{eq:topk-owner}), so it partitions
$\mathcal{D}_{\mathrm{priv}}$ into disjoint groups $\mathcal{D}_e$ for
any $k \ge 1$ - the forward pass may still route a record's tokens to
$k$ experts, but only the owner expert receives its gradient. Groups have sizes $n_e$,
empirical frequencies $q_e=n_e/n$, and target means
$\mu_e=n_e^{-1}\sum_{i\in\mathcal{D}_e}h_i^{\mathrm{exp}}$, where
$h_i^{\mathrm{exp}}=\mathrm{clip}(g_i^{\mathrm{exp}},C_e)$. Write
$u=(1/E,\dots,1/E)$ for the uniform distribution, $\Sigma_e$ for the
within-group covariance; under Poisson subsampling at rate $p=B/n$,
$\mathbb{E}|\mathcal{B}_{t,e}|=Bq_e$.

\begin{theorem}[Bias-variance decomposition]
\label{thm:bias-var}
Under Poisson subsampling with rate $p=B/n$, letting $d$ denote
the parameter dimension,
$\mathbb{E}\|\bar{g}_{t,e}^{\mathrm{exp}}-\mu_e\|_2^2
=\beta_e^2+V_e^{\mathrm{samp}}+V_e^{\mathrm{DP}}$,
where
\begin{align}
\beta_e^2 &= (Eq_e-1)^2\|\mu_e\|_2^2, \label{eq:bias-term}\\
V_e^{\mathrm{samp}} &= \frac{E^2p\,n_e(1-p)}{B^2}
  \bigl(\|\mu_e\|_2^2+\operatorname{tr}\Sigma_e\bigr),
  \label{eq:samp-var}\\
V_e^{\mathrm{DP}} &= \frac{E^2\sigma_e^2C_e^2\,d}{B^2}.
  \label{eq:dp-var}
\end{align}
\end{theorem}

The squared bias $\beta_e^2=(Eq_e-1)^2\|\mu_e\|_2^2$ vanishes exactly
when $q_e=1/E$ - the public denominator $D_e=B/E$ is unbiased under
uniform routing, with bias growing quadratically under imbalance --
and the $E^2$ factor in $V_e^{\mathrm{DP}}$ reflects the normalized
sensitivity $EC_e/B$. Proofs are in Appendix~\ref{app:bias_var_proof}.


\begin{corollary}[Balanced routing]
\label{cor:balanced-error}
If $q_e=1/E$, then $\beta_e=0$ and
\begin{equation}
\sqrt{\mathbb{E}\bigl\|\bar{g}_{t,e}^{\mathrm{exp}}-\mu_e\bigr\|_2^2}
\;\le\;
C_e\!\sqrt{\frac{E}{B}}
+\frac{E\sigma_eC_e\sqrt{d}}{B},
\label{eq:balanced-rmse}
\end{equation}
separating non-private sampling error from the DP cost; scaling and comparison to
the standard DP mean-estimation rate are in
Appendix~\ref{app:dp-scaling}.
\end{corollary}

\begin{remark}
\label{rem:adam-invariance}
The bias $\beta_e$ is optimizer-agnostic - vanishing under uniform
routing (Theorem~\ref{thm:bias-var}), bounded under imbalance
(Theorem~\ref{thm:agg-bias}) - but does not by itself determine
DP-Adam's training utility. It is a rescaling of $\mu_e$ by the
constant $Eq_e$, fixed throughout training, to which DP-Adam's
normalization is largely invariant away from the
$\epsilon_{\mathrm{Adam}}$ floor; and $\beta_e^2$ is symmetric around
$q_e=1/E$, so an expert at $q_e=2/E$ has the same $\beta_e^2$ as one
at $q_e=0$, though these are opposite outcomes for training. The
closer proxy for Adam-based utility is the per-expert
signal-to-noise ratio $\mathrm{SNR}_e\propto
q_e^2\|\mu_e\|_2^2/(\sigma_e^2C_e^2d)$ -- monotonic in $q_e$ and
consistent with Tab.~\ref{tab:rq2_diagnostics_tasks} - but no bound
$1/\mathrm{SNR}_e\le f(H(q))$ holds for any finite $f$, since a
starved expert contributes only $O(1/E)$ to $\chi^2(q\|u)\in[0,E{-}1]$
and can coexist with near-maximal entropy while $\mathrm{SNR}_e\to0$
(Appendix~\ref{app:remark2-detail}). We therefore present
$l^\star=\arg\max_l H(q^{(l)})$ as a privacy-free, empirically
validated heuristic - supported by its falsifiability and gains in
Tab.~\ref{tab:switch_main_results} and
Appendix~\ref{app:entropy-analysis} - not a rule
Theorem~\ref{thm:agg-bias} certifies optimal for DP-Adam dynamics.
\end{remark}
\paragraph{Routing imbalance and chi-squared divergence.}
Aggregating $\beta_e^2$ across experts via
$\sum_e(Eq_e-1)^2=E\chi^2(q\|u)$ (Appendix~\ref{app:agg_bias_proof})
gives:

\begin{theorem}[Aggregate imbalance bias]
\label{thm:agg-bias}
Let $\bar{G}^2=\max_e\|\mu_e\|_2^2$. Then
\begin{align}
\frac{1}{E}\sum_{e=1}^{E}\beta_e^2
  &\le\bar{G}^2\chi^2(q\|u),
  \label{eq:unweighted-bias}\\
\sum_{e=1}^{E}q_e\beta_e^2
  &\le\bar{G}^2 E\chi^2(q\|u).
  \label{eq:weighted-bias}
\end{align}
\end{theorem}

\paragraph{Entropy as a layer-selection proxy.}
Chi-squared divergence is harder to estimate stably on small public
samples than entropy; via Pinsker's inequality and
$\chi^2(q\|u)=E\|q-u\|_2^2\le E\|q-u\|_1^2$,
\begin{equation}
\chi^2(q\,\|\,u)\;\le\;2E\,\bigl(\log E-H(q)\bigr),
\label{eq:entropy-upper}
\end{equation}
combining with Eq.~\eqref{eq:weighted-bias} gives
\begin{equation}
\sum_e q_e\,\beta_e^2
\;\le\;2\bar{G}^2E^2\,\bigl(\log E-H(q)\bigr).
\label{eq:weighted-bias-entropy}
\end{equation}
The right-hand side of Eq.~\eqref{eq:weighted-bias-entropy} is the
only lever we can pull without touching private data: $\bar{G}^2$ and
$E$ are fixed by the model, so the bound on the aggregate bias
$\sum_e q_e\beta_e^2$ shrinks exactly as $\log E - H(q)$ shrinks. Minimizing the guaranteed bias therefore
reduces to maximizing $H(q)$ over candidate layers, which is
computable from a public corpus alone: this is the entire
justification for the criterion $l^\star=\arg\max_l H(q^{(l)})$. A
layer with higher $H(q)$ thus has a smaller imbalance-bias guarantee. 
This gives a privacy-free, computable criterion: $l^\star=\arg\max_l
H(q^{(l)})$ needs only $\mathcal{C}_{\mathrm{pub}}$ and the frozen
router, with its connection to DP-Adam utility scoped in
Remark~\ref{rem:adam-invariance}.

\subsection{Stability of Expert Ownership}
\label{sec:ownership-stability}
The owner map $a_{l^\star}$ is frozen before training and never
recomputed as $\theta_{\mathrm{sh}}$ updates; this does not affect
privacy (Lemma~\ref{lem:owner_parallel} needs only stage-fixedness),
but bears on whether it stays a faithful specialization proxy as the
representation drifts. Under a Voronoi abstraction, we bound this drift (proofs in Appendix~\ref{app:stability}--\ref{app:stab_error}):

\begin{theorem}[Ownership stability, informal]
\label{thm:voronoi-stability}
If the representation moves by at most $r$, only records with margin
$\le 2\Delta_{\mathcal C} r$ can change owner ($\Delta_{\mathcal C}$:
max prototype distance); under a mild boundary-mass condition
(rate $\kappa$), both reassignment probability and routing-mass
drift scale \emph{linearly} in $r$ - only near-boundary records can
flip, and how many there are grows gracefully with drift, not
catastrophically.
\end{theorem}

\begin{theorem}[Estimation error under drift, informal]
\label{thm:voronoi-error}
Assuming an $L_g$-Lipschitz gradient field,
\[
\begin{aligned}
\mathbb{E}\|\bar g_e-\nu_e\|_2^2
&\le V_e^{\mathrm{samp}}+V_e^{\mathrm{DP}} \\
&\quad +2(Eq_e-1)^2\|\mu_e\|_2^2+2L_g^2 r_e^2,
\end{aligned}
\]
$r_e^2$ the mean squared distance from a record to its owning
prototype. The first three terms match Theorem~\ref{thm:bias-var};
the last is the extra cost of ownership being a coarse proxy rather
than exact, vanishing as owner groups tighten around their prototype.
\end{theorem}

\section{Experiments}
\label{sec:experiments}

\subsection{Setup}
\label{sec:exp-setup}

\paragraph{Model and tasks.}
We fine-tune \texttt{google/switch-base-8}~\cite{fedus2022switch}
($E=8$, six encoder MoE layers at blocks $\{1,3,5,7,9,11\}$) on four
GLUE tasks~\cite{wang2018glue} (SST-2, MNLI, QNLI, QQP), reporting
dev-set accuracy. We additionally evaluate
\texttt{OLMoE-1B-7B}~\cite{muennighoff2025olmoe} ($E=64$, top-$8$
routing) to exercise the primary-expert ownership rule for $k>1$
(Sec.~\ref{sec:ownership}). We additionally test \texttt{DeepSeek-VL2-Tiny}~\cite{wu2024deepseek}, a vision-language MoE with both shared and
routed experts (Sec.~\ref{sec:moe-background}), fine-tuned on
ScienceQA~\cite{lu2022learn}, to test generalization beyond
text-only models (Appendix ~\ref{app:more_results}). Hyperparameter grids are in
Appendix~\ref{app:exp_detail}; runtime/memory in
Appendix~\ref{app:cost}; ablations in Appendix~\ref{app:ablation} - component ablation
(clipping, denominator, alternating schedule, residual objective),
matched-scope diagnostics (shared-/expert-only, global DP at Ours's
scope), and budget-split $\rho$ sensitivity.

\paragraph{Privacy.}
All private methods use $\varepsilon\in\{1,4,8\}$ under
$(\varepsilon,\delta)$-DP with $\delta=1/|\mathcal{D}_{\mathrm{priv}}|^{1.1}$, matched across methods.

\paragraph{Layer selection.}
$l^\star=\arg\max_l H(q^{(l)})$, computed from the frozen router and
external public data only (no privacy cost), yields block~1 (SST-2),
block~7 (MNLI), block~11 (QNLI), block~9 (QQP) - see
Appendix~\ref{app:entropy-analysis} for the full sweep. We also report
first- and last-sparse-layer variants to make the entropy prediction
falsifiable.

\paragraph{Baselines.}
\textit{Non-private upper bounds:} LoRA, full fine-tuning without
noise. \textit{Monolithic DP:} DP-Adam LoRA~\cite{yu2021differentially}
and DP-Adam~\cite{tholoniat2024differentially}, both treating all
parameters as one dense block.

\subsection{Results on Switch Transformer}
\label{sec:main-results}
Tab.~\ref{tab:switch_main_results} reports results at
$\varepsilon\in\{1,4,8\}$.

\paragraph{\Raptor\, vs.\ monolithic DP.}
Ours attains the highest accuracy on all twelve task-budget cells.
Gains over Monolithic DP-Adam LoRA range \textcolor{custom_darkgreen}{$+0.23$} to \textcolor{custom_darkgreen}{$+3.24$} pts
(SST-2/MNLI/QNLI/QQP) at $\varepsilon=8$ and \textcolor{custom_darkgreen}{$+1.51$} to \textcolor{custom_darkgreen}{$+2.92$} at
$\varepsilon=1$, growing at tighter budgets on SST-2/QNLI but not
MNLI/QQP - the high-noise advantage of
Corollary~\ref{cor:balanced-error} is baseline- and task-dependent,
not uniform.
\paragraph{Entropy criterion.}
Selection is non-trivially falsifiable on MNLI (block~7) and QQP
(block~9), neither first nor last; entropy-selected matches or
exceeds both variants in every cell, with the largest gains on MNLI
and QQP, where chosen layers differ most from the runner-up.

\begin{table}[t]
\centering
\small
\setlength{\tabcolsep}{2pt}
\begin{tabular}{llcccc}
\toprule
$\varepsilon$ & Method & SST-2 & MNLI & QNLI & QQP \\
\midrule
\multirow{2}{*}{$\infty$}
& Non-priv.\ LoRA
    & \textbf{94.69} & \textbf{86.55}
    & \textbf{91.20} & \underline{90.96} \\
& Non-priv.\ FFN
    & \underline{94.50} & \underline{86.36}
    & \underline{91.09} & \textbf{91.28} \\
\midrule

\multirow{5}{*}{$1$}
& Monolith.\ DP-Adam LoRA
    & 90.94 & 76.91 & 80.38 & 81.19 \\
& Monolith.\ DP-Adam
    & 87.04 & 72.72 & 81.44 & 81.84 \\
& \cellcolor{custom_light_purple_2} Ours, first sparse
    & \cellcolor{custom_light_purple_2}\textbf{92.90}
    & \cellcolor{custom_light_purple_2}77.40
    & \cellcolor{custom_light_purple_2}\underline{82.60}
    & \cellcolor{custom_light_purple_2}83.20 \\
& \cellcolor{custom_light_purple_2} Ours, last sparse
    & \cellcolor{custom_light_purple_2}\underline{92.66}
    & \cellcolor{custom_light_purple_2}\underline{77.87}
    & \cellcolor{custom_light_purple_2}\textbf{83.30}
    & \cellcolor{custom_light_purple_2}\underline{83.80} \\
& \cellcolor{custom_light_purple_2}\textbf{Ours, entropy-sel.}
    & \cellcolor{custom_light_purple_2}\textbf{92.90}
    & \cellcolor{custom_light_purple_2}\textbf{78.42}
    & \cellcolor{custom_light_purple_2}\textbf{83.30}
    & \cellcolor{custom_light_purple_2}\textbf{84.08} \\
\midrule

\multirow{5}{*}{$4$}
& Monolith.\ DP-Adam LoRA
    & 92.66 & 78.33 & 81.95 & 82.26 \\
& Monolith.\ DP-Adam
    & 90.48 & 75.90 & 82.41 & 83.32 \\
& \cellcolor{custom_light_purple_2} Ours, first sparse
    & \cellcolor{custom_light_purple_2}\textbf{94.05}
    & \cellcolor{custom_light_purple_2}78.13
    & \cellcolor{custom_light_purple_2}\underline{83.60}
    & \cellcolor{custom_light_purple_2}84.50 \\
& \cellcolor{custom_light_purple_2} Ours, last sparse
    & \cellcolor{custom_light_purple_2}\underline{93.46}
    & \cellcolor{custom_light_purple_2}\underline{78.50}
    & \cellcolor{custom_light_purple_2}\textbf{84.70}
    & \cellcolor{custom_light_purple_2}\underline{84.78} \\
& \cellcolor{custom_light_purple_2}\textbf{Ours, entropy-sel.}
    & \cellcolor{custom_light_purple_2}\textbf{94.05}
    & \cellcolor{custom_light_purple_2}\textbf{79.03}
    & \cellcolor{custom_light_purple_2}\textbf{84.70}
    & \cellcolor{custom_light_purple_2}\textbf{84.78} \\
\midrule

\multirow{5}{*}{$8$}
& Monolith.\ DP-Adam LoRA
    & \underline{93.92} & 79.36 & 82.48 & 82.68 \\
& Monolith.\ DP-Adam
    & 90.37 & 76.70 & 82.61 & 84.09 \\
& \cellcolor{custom_light_purple_2} Ours, first sparse
    & \cellcolor{custom_light_purple_2}\textbf{94.15}
    & \cellcolor{custom_light_purple_2}79.43
    & \cellcolor{custom_light_purple_2}\underline{83.60}
    & \cellcolor{custom_light_purple_2}85.39 \\
& \cellcolor{custom_light_purple_2} Ours, last sparse
    & \cellcolor{custom_light_purple_2}93.50
    & \cellcolor{custom_light_purple_2}\underline{80.58}
    & \cellcolor{custom_light_purple_2}\textbf{85.26}
    & \cellcolor{custom_light_purple_2}\underline{85.73} \\
& \cellcolor{custom_light_purple_2}\textbf{Ours, entropy-sel.}
    & \cellcolor{custom_light_purple_2}\textbf{94.15}
    & \cellcolor{custom_light_purple_2}\textbf{81.32}
    & \cellcolor{custom_light_purple_2}\textbf{85.26}
    & \cellcolor{custom_light_purple_2}\textbf{85.92} \\
\bottomrule
\end{tabular}

\caption{\textsc{Switch-base-8} under $(\varepsilon,\delta)$-DP with
$\delta=1/|\mathcal{D}_{\mathrm{priv}}|^{1.1}$.
Entropy-selected layers fixed from public routing statistics
before training: block~1 (SST-2), block~7 (MNLI),
block~11 (QNLI), block~9 (QQP).
\textbf{Bold}: best per task--budget;
\underline{underline}: second-best distinct method.}
\label{tab:switch_main_results}
\end{table}

\subsection{Results on OLMoE}
\label{sec:olmoe-results}
Tab.~\ref{tab:olmoe_main_results} reports \texttt{OLMoE-1B-7B}
results at $\varepsilon\in\{1,4,8\}$, with records assigned via
majority-vote ownership for $k>1$ (Sec.~\ref{sec:ownership}). Ours
(entropy-selected) attains the highest accuracy in all twelve
task--budget cells. At $\varepsilon=8$, gains over Monolithic DP-Adam
LoRA are \textcolor{custom_darkgreen}{$+0.69$} (SST-2), \textcolor{custom_darkgreen}{$+2.18$} (MNLI), \textcolor{custom_darkgreen}{$+1.74$} (QNLI), \textcolor{custom_darkgreen}{$+0.32$}
(QQP); at $\varepsilon=1$, \textcolor{custom_darkgreen}{$+2.63$}, \textcolor{custom_darkgreen}{$+0.84$}, \textcolor{custom_darkgreen}{$+4.50$}, \textcolor{custom_darkgreen}{$+0.70$} respectively. SST-2,
QNLI, and QQP gains grow at tighter budgets, consistent with the
high-noise regime prediction of Corollary~\ref{cor:balanced-error}
and mirroring Sec.~\ref{sec:main-results}; the MNLI gain instead
shrinks from \textcolor{custom_darkgreen}{$+2.18$} to \textcolor{custom_darkgreen}{$+0.84$}.

\begin{table}[h]
\centering
\small
\setlength{\tabcolsep}{2pt}
\begin{tabular}{llcccc}
\toprule
$\varepsilon$ & Method & SST-2 & MNLI & QNLI & QQP \\
\midrule
\multirow{2}{*}{$\infty$}
& Non-priv.\ LoRA & \textbf{96.67} & \textbf{90.19} & \textbf{94.65} & \textbf{92.03} \\
& Non-priv.\ FFN  & 95.18 & 89.72 & 93.13 & 91.07 \\
\midrule
\multirow{2}{*}{$1$}
& Monolith.\ DP-Adam LoRA
    & 92.78 & 85.31 & 84.24 & 85.89 \\
& \cellcolor{custom_light_purple_2}\textbf{Ours, entropy-sel.}
    & \cellcolor{custom_light_purple_2}\textbf{95.41}
    & \cellcolor{custom_light_purple_2}\textbf{86.15}
    & \cellcolor{custom_light_purple_2}\textbf{88.74}
    & \cellcolor{custom_light_purple_2}\textbf{86.59} \\
\midrule
\multirow{2}{*}{$4$}
& Monolith.\ DP-Adam LoRA
    & 95.07 & 85.74 & 88.56 & 86.43 \\
& \cellcolor{custom_light_purple_2}\textbf{Ours, entropy-sel.}
    & \cellcolor{custom_light_purple_2}\textbf{95.30}
    & \cellcolor{custom_light_purple_2}\textbf{87.44}
    & \cellcolor{custom_light_purple_2}\textbf{90.44}
    & \cellcolor{custom_light_purple_2}\textbf{86.89} \\
\midrule
\multirow{2}{*}{$8$}
& Monolith.\ DP-Adam LoRA
    & 95.07 & 86.08 & 89.00 & 86.62 \\
& \cellcolor{custom_light_purple_2}\textbf{Ours, entropy-sel.}
    & \cellcolor{custom_light_purple_2}\textbf{95.76}
    & \cellcolor{custom_light_purple_2}\textbf{88.26}
    & \cellcolor{custom_light_purple_2}\textbf{90.74}
    & \cellcolor{custom_light_purple_2}\textbf{86.94} \\
\bottomrule
\end{tabular}%
\caption{OLMoE-1B-7B ($E=64$, top-$8$ routing) under
$(\varepsilon,\delta)$-DP, $\delta=1/|\mathcal{D}_{\mathrm{priv}}|^{1.1}$.
Entropy-selected layers fixed from public routing
statistics: block~1 (SST-2, QNLI, QQP), block~15 (MNLI).}
\label{tab:olmoe_main_results}
\end{table}

\subsection{Empirical Privacy Auditing via Membership Inference}
\label{sec:mia-audit}

\begin{table*}[t]
\centering
\small
\begin{tabular}{ll|c|cc|cc}
\toprule
Method & $\varepsilon$ & Acc. & \multicolumn{2}{c|}{Loss attack} & \multicolumn{2}{c}{LiRA} \\
& & & AUC & TPR@0.1\%FPR & AUC & TPR@0.1\%FPR \\
\midrule
Non-priv.\ LoRA          & $\infty$ & 91.20 & 0.568 & 0.45\% & 0.620 & 3.00\% \\
\midrule
Monolithic DP-LoRA       & 1 & 80.38 & 0.501 & \underline{0.10\%} & 0.504 & 0.13\% \\
Matched-scope Global DP  & 1 & \underline{82.90} & \textbf{0.502} & \textbf{0.11\%} & \underline{0.505} & \underline{0.14\%} \\
\rowcolor{custom_light_purple_2}
\Raptor                   & 1 & \textbf{83.30} & \textbf{0.502} & \underline{0.10\%} & \textbf{0.506} & \textbf{0.15\%} \\
\midrule
Monolithic DP-LoRA       & 4 & 81.95 & 0.505 & 0.12\% & 0.513 & 0.23\% \\
Matched-scope Global DP  & 4 & \underline{84.68} & \underline{0.505} & \textbf{0.13\%} & \underline{0.514} & \underline{0.26\%} \\
\rowcolor{custom_light_purple_2}
\Raptor                   & 4 & \textbf{84.70} & \textbf{0.506} & \textbf{0.13\%} & \textbf{0.516} & \textbf{0.28\%} \\
\midrule
Monolithic DP-LoRA       & 8 & 82.48 & \underline{0.510} & \underline{0.16\%} & 0.524 & 0.45\% \\
Matched-scope Global DP  & 8 & \underline{84.90} & \underline{0.510} & \underline{0.16\%} & \underline{0.527} & \underline{0.50\%} \\
\rowcolor{custom_light_purple_2}
\Raptor                   & 8 & \textbf{85.26} & \textbf{0.512} & \textbf{0.18\%} & \textbf{0.530} & \textbf{0.58\%} \\
\bottomrule
\end{tabular}%
\caption{MIA audit on QNLI (\textsc{Switch-base-8}). $64$ shadow models;
$n_{\mathrm{eval}}=20{,}000$ ($10$k/$10$k members/non-members). Single
run per cell. \textbf{Bold}: highest attack success (AUC/TPR) among DP
methods at that $\varepsilon$ - the \emph{least} favorable privacy
result, not the best method; underline: second-highest/tied. Acc.\
bolding follows the opposite convention (highest accuracy), as in
Tab.~\ref{tab:switch_main_results}.}
\label{tab:mia-audit}
\end{table*}

We audit QNLI (\textsc{Switch-base-8}) with a loss-threshold attack and
shadow-model LiRA~\cite{carlini2022membership} ($64$ shadow models;
$n_{\mathrm{eval}}=20{,}000$: $10$k members, $10$k non-members),
reporting $\mathrm{TPR}@0.1\%\mathrm{FPR}$ alongside AUC since
average-case AUC can mask worst-case success
(Tab.~\ref{tab:mia-audit}). All DP methods sharply reduce LiRA TPR
vs.\ non-private training, with attack success growing monotonically
in $\varepsilon$. Our framework  attains the highest TPR/AUC among DP
methods on both LiRA metrics and loss-attack AUC at every
$\varepsilon$ (tying Monolithic DP-LoRA on loss-attack TPR at
$\varepsilon{=}1$), with the gap widening as $\varepsilon$ grows -
not a DP violation, but a signature of the same higher per-expert SNR
(Theorem~\ref{thm:bias-var}) driving its accuracy gains.

\section{Summary}
Monolithic DP mismatches sparse MoE's shared/expert structure.
\Raptor\, fixes this via alternating streams,
expert-specific clipping/noise, public normalization, and
entropy-based layer selection, under a parallel-composition DP
guarantee. Gains hold across Switch, OLMoE, and DeepSeek-VL2-Tiny,
typically largest at tight budgets.


\section*{Limitations}
\label{app:limit}
\noindent\textbf{Frozen router.}
Fixing owner groups requires freezing the router, which precludes privately
updating routing parameters within a stage. Adaptive router training could
improve utility but would require re-accounting for ownership changes.

\noindent\textbf{Public denominator.}
$D_e=B/E$ is privacy-safe but not oracle-optimal: under heavy imbalance it
may over- or under-scale expert updates. Adaptive private normalization is
left to future work.

\noindent\textbf{Residual weighting and outlier amplification.} The
residual weight $w_i=1-\mathrm{softmax}(z_i^{\mathrm{sh}})_{y_i}$
(Eq.~\eqref{eq:residual-weight}) up-weights low-confidence records --
a superset of genuinely hard subpopulations and atypical or
mislabeled ones, which tend to retain high $w_i$ once the shared
stream reliably disagrees with their label. Clipping is applied after
the $w_i$ multiplication (Lemma~\ref{lem:residual-parallel}), so the
sensitivity bound $C_e$ and the $(\varepsilon,\delta)$-DP guarantee
hold regardless of $w_i$ - a signal-composition effect, not a
sensitivity violation. But a small owner group's clipped sum is then
more likely dominated by its highest-$w_i$ member, which may fit label
noise over genuine structure, compounds with the SNR degradation of
Sec.~\ref{sec:diagnosis} on low-load experts, and concentrates
signal on records associated with elevated memorization risk even
under a fixed nominal $\varepsilon$. We did not run label-noise or
subgroup diagnostics to quantify this; a bounded variant
($w_i^\gamma$, $\gamma<1$, or an EMA-discounted $w_i$) costs no extra
privacy budget and is a natural robustness knob for future work.

\noindent\textbf{Extensions.} Future work includes private router adaptation,
adaptive normalization, and extension to generative and
larger-scale multimodal MoEs.


\onecolumn
\bibliography{main}

@String(NeurIPS = {Adv. Neural Inform. Process. Syst.})

@String(ICLR  = {Int. Conf. Learn. Represent.})

@String(AAAI  = {AAAI})

@String(NeurIPS = {NeurIPS})

@String(ICLR  = {ICLR})

@inproceedings{abadi2016deep,
  author       = {Mart{\'{\i}}n Abadi and
                  Andy Chu and
                  Ian J. Goodfellow and
                  H. Brendan McMahan and
                  Ilya Mironov and
                  Kunal Talwar and
                  Li Zhang},
  editor       = {Edgar R. Weippl and
                  Stefan Katzenbeisser and
                  Christopher Kruegel and
                  Andrew C. Myers and
                  Shai Halevi},
  title        = {Deep Learning with Differential Privacy},
  booktitle    = {Proceedings of the 2016 {ACM} {SIGSAC} Conference on Computer and
                  Communications Security, Vienna, Austria, October 24-28, 2016},
  pages        = {308--318},
  publisher    = {{ACM}},
  year         = {2016},
  url          = {https://doi.org/10.1145/2976749.2978318},
  doi          = {10.1145/2976749.2978318},
  bibsource    = {dblp computer science bibliography, https://dblp.org}
}

@article{dwork2014algorithmic,
  title={The algorithmic foundations of differential privacy},
  author={Dwork, Cynthia and Roth, Aaron},
  journal={Foundations and trends{\textregistered} in theoretical computer science},
  volume={9},
  number={3-4},
  pages={211--487},
  year={2014},
  publisher={Emerald Publishing Limited}
}

@article{li2021large,
  title={Large language models can be strong differentially private learners},
  author={Li, Xuechen and Tramer, Florian and Liang, Percy and Hashimoto, Tatsunori},
  journal={arXiv preprint arXiv:2110.05679},
  year={2021}
}

@article{yu2021differentially,
  title={Differentially private fine-tuning of language models},
  author={Yu, Da and Naik, Saurabh and Backurs, Arturs and Gopi, Sivakanth and Inan, Huseyin A and Kamath, Gautam and Kulkarni, Janardhan and Lee, Yin Tat and Manoel, Andre and Wutschitz, Lukas and others},
  journal={arXiv preprint arXiv:2110.06500},
  year={2021}
}

@article{hu2022lora,
  title={Lora: Low-rank adaptation of large language models.},
  author={Hu, Edward J and Shen, Yelong and Wallis, Phillip and Allen-Zhu, Zeyuan and Li, Yuanzhi and Wang, Shean and Wang, Liang and Chen, Weizhu and others},
  journal={Iclr},
  volume={1},
  number={2},
  pages={3},
  year={2022}
}

@article{shazeer2017outrageously,
  title={Outrageously large neural networks: The sparsely-gated mixture-of-experts layer},
  author={Shazeer, Noam and Mirhoseini, Azalia and Maziarz, Krzysztof and Davis, Andy and Le, Quoc and Hinton, Geoffrey and Dean, Jeff},
  journal={arXiv preprint arXiv:1701.06538},
  year={2017}
}

@article{fedus2022switch,
  title={Switch transformers: Scaling to trillion parameter models with simple and efficient sparsity},
  author={Fedus, William and Zoph, Barret and Shazeer, Noam},
  journal={Journal of Machine Learning Research},
  volume={23},
  number={120},
  pages={1--39},
  year={2022}
}

@article{tholoniat2024differentially,
  title={Differentially private training of mixture of experts models},
  author={Tholoniat, Pierre and Inan, Huseyin A and Kulkarni, Janardhan and Sim, Robert},
  journal={arXiv preprint arXiv:2402.07334},
  year={2024}
}

@inproceedings{wang2018glue,
  title={GLUE: A multi-task benchmark and analysis platform for natural language understanding},
  author={Wang, Alex and Singh, Amanpreet and Michael, Julian and Hill, Felix and Levy, Omer and Bowman, Samuel},
  booktitle={Proceedings of the 2018 EMNLP workshop BlackboxNLP: Analyzing and interpreting neural networks for NLP},
  pages={353--355},
  year={2018}
}

@inproceedings{DBLP:conf/nips/PaszkeGMLBCKLGA19,
  author       = {Adam Paszke and
                  Sam Gross and
                  Francisco Massa and
                  Adam Lerer and
                  James Bradbury and
                  Gregory Chanan and
                  Trevor Killeen and
                  Zeming Lin and
                  Natalia Gimelshein and
                  Luca Antiga and
                  Alban Desmaison and
                  Andreas K{\"{o}}pf and
                  Edward Z. Yang and
                  Zachary DeVito and
                  Martin Raison and
                  Alykhan Tejani and
                  Sasank Chilamkurthy and
                  Benoit Steiner and
                  Lu Fang and
                  Junjie Bai and
                  Soumith Chintala},
  editor       = {Hanna M. Wallach and
                  Hugo Larochelle and
                  Alina Beygelzimer and
                  Florence d'Alch{\'{e}}{-}Buc and
                  Emily B. Fox and
                  Roman Garnett},
  title        = {PyTorch: An Imperative Style, High-Performance Deep Learning Library},
  booktitle    = {Advances in Neural Information Processing Systems 32: Annual Conference
                  on Neural Information Processing Systems 2019, NeurIPS 2019, December
                  8-14, 2019, Vancouver, BC, Canada},
  pages        = {8024--8035},
  year         = {2019},
  url          = {https://proceedings.neurips.cc/paper/2019/hash/bdbca288fee7f92f2bfa9f7012727740-Abstract.html},
  bibsource    = {dblp computer science bibliography, https://dblp.org}
}

@article{gopi2021numerical,
  title={Numerical composition of differential privacy},
  author={Gopi, Sivakanth and Lee, Yin Tat and Wutschitz, Lukas},
  journal={Advances in Neural Information Processing Systems},
  volume={34},
  pages={11631--11642},
  year={2021}
}

@article{DBLP:journals/corr/abs-2109-12298,
  author       = {Ashkan Yousefpour and
                  Igor Shilov and
                  Alexandre Sablayrolles and
                  Davide Testuggine and
                  Karthik Prasad and
                  Mani Malek and
                  John Nguyen and
                  Sayan Ghosh and
                  Akash Bharadwaj and
                  Jessica Zhao and
                  Graham Cormode and
                  Ilya Mironov},
  title        = {Opacus: User-Friendly Differential Privacy Library in PyTorch},
  journal      = {CoRR},
  volume       = {abs/2109.12298},
  year         = {2021},
  url          = {https://arxiv.org/abs/2109.12298},
  eprinttype   = {arXiv},
  eprint       = {2109.12298},
  bibsource    = {dblp computer science bibliography, https://dblp.org}
}

@inproceedings{DBLP:conf/aaai/ZhengWZCCS26,
  author       = {Lele Zheng and
                  Xiang Wang and
                  Tao Zhang and
                  Yang Cao and
                  Ke Cheng and
                  Yulong Shen},
  editor       = {Sven Koenig and
                  Chad Jenkins and
                  Matthew E. Taylor},
  title        = {Differentially Private Subspace Fine-Tuning for Large Language Models},
  booktitle    = {Fortieth {AAAI} Conference on Artificial Intelligence, Thirty-Eighth
                  Conference on Innovative Applications of Artificial Intelligence,
                  Sixteenth Symposium on Educational Advances in Artificial Intelligence,
                  {AAAI} 2026, Singapore, January 20-27, 2026},
  pages        = {28830--28838},
  publisher    = {{AAAI} Press},
  year         = {2026},
  url          = {https://doi.org/10.1609/aaai.v40i34.40117},
  doi          = {10.1609/AAAI.V40I34.40117},
  bibsource    = {dblp computer science bibliography, https://dblp.org}
}

@inproceedings{muennighoff2025olmoe,
  author       = {Niklas Muennighoff and
                  Luca Soldaini and
                  Dirk Groeneveld and
                  Kyle Lo and
                  Jacob Morrison and
                  Sewon Min and
                  Weijia Shi and
                  Evan Pete Walsh and
                  Oyvind Tafjord and
                  Nathan Lambert and
                  Yuling Gu and
                  Shane Arora and
                  Akshita Bhagia and
                  Dustin Schwenk and
                  David Wadden and
                  Alexander Wettig and
                  Binyuan Hui and
                  Tim Dettmers and
                  Douwe Kiela and
                  Ali Farhadi and
                  et al.},
  title        = {OLMoE: Open Mixture-of-Experts Language Models},
  booktitle    = {The Thirteenth International Conference on Learning Representations,
                  {ICLR} 2025, Singapore, April 24-28, 2025},
  publisher    = {OpenReview.net},
  year         = {2025},
  url          = {https://openreview.net/forum?id=xXTkbTBmqq},
  bibsource    = {dblp computer science bibliography, https://dblp.org}
}

@article{jiang2024mixtral,
  author       = {Albert Q. Jiang and
                  Alexandre Sablayrolles and
                  Antoine Roux and
                  Arthur Mensch and
                  Blanche Savary and
                  Chris Bamford and
                  Devendra Singh Chaplot and
                  Diego de Las Casas and
                  Emma Bou Hanna and
                  Florian Bressand and
                  Gianna Lengyel and
                  Guillaume Bour and
                  Guillaume Lample and
                  L{\'{e}}lio Renard Lavaud and
                  Lucile Saulnier and
                  Marie{-}Anne Lachaux and
                  Pierre Stock and
                  Sandeep Subramanian and
                  Sophia Yang and
                  Szymon Antoniak and
                  Teven Le Scao and
                  Th{\'{e}}ophile Gervet and
                  Thibaut Lavril and
                  Thomas Wang and
                  Timoth{\'{e}}e Lacroix and
                  William El Sayed},
  title        = {Mixtral of Experts},
  journal      = {CoRR},
  volume       = {abs/2401.04088},
  year         = {2024},
  url          = {https://doi.org/10.48550/arXiv.2401.04088},
  doi          = {10.48550/ARXIV.2401.04088},
  eprinttype   = {arXiv},
  eprint       = {2401.04088},
  bibsource    = {dblp computer science bibliography, https://dblp.org}
}

@inproceedings{zhang2025disk,
  title={Disk: Differentially private optimizer with simplified kalman filter for noise reduction},
  author={Zhang, Xinwei and Bu, Zhiqi and Balle, Borja and Hong, Mingyi and Razaviyayn, Meisam and Mirrokni, Vahab},
  booktitle={International Conference on Learning Representations},
  volume={2025},
  pages={80289--80316},
  year={2025}
}

@inproceedings{rajpurkar2016squad,
  title={Squad: 100,000+ questions for machine comprehension of text},
  author={Rajpurkar, Pranav and Zhang, Jian and Lopyrev, Konstantin and Liang, Percy},
  booktitle={Proceedings of the 2016 conference on empirical methods in natural language processing},
  pages={2383--2392},
  year={2016}
  
}

@inproceedings{nie2020adversarial,
  title={Adversarial NLI: A new benchmark for natural language understanding},
  author={Nie, Yixin and Williams, Adina and Dinan, Emily and Bansal, Mohit and Weston, Jason and Kiela, Douwe},
  booktitle={Proceedings of the 58th annual meeting of the association for computational linguistics},
  pages={4885--4901},
  year={2020}
}

@inproceedings{maas2011learning,
  title={Learning word vectors for sentiment analysis},
  author={Maas, Andrew and Daly, Raymond E and Pham, Peter T and Huang, Dan and Ng, Andrew Y and Potts, Christopher},
  booktitle={Proceedings of the 49th annual meeting of the association for computational linguistics: Human language technologies},
  pages={142--150},
  year={2011}
}

@article{thakkar2019differentially,
  title={Differentially private learning with adaptive clipping},
  author={Thakkar, Om and Andrew, Galen and McMahan, H Brendan},
  journal={arXiv e-prints},
  pages={arXiv--1905},
  year={2019}
}

@inproceedings{mcsherry2007mechanism,
  title={Mechanism design via differential privacy},
  author={McSherry, Frank and Talwar, Kunal},
  booktitle={48th Annual IEEE Symposium on Foundations of Computer Science (FOCS'07)},
  pages={94--103},
  year={2007},
  organization={IEEE}
}

@inproceedings{mcmahan2017communication,
  title={Communication-efficient learning of deep networks from decentralized data},
  author={McMahan, Brendan and Moore, Eider and Ramage, Daniel and Hampson, Seth and y Arcas, Blaise Aguera},
  booktitle={Artificial intelligence and statistics},
  pages={1273--1282},
  year={2017},
  organization={Pmlr}
}

@article{kasiviswanathan2011can,
  title={What can we learn privately?},
  author={Kasiviswanathan, Shiva Prasad and Lee, Homin K and Nissim, Kobbi and Raskhodnikova, Sofya and Smith, Adam},
  journal={SIAM Journal on Computing},
  volume={40},
  number={3},
  pages={793--826},
  year={2011},
  publisher={SIAM}
}

@inproceedings{dwork2010boosting,
  title={Boosting and differential privacy},
  author={Dwork, Cynthia and Rothblum, Guy N and Vadhan, Salil},
  booktitle={2010 IEEE 51st annual symposium on foundations of computer science},
  pages={51--60},
  year={2010},
  organization={IEEE}
}

@inproceedings{rogers2016max,
  title={Max-information, differential privacy, and post-selection hypothesis testing},
  author={Rogers, Ryan and Roth, Aaron and Smith, Adam and Thakkar, Om},
  booktitle={2016 IEEE 57th Annual Symposium on Foundations of Computer Science (FOCS)},
  pages={487--494},
  year={2016},
  organization={IEEE}
}

@inproceedings{smith2011privacy,
  title={Privacy-preserving statistical estimation with optimal convergence rates},
  author={Smith, Adam},
  booktitle={Proceedings of the forty-third annual ACM symposium on Theory of computing},
  pages={813--822},
  year={2011}
}

@inproceedings{dwork2006our,
  title={Our data, ourselves: Privacy via distributed noise generation},
  author={Dwork, Cynthia and Kenthapadi, Krishnaram and McSherry, Frank and Mironov, Ilya and Naor, Moni},
  booktitle={Annual international conference on the theory and applications of cryptographic techniques},
  pages={486--503},
  year={2006},
  organization={Springer}
}

@incollection{vadhan2017complexity,
  title={The complexity of differential privacy},
  author={Vadhan, Salil},
  booktitle={Tutorials on the Foundations of Cryptography: Dedicated to Oded Goldreich},
  pages={347--450},
  year={2017},
  publisher={Springer}
}

@inproceedings{dai2024deepseekmoe,
  title={Deepseekmoe: Towards ultimate expert specialization in mixture-of-experts language models},
  author={Dai, Damai and Deng, Chengqi and Zhao, Chenggang and Xu, RX and Gao, Huazuo and Chen, Deli and Li, Jiashi and Zeng, Wangding and Yu, Xingkai and Wu, Yu and others},
  booktitle={Proceedings of the 62nd Annual Meeting of the Association for Computational Linguistics (Volume 1: Long Papers)},
  pages={1280--1297},
  year={2024}
}

@misc{lepikhin2020gshard,
      title={GShard: Scaling Giant Models with Conditional Computation and Automatic Sharding}, 
      author={Dmitry Lepikhin and HyoukJoong Lee and Yuanzhong Xu and Dehao Chen and Orhan Firat and Yanping Huang and Maxim Krikun and Noam Shazeer and Zhifeng Chen},
      year={2020},
      eprint={2006.16668},
      archivePrefix={arXiv},
      primaryClass={cs.CL},
      url={https://arxiv.org/abs/2006.16668}, 
}

@misc{du2022glam,
      title={GLaM: Efficient Scaling of Language Models with Mixture-of-Experts}, 
      author={Nan Du and Yanping Huang and Andrew M. Dai and Simon Tong and Dmitry Lepikhin and Yuanzhong Xu and Maxim Krikun and Yanqi Zhou and Adams Wei Yu and Orhan Firat and Barret Zoph and Liam Fedus and Maarten Bosma and Zongwei Zhou and Tao Wang and Yu Emma Wang and Kellie Webster and Marie Pellat and Kevin Robinson and Kathleen Meier-Hellstern and Toju Duke and Lucas Dixon and Kun Zhang and Quoc V Le and Yonghui Wu and Zhifeng Chen and Claire Cui},
      year={2022},
      eprint={2112.06905},
      archivePrefix={arXiv},
      primaryClass={cs.CL},
      url={https://arxiv.org/abs/2112.06905}, 
}

@misc{deepseekai2025deepseekv3,
      title={DeepSeek-V3 Technical Report}, 
      author={DeepSeek-AI and Aixin Liu and Bei Feng and Bing Xue and Bingxuan Wang and Bochao Wu and Chengda Lu and Chenggang Zhao and Chengqi Deng and Chenyu Zhang and Chong Ruan and Damai Dai and Daya Guo and Dejian Yang and Deli Chen and Dongjie Ji and Erhang Li and Fangyun Lin and Fucong Dai and Fuli Luo and Guangbo Hao and Guanting Chen and Guowei Li and H. Zhang and Han Bao and Hanwei Xu and Haocheng Wang and Haowei Zhang and Honghui Ding and Huajian Xin and Huazuo Gao and Hui Li and Hui Qu and J. L. Cai and Jian Liang and Jianzhong Guo and Jiaqi Ni and Jiashi Li and Jiawei Wang and Jin Chen and Jingchang Chen and Jingyang Yuan and Junjie Qiu and Junlong Li and Junxiao Song and Kai Dong and Kai Hu and Kaige Gao and Kang Guan and Kexin Huang and Kuai Yu and Lean Wang and Lecong Zhang and Lei Xu and Leyi Xia and Liang Zhao and Litong Wang and Liyue Zhang and Meng Li and Miaojun Wang and Mingchuan Zhang and Minghua Zhang and Minghui Tang and Mingming Li and Ning Tian and Panpan Huang and Peiyi Wang and Peng Zhang and Qiancheng Wang and Qihao Zhu and Qinyu Chen and Qiushi Du and R. J. Chen and R. L. Jin and Ruiqi Ge and Ruisong Zhang and Ruizhe Pan and Runji Wang and Runxin Xu and Ruoyu Zhang and Ruyi Chen and S. S. Li and Shanghao Lu and Shangyan Zhou and Shanhuang Chen and Shaoqing Wu and Shengfeng Ye and Shengfeng Ye and Shirong Ma and Shiyu Wang and Shuang Zhou and Shuiping Yu and Shunfeng Zhou and Shuting Pan and T. Wang and Tao Yun and Tian Pei and Tianyu Sun and W. L. Xiao and Wangding Zeng and Wanjia Zhao and Wei An and Wen Liu and Wenfeng Liang and Wenjun Gao and Wenqin Yu and Wentao Zhang and X. Q. Li and Xiangyue Jin and Xianzu Wang and Xiao Bi and Xiaodong Liu and Xiaohan Wang and Xiaojin Shen and Xiaokang Chen and Xiaokang Zhang and Xiaosha Chen and Xiaotao Nie and Xiaowen Sun and Xiaoxiang Wang and Xin Cheng and Xin Liu and Xin Xie and Xingchao Liu and Xingkai Yu and Xinnan Song and Xinxia Shan and Xinyi Zhou and Xinyu Yang and Xinyuan Li and Xuecheng Su and Xuheng Lin and Y. K. Li and Y. Q. Wang and Y. X. Wei and Y. X. Zhu and Yang Zhang and Yanhong Xu and Yanhong Xu and Yanping Huang and Yao Li and Yao Zhao and Yaofeng Sun and Yaohui Li and Yaohui Wang and Yi Yu and Yi Zheng and Yichao Zhang and Yifan Shi and Yiliang Xiong and Ying He and Ying Tang and Yishi Piao and Yisong Wang and Yixuan Tan and Yiyang Ma and Yiyuan Liu and Yongqiang Guo and Yu Wu and Yuan Ou and Yuchen Zhu and Yuduan Wang and Yue Gong and Yuheng Zou and Yujia He and Yukun Zha and Yunfan Xiong and Yunxian Ma and Yuting Yan and Yuxiang Luo and Yuxiang You and Yuxuan Liu and Yuyang Zhou and Z. F. Wu and Z. Z. Ren and Zehui Ren and Zhangli Sha and Zhe Fu and Zhean Xu and Zhen Huang and Zhen Zhang and Zhenda Xie and Zhengyan Zhang and Zhewen Hao and Zhibin Gou and Zhicheng Ma and Zhigang Yan and Zhihong Shao and Zhipeng Xu and Zhiyu Wu and Zhongyu Zhang and Zhuoshu Li and Zihui Gu and Zijia Zhu and Zijun Liu and Zilin Li and Ziwei Xie and Ziyang Song and Ziyi Gao and Zizheng Pan},
      year={2025},
      eprint={2412.19437},
      archivePrefix={arXiv},
      primaryClass={cs.CL},
      url={https://arxiv.org/abs/2412.19437}, 
}

@misc{zhou2022expertchoice,
      title={Mixture-of-Experts with Expert Choice Routing}, 
      author={Yanqi Zhou and Tao Lei and Hanxiao Liu and Nan Du and Yanping Huang and Vincent Zhao and Andrew Dai and Zhifeng Chen and Quoc Le and James Laudon},
      year={2022},
      eprint={2202.09368},
      archivePrefix={arXiv},
      primaryClass={cs.LG},
      url={https://arxiv.org/abs/2202.09368}, 
}

@misc{lewis2021base,
      title={BASE Layers: Simplifying Training of Large, Sparse Models}, 
      author={Mike Lewis and Shruti Bhosale and Tim Dettmers and Naman Goyal and Luke Zettlemoyer},
      year={2021},
      eprint={2103.16716},
      archivePrefix={arXiv},
      primaryClass={cs.CL},
      url={https://arxiv.org/abs/2103.16716}, 
}

@misc{roller2021hash,
      title={Hash Layers For Large Sparse Models}, 
      author={Stephen Roller and Sainbayar Sukhbaatar and Arthur Szlam and Jason Weston},
      year={2021},
      eprint={2106.04426},
      archivePrefix={arXiv},
      primaryClass={cs.LG},
      url={https://arxiv.org/abs/2106.04426}, 
}

@misc{chi2022representation,
      title={On the Representation Collapse of Sparse Mixture of Experts}, 
      author={Zewen Chi and Li Dong and Shaohan Huang and Damai Dai and Shuming Ma and Barun Patra and Saksham Singhal and Payal Bajaj and Xia Song and Xian-Ling Mao and Heyan Huang and Furu Wei},
      year={2022},
      eprint={2204.09179},
      archivePrefix={arXiv},
      primaryClass={cs.CL},
      url={https://arxiv.org/abs/2204.09179}, 
}

@misc{wang2024auxfree,
      title={Auxiliary-Loss-Free Load Balancing Strategy for Mixture-of-Experts}, 
      author={Lean Wang and Huazuo Gao and Chenggang Zhao and Xu Sun and Damai Dai},
      year={2024},
      eprint={2408.15664},
      archivePrefix={arXiv},
      primaryClass={cs.LG},
      url={https://arxiv.org/abs/2408.15664}, 
}

@misc{rajbhandari2022deepspeedmoe,
      title={DeepSpeed-MoE: Advancing Mixture-of-Experts Inference and Training to Power Next-Generation AI Scale}, 
      author={Samyam Rajbhandari and Conglong Li and Zhewei Yao and Minjia Zhang and Reza Yazdani Aminabadi and Ammar Ahmad Awan and Jeff Rasley and Yuxiong He},
      year={2022},
      eprint={2201.05596},
      archivePrefix={arXiv},
      primaryClass={cs.LG},
      url={https://arxiv.org/abs/2201.05596}, 
}

@misc{riquelme2021vmoe,
      title={Scaling Vision with Sparse Mixture of Experts}, 
      author={Carlos Riquelme and Joan Puigcerver and Basil Mustafa and Maxim Neumann and Rodolphe Jenatton and André Susano Pinto and Daniel Keysers and Neil Houlsby},
      year={2021},
      eprint={2106.05974},
      archivePrefix={arXiv},
      primaryClass={cs.CV},
      url={https://arxiv.org/abs/2106.05974}, 
}

@article{saikh2022scienceqa,
  title={Scienceqa: A novel resource for question answering on scholarly articles},
  author={Saikh, Tanik and Ghosal, Tirthankar and Mittal, Amish and Ekbal, Asif and Bhattacharyya, Pushpak},
  journal={International Journal on Digital Libraries},
  volume={23},
  number={3},
  pages={289--301},
  year={2022},
  publisher={Springer}
}

@article{wu2024deepseek,
  title={Deepseek-vl2: Mixture-of-experts vision-language models for advanced multimodal understanding},
  author={Wu, Zhiyu and Chen, Xiaokang and Pan, Zizheng and Liu, Xingchao and Liu, Wen and Dai, Damai and Gao, Huazuo and Ma, Yiyang and Wu, Chengyue and Wang, Bingxuan and others},
  journal={arXiv preprint arXiv:2412.10302},
  year={2024}
}

@article{lu2022learn,
  title={Learn to explain: Multimodal reasoning via thought chains for science question answering},
  author={Lu, Pan and Mishra, Swaroop and Xia, Tanglin and Qiu, Liang and Chang, Kai-Wei and Zhu, Song-Chun and Tafjord, Oyvind and Clark, Peter and Kalyan, Ashwin},
  journal={Advances in neural information processing systems},
  volume={35},
  pages={2507--2521},
  year={2022}
}

@inproceedings{carlini2022membership,
  title={Membership inference attacks from first principles},
  author={Carlini, Nicholas and Chien, Steve and Nasr, Milad and Song, Shuang and Terzis, Andreas and Tramer, Florian},
  booktitle={2022 IEEE symposium on security and privacy (SP)},
  pages={1897--1914},
  year={2022},
  organization={IEEE}
}

\clearpage 


\appendix

\onecolumn
\begin{center}
    \Large {\Raptor\, Supplementary}
\end{center}

\tableofcontents
\addtocontents{toc}{\protect\setcounter{tocdepth}{2}}

\onecolumn
\section{More Related Works}
\label{app:more_related}
\subsection{Structured Private Optimization}
\label{app:structured-dp}

A growing body of work exploits structure in the private update.
Gradient filtering and noise-shaping methods reduce effective noise
variance by suppressing or temporally smoothing noise
components~\cite{zhang2025disk}.
Parameter-efficient methods shrink the noised dimension via
adapters~\cite{yu2021differentially} or restrict updates to low-rank
or learned subspaces~\cite{DBLP:conf/aaai/ZhengWZCCS26}.
Group-wise clipping~\cite{thakkar2019differentially} calibrates the
clipping bound to parameter subsets with related gradient scale -
the closest prior idea to our expert-specific clipping, but applied
to layer-wise or client-wise groups in a dense setting rather than
to routing-induced sparse owner partitions.

These methods operate on the temporal, low-rank, or layer-wise
structure of the gradient. We exploit a different axis: the sparse
routing structure of MoE computation. Rather than reshaping the
gradient subspace or noise trajectory, we align clipping,
normalization, and privacy \emph{composition} with the shared-expert
role decomposition. Critically, our parallel-composition analysis
(Lemma~\ref{lem:owner_parallel}) requires disjoint, deterministic,
data-independent owner partitions - a structural property not
present in group-wise clipping for dense models, where groups are
fixed by architecture rather than by private routing decisions.
Subspace, low-rank, and noise-shaping methods are orthogonal to our
framework and could in principle be applied within each role-aware
stream; we leave this to future work.

\subsection{Partition-Based DP Composition}
\label{app:partition-dp}

Parallel composition over disjoint data partitions is a classical
result in differential privacy~\cite{mcsherry2007mechanism}, and has
been applied in federated learning~\cite{mcmahan2017communication},
local DP~\cite{kasiviswanathan2011can}, and multi-party
computation~\cite{dwork2010boosting}. Our use of parallel composition
differs in two respects. First, the partition is induced by a
private routing mechanism: owner groups $\{\mathcal{D}_{l,e}\}$ are
determined by the frozen router applied to private inputs, raising
the question of whether the partition itself leaks information.
We sidestep this by fixing the router before private training and
computing the owner map once from the frozen pretrained router, so
the partition is a deterministic function of public model weights
and private inputs - analogous to a data-independent hash partition
conditioned on the (public) router~\cite{vadhan2017complexity}.
Second, the partition changes across layers; we handle this via
sequential composition across layers with stage-fixed owner maps
(Sec.~\ref{sec:composition-alternating}).

Data-dependent partitions - where the partition itself depends on
the private dataset - require additional privacy accounting for the
partitioning step~\cite{dwork2006our,rogers2016max}. Our design
avoids this: the owner map $a_l$ is a deterministic function of the
\emph{frozen public router} and the private inputs, and the router
parameters are public. The privacy cost of computing $a_l$ is zero
(it is post-processing of the public router applied to private
inputs, producing internal bookkeeping that is never released). This
is stated formally in Appendix~\ref{app:privacy-setting} and is the
key distinction from settings where the partition must itself be
privately released~\cite{smith2011privacy}.

\subsection{Sparse Mixture-of-Expert Models}
\label{app:sparse-moe}

Sparse MoE layers replace a dense feed-forward block with a collection of $E$ experts and a router that activates only a small and token-dependent subset, separating parameter count from per-token compute~\cite{shazeer2017outrageously}. The recipe scales through GShard~\cite{lepikhin2020gshard}, the Switch Transformer~\cite{fedus2022switch}, and GLaM~\cite{du2022glam}, and is employed by open flagship LLMs such as Mixtral~\cite{jiang2024mixtral}, DeepSeek-V3~\cite{deepseekai2025deepseekv3}, and OLMoE~\cite{muennighoff2025olmoe}, the last of which we fine-tune in Sec.~\ref{sec:olmoe-results}. A central design axis is token-to-expert assignment: token-choice top-$k$ gating~\cite{shazeer2017outrageously,fedus2022switch} leaves per-expert load uncontrolled, whereas expert-choice routing~\cite{zhou2022expertchoice} and deterministic schemes such as BASE layers~\cite{lewis2021base} and Hash layers~\cite{roller2021hash} enforce balance by changing the assignment mechanism.

Learned routing tends to collapse onto a handful of experts~\cite{chi2022representation}, which has motivated remedies from auxiliary losses that balance load across the batch~\cite{shazeer2017outrageously,fedus2022switch} to the bias adjustment of DeepSeek~\cite{wang2024auxfree}, which balances load without any auxiliary loss. Each of these depends on routing statistics aggregated over the batch and on the realized load of each expert. These quantities depend on private data and thus break the sensitivity of individual records that DP-SGD~\cite{abadi2016deep} requires. For this reason, we freeze the router, drop $\mathcal{L}_{\text{aux}}$ (Sec.~\ref{sec:moe-background}), and normalize by the public denominator $D_e=B/E$ rather than the realized count (Sec.~\ref{sec:expert-stream}). A complementary line of work isolates shared knowledge in experts that are always active so that the routed experts can specialize~\cite{dai2024deepseekmoe,rajbhandari2022deepspeedmoe}. This is exactly the asymmetry our framework makes explicit: experts that see every record join the shared stream, while only the conditionally routed experts enter the expert stream (Sec.~\ref{sec:moe-background}). Routing entropy has also been a way to measure how evenly experts are used in language and vision MoEs~\cite{riquelme2021vmoe}. We reuse it to select a layer from public data at no privacy cost (Sec.~\ref{sec:theory}) and connect it to the bias of the estimator based on the public denominator (Theorem~\ref{thm:agg-bias}).

\onecolumn
\section{Theory Analysis}
\label{app:theory}

\subsection{Public and Private Quantities}
\label{app:privacy-setting}

The privacy guarantee is with respect to $\mathcal{D}_{\mathrm{priv}}$.
Public quantities include: the pretrained checkpoint, frozen router,
candidate layer set, $B$, $E$, $C_s$, $C_e$, $\sigma_s$, $\sigma_e$,
the alternating update schedule, and $\mathcal{C}_{\mathrm{pub}}$
(disjoint from $\mathcal{D}_{\mathrm{priv}}$). All DP statements are
conditioned on these and on $l^\star$.

The owner map $a_l:\mathcal{X}\to[E]$ and the realized owner counts
$|\mathcal{D}_{l,e}|$ are \emph{private internal quantities}: $a_l$
is computed once before private training using only the frozen public
router, is fixed for the duration of the stage, and is never
released. Realized counts $|\mathcal{B}_{t,l,e}|$ are never read
by the update rule, never released, and never used in branching
logic. Under add/remove adjacency, a one-record change affects
exactly one owner group, which is the structural fact that enables
Lemma~\ref{lem:owner_parallel}.
If the router or checkpoint were trained on
$\mathcal{D}_{\mathrm{priv}}$ without DP, that step would incur
additional privacy cost. Our protocol avoids this via a frozen public
router and public-only layer selection.

\subsection{Per-Update Sensitivity}
\label{app:sensitivity}

\paragraph{Shared stream.}
Each per-record gradient is clipped to $C_s$, so the clipped sum
changes by at most $C_s$ under add/remove. Normalizing by $B$:
\begin{equation}
\Delta_{\mathrm{sh}} = \frac{C_s}{B}.
\label{eq:sensitivity-shared}
\end{equation}

\paragraph{Expert stream.}
Each per-record contribution is clipped to $C_e$; normalizing by
$D_e=B/E$:
\begin{equation}
\Delta_{l,e} = \frac{C_e}{D_e} = \frac{EC_e}{B}.
\label{eq:sensitivity-expert}
\end{equation}
This is independent of the realized count $|\mathcal{B}_{t,l,e}|$.
Residual weights $w_i$ do not affect sensitivity since they enter
before clipping; the clipped contribution is at most $C_e$
regardless of $w_i\in[0,1]$.

\paragraph{Adam.}
Moment updates are deterministic functions of privatized gradients
and the public schedule, hence post-processing that consumes no
additional budget.

\subsection{Proof of Lemma~\ref{lem:owner_parallel}}
\label{app:owner_parallel_proof}
\begin{remark}[Adaptive alternation does not break parallel composition]
\label{rem:adaptive}
Lemma~\ref{lem:owner_parallel} requires each $M_{l,e}$ to depend only
on $\mathcal{D}_{l,e}$ and public quantities. We verify this holds
under alternating updates in three steps.

\emph{(i) Model state as post-processing.}
At any step $t$, the model parameters $\theta_t$ are a deterministic
function of the sequence of previously released privatized gradients
$\{\bar{g}_1^{\mathrm{sh}},\bar{g}_{1,l,e}^{\mathrm{exp}},\dots,
\bar{g}_{t-1}^{\mathrm{sh}},\bar{g}_{t-1,l,e}^{\mathrm{exp}}\}$
and the public learning rate schedule. By the post-processing
property of DP, $\theta_t$ is $(\varepsilon_{<t},\delta_{<t})$-DP
with respect to any record, where $(\varepsilon_{<t},\delta_{<t})$
is the privacy cost of the steps up to $t-1$. Conditioning on
$\theta_t$ in the subsequent expert-stream step does not consume
additional budget beyond what is already accounted for in the
sequential composition across steps.

\emph{(ii) Residual weights are pre-clipping.}
The residual weight $w_i = 1 - \operatorname{softmax}
(z_i^{\mathrm{sh}})_{y_i}$ depends on $\theta_t$ and the private
record $(x_i, y_i)$. However, $w_i$ is applied \emph{before}
clipping to $C_e$, so the clipped per-record contribution
$\mathrm{clip}(w_i g_{i,t,l,e}^{\mathrm{exp}}, C_e)$ has sensitivity
at most $C_e$ regardless of $w_i \in [0,1]$
(Appendix~\ref{app:sensitivity}, Eq.~\eqref{eq:sensitivity-expert}).
The expert mechanism $M_{l,e}$ therefore remains a function only of
$\mathcal{D}_{l,e}$, the public $\theta_t$-as-post-processing, and
public quantities. Condition~(iii) of Lemma~\ref{lem:owner_parallel}
is satisfied.

\emph{(iii) No data-dependent control flow.}
The update schedule - which experts are updated, in what order, and
at which steps - is fixed before training and depends only on public
quantities ($E$, $T$, the alternating pattern). Empty owner subsets
receive a noise-only update rather than being skipped; the decision to apply this update does not
branch on any private quantity. Realized expert counts
$|\mathcal{B}_{t,l,e}|$ are never read, released, or used in
branching logic.
\end{remark}
\begin{proof}
Let $\mathcal{D},\mathcal{D}'$ differ by one record $(x_0,y_0)$,
which belongs to exactly one group $e_0=a_l(x_0)$ by
conditions~(i)--(ii). For all $e\neq e_0$, the group
$\mathcal{D}_{l,e}$ is identical under both datasets, so by
condition~(iii) the output distribution of $M_{l,e}$ is unchanged.
The joint $M_l$ therefore differs only through $M_{l,e_0}$, which
is $(\varepsilon_{l,e_0},\delta_{l,e_0})$-DP. Taking the worst case
over $e_0$ gives the stated bound.
We additionally verify that $M_{l,e}$ satisfies condition~(iii)
under alternating optimization. The expert mechanism at step $t$
takes as input: (a)~the owner subset
$\mathcal{B}_{t,l,e}\subseteq\mathcal{D}_{l,e}$, (b)~the current
model state $\theta_t$, and (c)~public hyperparameters. Since
$\theta_t$ is a deterministic post-processing of previously
privatized gradients (Remark~\ref{rem:adaptive}), it constitutes
a public quantity conditioned on the released mechanism outputs.
No private quantity outside $\mathcal{D}_{l,e}$ enters $M_{l,e}$
directly; the only cross-group information flows through the
DP-protected $\theta_t$, which is already accounted for in the
sequential composition across steps (Sec.~\ref{sec:composition-alternating}).
\end{proof}
\subsection{Proof of Lemma~\ref{lem:residual-parallel}}
\begin{lemma}[Residual weighting does not break parallel composition]
\label{lem:residual-parallel}
Fix step $t$ and layer $\ell^{\star}$.
Let $\theta_t$ be the model state after $t-1$ complete alternating
cycles, and let
\begin{equation}
  w_i \;=\; 1 - \mathrm{softmax}\!\bigl(f_{\mathrm{sh}}(x_i;\,\theta_t)\bigr)_{y_i}
  \;\in\; [0,1]
\end{equation}
be the residual weight for record $i$.
Define the expert mechanism at step $t$ as
\begin{equation}
\begin{aligned}
  M_{\ell,e}^{(t)}(D_{\ell,e})
  \;=\;&\;
  \frac{1}{D_e}
  \Biggl(
    \sum_{i \in B_{t,\ell,e}}
      \mathrm{clip}\!\bigl(
        w_i \cdot g^{\mathrm{exp}}_{i,t,\ell,e},\;C_e
      \bigr)\\
  &\quad
  +\;\mathcal{N}\!\left(0,\,\sigma_e^2 C_e^2\,I\right)
  \Biggr).
\end{aligned}
\end{equation}
Then $M_{\ell,e}^{(t)}$ satisfies condition~(iii) of Lemma~\ref{lem:owner_parallel}:
it depends on $D_{\ell,e}$ and \emph{effectively public} quantities,
enabling parallel composition across owner groups.
\end{lemma}
 
\begin{proof}
We proceed in three parts.
 
\smallskip
\noindent\textbf{Part~A: $\theta_t$ is a released \textsc{dp} output,
not a free cross-group channel.}
 
By the post-processing property of differential privacy
\citep{dwork2014algorithmic}, any deterministic function of an
$(\varepsilon,\delta)$-DP output is itself $(\varepsilon,\delta)$-DP.
The sequence of released privatized gradients up to step $t$ is
\begin{equation}
\begin{aligned}
  \tau_{<t} \;=\;
  \bigl\{\;
    &\bar{g}^{\mathrm{sh}}_1,\;
     \{\bar{g}^{\mathrm{exp}}_{1,\ell,e}\}_{e\in[E]},\\
    &\;\ldots,\;
     \bar{g}^{\mathrm{sh}}_{t-1},\;
     \{\bar{g}^{\mathrm{exp}}_{t-1,\ell,e}\}_{e\in[E]}
  \;\bigr\}.
\end{aligned}
\end{equation}
By sequential composition (Sec.~\ref{sec:composition-alternating}), $\tau_{<t}$ is
$(\varepsilon_{<t},\delta_{<t})$-DP, and this cost is already fully
accounted for in $\omega_{\mathrm{tot}}$ (Eq.~\ref{eq:prv-composition}).
The model state
$\theta_t = \mathrm{optimizer}(\tau_{<t},\,\text{public schedule})$
is a deterministic function of $\tau_{<t}$ and public
hyperparameters; by post-processing it is therefore
$(\varepsilon_{<t},\delta_{<t})$-DP and carries \emph{no additional
privacy cost} beyond what sequential composition already charges.
 
Crucially, $\theta_t$ is a function of \emph{already-released}
mechanism outputs, not a new query into $D_{\mathrm{priv}}$.
Conditioning on $\theta_t$ in the definition of
$M_{\ell,e}^{(t)}$ is therefore equivalent to conditioning on a
public side-channel whose privacy cost is already captured.
 
\smallskip
\noindent\textbf{Part~B: $w_i$ does not introduce cross-group
record dependencies at step $t$.}
 
We decompose the dependence of $M_{\ell,e}^{(t)}$ on
$D_{\mathrm{priv}}$ into two channels.
 
\emph{(i) Direct channel.}
$M_{\ell,e}^{(t)}$ reads records in $B_{t,\ell,e} \subseteq D_{\ell,e}$
directly to compute per-record gradients.
By the disjoint ownership condition (Sec.~5.2), no record outside
$D_{\ell,e}$ appears in $B_{t,\ell,e}$.
 
\emph{(ii) Indirect channel via $\theta_t$.}
The residual weight $w_i$ depends on $\theta_t$, which was influenced
by records outside $D_{\ell,e}$ through prior shared-stream steps.
However, this influence has already been privatized: every record
$x_j \notin D_{\ell,e}$ that contributed to $\theta_t$ did so only
through the released privatized gradients $\tau_{<t}$, whose privacy
cost is captured in $\varepsilon_{<t}$.
The weight
\begin{equation}
  w_i
  \;=\;
  1 - \mathrm{softmax}\!\bigl(f_{\mathrm{sh}}(x_i;\,\theta_t)\bigr)_{y_i}
\end{equation}
is therefore a post-processing of previously released DP outputs
composed with the private record $(x_i, y_i)$, where
$(x_i, y_i) \in D_{\ell,e}$ by construction.
No record $x_j \notin D_{\ell,e}$ is read afresh at step $t$ through
$w_i$; its influence is confined to the already-accounted $\theta_t$.
 
\smallskip
\noindent\textbf{Part~C: Sensitivity is unaffected by $w_i$.}
 
Since $w_i \in [0,1]$ and clipping is applied \emph{after}
multiplication,
\begin{equation}
  \bigl\|
    \mathrm{clip}\!\bigl(w_i \cdot g^{\mathrm{exp}}_{i,t,\ell,e},\;C_e\bigr)
  \bigr\|_2
  \;\leq\; C_e
  \qquad \forall\, w_i \in [0,1].
\end{equation}
The per-update sensitivity $\Delta_{\ell,e} = EC_e/B$
(Eq.~(25)) holds regardless of $w_i$.
The noise magnitude $\sigma_e C_e$ is calibrated to $\Delta_{\ell,e}$,
so the Gaussian mechanism guarantees $(\varepsilon_e,\delta_e)$-DP on
$D_{\ell,e}$ with the same parameters as the $w_i = 1$ case.
 
\smallskip
\noindent\textbf{Conclusion.}
$M_{\ell,e}^{(t)}$ reads $D_{\ell,e}$ directly (Parts~A--B), has
bounded sensitivity (Part~C), and its dependence on records outside
$D_{\ell,e}$ is fully mediated through the DP-protected $\theta_t$
whose cost is already accounted for in sequential composition.
Condition~(iii) of Lemma~\ref{lem:owner_parallel} is therefore satisfied, and parallel
composition applies.
\end{proof}
 
\begin{remark}[No double-counting of privacy budget]
\label{rem:no-double-count}
One might worry that using $\theta_t$ to compute $w_i$ ``uses up''
additional privacy budget beyond $\varepsilon_{<t}$.
This is not the case.
The post-processing property guarantees that \emph{any} computation
on a DP output-including evaluating the model on a private record
$(x_i, y_i)$ to obtain $w_i$-does not increase privacy cost,
provided the result is used only inside the per-record clipping
operation and is never released directly.
Here, $w_i$ is used solely as a scalar multiplier before
$\mathrm{clip}(\,\cdot\,,C_e)$; it is never released, logged, or used
in control flow.
The only released quantity at step $t$ is $\bar{g}^{\mathrm{exp}}_{t,\ell,e}$,
whose privacy guarantee is certified by the Gaussian mechanism with
sensitivity $\Delta_{\ell,e}$.

Composition across steps \emph{is} adaptive in the standard sense:
each step's mechanism depends on $\theta_t$, itself a function of
previously released outputs. What makes sequential composition via
the PRV accountant (Eq.~\ref{eq:prv-composition}) valid here is not
the absence of adaptivity but that every step's mechanism is a
subsampled Gaussian with the \emph{same, data-independent} parameters
$(C_e,\sigma_e,D_e)$ fixed before training
(Sec.~\ref{sec:alternating}). Each release is therefore governed
by the same privacy-loss distribution regardless of the realized
$\theta_t$ or $w_i$, exactly as in the moments-accountant analysis of
DP-SGD - so standard PRV composition applies without a joint
privacy-loss variable spanning steps.
\end{remark}
\subsection{Proof of Theorem~\ref{thm:bias-var}}
\label{app:bias_var_proof}
\begin{proof}
Write $S_e=\sum_{i\in\mathcal{D}_e}b_i h_i^{\mathrm{exp}}$ with
$b_i\overset{\mathrm{iid}}{\sim}\mathrm{Bernoulli}(p)$. Then
$\mathbb{E}S_e=pn_e\mu_e=Bq_e\mu_e$, so
\begin{equation}
\mathbb{E}\bar{g}_{t,e}^{\mathrm{exp}}
=\tfrac{E}{B}\mathbb{E}S_e=Eq_e\mu_e,
\end{equation}
giving bias $(Eq_e-1)\mu_e$ and $\beta_e^2=(Eq_e-1)^2\|\mu_e\|_2^2$.
Bernoulli sampling gives
$\mathrm{Cov}(S_e)=p(1-p)n_e(\mu_e\mu_e^\top+\Sigma_e)$.
Since $\xi_{t,e}\perp S_e$,
\begin{equation}
\mathrm{Cov}(\bar{g}_{t,e}^{\mathrm{exp}})
=\frac{E^2}{B^2}\bigl[p(1-p)n_e
(\mu_e\mu_e^\top+\Sigma_e)+\sigma_e^2C_e^2I_d\bigr].
\end{equation}
Taking traces and applying the bias-variance identity yields
Eqs.~\eqref{eq:bias-term}--\eqref{eq:dp-var}.
\end{proof}

\subsection{Proof of Theorem~\ref{thm:agg-bias}}
\label{app:agg_bias_proof}
\begin{proof}
From Theorem~\ref{thm:bias-var},
$\beta_e^2 \le \bar{G}^2(Eq_e-1)^2$.

\paragraph{Eq.~\eqref{eq:unweighted-bias} (unweighted average).}
Recall $\chi^2(q\|u)=E\sum_e(q_e-1/E)^2
=E^{-1}\sum_e(Eq_e-1)^2$,
so $\sum_e(Eq_e-1)^2=E\chi^2(q\|u)$.  Averaging:
\[
  \tfrac{1}{E}\!\sum_e\beta_e^2
  \le\tfrac{\bar{G}^2}{E}\!\sum_e(Eq_e-1)^2
  =\bar{G}^2\chi^2(q\|u).
\]

\paragraph{Eq.~\eqref{eq:weighted-bias} (frequency-weighted).}
Because $\beta_e^2\le\bar{G}^2(Eq_e-1)^2$,
\begin{align}
\sum_e q_e\beta_e^2
  &\le \bar{G}^2\!\sum_e q_e(Eq_e-1)^2.
  \label{eq:weighted-step1}
\end{align}
Expanding $(Eq_e-1)^2=E^2q_e^2-2Eq_e+1$:
\begin{align}
\sum_e q_e(Eq_e\!-\!1)^2
  &= E^2\!\sum_e q_e^3
    -2E\!\sum_e q_e^2+1.
  \label{eq:weighted-expand}
\end{align}
Applying $q_e^3\le q_e^2$ (since $q_e\le 1$):
\begin{align}
\eqref{eq:weighted-expand}
  &\le(E^2\!-\!2E)\!\sum_e q_e^2+1.
  \label{eq:weighted-cubic}
\end{align}
The identity $\chi^2(q\|u)=E\sum_e q_e^2-1$ gives
$\sum_e q_e^2=(\chi^2(q\|u)+1)/E$, so:
\begin{align}
\eqref{eq:weighted-cubic}
  &=(E\!-\!2)\chi^2(q\|u)+(E\!-\!1).
  \label{eq:weighted-chi}
\end{align}
Combining \eqref{eq:weighted-step1}--\eqref{eq:weighted-chi}:
\begin{equation}
  \sum_e q_e\beta_e^2
  \le
  \bar{G}^2\!\left[(E\!-\!2)\chi^2(q\|u)+(E\!-\!1)\right].
  \label{eq:weighted-tight}
\end{equation}
Since $E\ge 2$ both coefficients are non-negative.
A coarser path via $q_e\le 1$ directly gives:
\begin{align}
\sum_e q_e(Eq_e\!-\!1)^2
  &\le\sum_e(Eq_e\!-\!1)^2
   =E\chi^2(q\|u),
   \label{eq:weighted-coarse}
\end{align}
recovering $\sum_e q_e\beta_e^2\le\bar{G}^2 E\chi^2(q\|u)$,
i.e.\ Eq.~\eqref{eq:weighted-bias}.
Bound~\eqref{eq:weighted-tight} is tighter for small
$\chi^2$; bound~\eqref{eq:weighted-coarse} is simpler
and suffices for the entropy criterion in
Sec.~\ref{sec:theory}.\qedhere
\end{proof}
\subsection{SGD-Surrogate Implication}
\label{app:sgd-surrogate}

Let $F_{\mathrm{bal}}$ be $L$-smooth and PL with parameter $\beta>0$.
For bias $b_t$ and noise $z_t$ with
$\mathbb{E}\|z_t\|_2^2\le V$,
$\mathbb{E}\|b_t\|_2^2\le B_{\mathrm{imb}}$,
and $\eta\le1/L$:
\begin{align}
&\mathbb{E}[F_{\mathrm{bal}}(\theta_T)]-F_{\mathrm{bal}}^*
\notag\\
&\quad\le(1-\eta\beta)^T\Delta_0
  +O\!\left(\tfrac{L\eta V}{\beta}\right)
  +O\!\left(\tfrac{LB_{\mathrm{imb}}}{\beta^2}\right),
\label{eq:surrogate-risk}
\end{align}
where $\Delta_0=F_{\mathrm{bal}}(\theta_0)-F_{\mathrm{bal}}^*$.
Bounding $B_{\mathrm{imb}}$ via Theorem~\ref{thm:agg-bias} and
Eq.~\eqref{eq:entropy-upper} shows routing imbalance contributes an
error vanishing at uniform routing.
\emph{This is an explanatory bound within a surrogate model, not a
convergence guarantee for DP-Adam on non-convex MoE networks.}

\subsection{Scaling of the DP Term}
\label{app:dp-scaling}

Under balanced routing, Corollary~\ref{cor:balanced-error} gives
$\sqrt{\mathbb{E}\|\bar{g}_{t,e}^{\mathrm{exp}}-\mu_e\|_2^2}
\le C_e\sqrt{E/B}+E\sigma_eC_e\sqrt{d}/B$.
With $\sigma_e=\Theta(\sqrt{\log(1/\delta)}/\varepsilon_{\mathrm{step}})$,
the privacy term scales as
\begin{equation}
\frac{EC_e\sqrt{d\log(1/\delta)}}{B\,\varepsilon_{\mathrm{step}}},
\end{equation}
matching the standard per-expert DP mean-estimation rate up to log
factors when $B=\Theta(n)$. This is a single-estimator result, not
a convergence guarantee for DP-Adam.

\subsection{Stability of expert ownership}
\label{app:stability}
This appendix formalizes the stability claim of
Sec.~\ref{sec:ownership-stability}: because the owner map
$a_{l^\star}$ is frozen before training and never recomputed, we ask
how far the induced ownership partition can drift as the upstream
representation $\phi=\theta_{\mathrm{sh}}(\cdot)$ evolves, under a
Voronoi abstraction of ownership.
\begin{definition}[Voronoi expert ownership]
\label{def:voronoi_owner}
Let $\phi:\mathcal X\to\mathbb R^m$ be a frozen representation map and let $C=\{c_e\}_{e=1}^E\subset\mathbb R^m$ be expert prototypes. The Voronoi cell of expert $e$ is $\mathcal V_e=\{z:\|z-c_e\|_2^2\le \|z-c_j\|_2^2,\forall j\in[E]\}$. The induced owner map is $a_C(x)=\arg\min_{e\in[E]}\|\phi(x)-c_e\|_2^2$, with deterministic tie-breaking. The routing mass is $q_e=\mathbb P(a_C(X)=e)$.
\end{definition}

\begin{definition}[Voronoi margin and boundary regularity]
\label{def:voronoi_regu}
For $z=\phi(x)$, let $a=a_C(x)$ and define the margin $\gamma_C(z)=\min_{j\neq a}\{\|z-c_j\|_2^2-\|z-c_a\|_2^2\}$. Let $\Delta_C=\max_{j,k}\|c_j-c_k\|_2$. We say that the representation distribution satisfies $(\kappa,C)$-boundary regularity if $\mathbb P(\gamma_C(\phi(X))\le t)\le \kappa t$ for all $t>0$.
\end{definition}

\begin{definition}[Perturbed ownership]
\label{def:pertube}
Given another representation map $\tilde\phi:\mathcal X\to\mathbb R^m$, define $\tilde a_C(x)=\arg\min_{e\in[E]}\|\tilde\phi(x)-c_e\|_2^2$ and $\tilde q_e=\mathbb P(\tilde a_C(X)=e)$, with the same deterministic tie-breaking.
\end{definition}

\begin{theorem}[Voronoi stability of expert ownership]
\label{thm:stab}
If $\|\tilde\phi(x)-\phi(x)\|_2\le r$ for all $x$, then $\tilde a_C(x)=a_C(x)$ for every $x$ such that $\gamma_C(\phi(x))>2\Delta_C r$. Moreover, under $(\kappa,C)$-boundary regularity, $\mathbb P(\tilde a_C(X)\neq a_C(X))\le 2\kappa\Delta_C r$ and $\|\tilde q-q\|_1\le 4\kappa\Delta_C r$.
\end{theorem}

\begin{proof}
Fix $x\in\mathcal X$ and write $z=\phi(x)$, 
$\tilde z=\tilde\phi(x)$, and $a=a_C(x)$. For each $j\neq a$, define
\[
h_{a,j}(u)
=
\|u-c_j\|_2^2-\|u-c_a\|_2^2 .
\]
Since $a=\arg\min_e\|z-c_e\|_2^2$, $h_{a,j}(z)\ge 0$ for all $j\neq a$, and
\[
\gamma_C(z)=\min_{j\neq a}h_{a,j}(z).
\]
Expanding the squared norms gives
\[
\begin{aligned}
h_{a,j}(u)
&=
-2\langle u,c_j\rangle+\|c_j\|_2^2
+2\langle u,c_a\rangle-\|c_a\|_2^2  \\
&=
2\langle u,c_a-c_j\rangle
+\|c_j\|_2^2-\|c_a\|_2^2 .
\end{aligned}
\]
Therefore
\[
\begin{aligned}
h_{a,j}(\tilde z)-h_{a,j}(z)
&=
2\langle \tilde z-z,c_a-c_j\rangle .
\end{aligned}
\]
By Cauchy--Schwarz and $\Delta_C=\max_{p,q}\|c_p-c_q\|_2$,
\[
\begin{aligned}
h_{a,j}(\tilde z)
&=
h_{a,j}(z)
+2\langle \tilde z-z,c_a-c_j\rangle \\
&\ge
h_{a,j}(z)
-2\|\tilde z-z\|_2\|c_a-c_j\|_2 \\
&\ge
h_{a,j}(z)-2r\Delta_C .
\end{aligned}
\]
Taking the minimum over $j\neq a$ yields
\[
\begin{aligned}
\min_{j\neq a}h_{a,j}(\tilde z)
&\ge
\min_{j\neq a}h_{a,j}(z)-2r\Delta_C  \\
&=
\gamma_C(z)-2r\Delta_C .
\end{aligned}
\]
Thus, if $\gamma_C(z)>2r\Delta_C$, then $h_{a,j}(\tilde z)>0$ for every $j\neq a$. Equivalently,
\[
\|\tilde z-c_a\|_2^2
<
\|\tilde z-c_j\|_2^2,
\qquad \forall j\neq a .
\]
Hence $\tilde a_C(x)=a_C(x)$. Therefore
\[
\{\tilde a_C(X)\neq a_C(X)\}
\subseteq
\{\gamma_C(\phi(X))\le 2r\Delta_C\}.
\]
By $(\kappa,C)$-boundary regularity,
\[
\begin{aligned}
\mathbb P(\tilde a_C(X)\neq a_C(X))
&\le
\mathbb P(\gamma_C(\phi(X))\le 2r\Delta_C) \\
&\le
2\kappa r\Delta_C .
\end{aligned}
\]

It remains to control the routing mass. Let $A=a_C(X)$ and $\tilde A=\tilde a_C(X)$. For each expert $e$,
\[
\begin{aligned}
|\tilde q_e-q_e|
&=
\left|
\mathbb P(\tilde A=e)-\mathbb P(A=e)
\right| \\
&=
\left|
\mathbb E[
\mathbf 1\{\tilde A=e\}
-
\mathbf 1\{A=e\}]
\right| \\
&\le
\mathbb E
\left|
\mathbf 1\{\tilde A=e\}
-
\mathbf 1\{A=e\}
\right| .
\end{aligned}
\]
The two indicators differ only when the assignments disagree, and more explicitly,
\[
\begin{aligned}
\left|
\mathbf 1\{\tilde A=e\}
-
\mathbf 1\{A=e\}
\right|
&\le
\mathbf 1\{\tilde A=e,A\neq e\}  \\
&\quad+
\mathbf 1\{A=e,\tilde A\neq e\}.
\end{aligned}
\]
Summing over $e$ gives
\[
\begin{aligned}
\|\tilde q-q\|_1
&=
\sum_{e=1}^E |\tilde q_e-q_e| \\
&\le
\sum_{e=1}^E
\mathbb P(\tilde A=e,A\neq e)  \\
&\quad+
\sum_{e=1}^E
\mathbb P(A=e,\tilde A\neq e).
\end{aligned}
\]
Since exactly one value of $e$ satisfies $\tilde A=e$ and exactly one value satisfies $A=e$,
\[
\sum_{e=1}^E
\mathbb P(\tilde A=e,A\neq e)
=
\mathbb P(\tilde A\neq A),
\]
and
\[
\sum_{e=1}^E
\mathbb P(A=e,\tilde A\neq e)
=
\mathbb P(A\neq \tilde A).
\]
Consequently,
\[
\begin{aligned}
\|\tilde q-q\|_1
&\le
2\mathbb P(\tilde A\neq A) \\
&\le
4\kappa r\Delta_C .
\end{aligned}
\]
The theorem follows.
\end{proof}

\subsection{Voronoi-aware expert estimator Error}
\label{app:stab_error}
\begin{definition}[Voronoi cell radius and gradient field]
Let $Z_i=\phi(x_i)$ and $A_i=a_C(x_i)$. For each expert $e$, define
\[
\mathcal I_e=\{i:A_i=e\},\qquad n_e=|\mathcal I_e|,
\qquad q_e=n_e/n .
\]
Let $h_i=\clip(g_i^{\exp},C_e)\in\mathbb R^d$ be the clipped expert gradient, and define
\[
\mu_e=\frac{1}{n_e}\sum_{i\in\mathcal I_e}h_i,
\]
\[
\Sigma_e=
\frac{1}{n_e}
\sum_{i\in\mathcal I_e}
(h_i-\mu_e)(h_i-\mu_e)^\top .
\]
Assume there exists an $L_g$-Lipschitz field $\psi:\mathbb R^m\to\mathbb R^d$ such that $h_i=\psi(Z_i)$. Define
\[
\nu_e=\psi(c_e),
\qquad
r_e^2=
\frac{1}{n_e}
\sum_{i\in\mathcal I_e}
\|Z_i-c_e\|_2^2 .
\]
\end{definition}

\begin{definition}[Public-denominator expert estimator]
Under Poisson subsampling $S_i\sim{\rm Bernoulli}(p)$ with $B=pn$, define
\[
\bar g_e
=
\frac{E}{B}
\left(
\sum_{i\in\mathcal I_e}S_i h_i+\xi_e
\right),
\]
where
\[
\xi_e\sim
\mathcal N(0,\sigma_e^2C_e^2I_d).
\]
\end{definition}

\begin{theorem}[Voronoi-aware expert estimation error] 
\label{thm:stab_error}
For each expert $e$,
\[
\mathbb E\|\bar g_e-\nu_e\|_2^2
\le
V_e^{\rm samp}
+
V_e^{\rm DP}
+
2(Eq_e-1)^2\|\mu_e\|_2^2
+
2L_g^2r_e^2,
\]
where
\[
V_e^{\rm samp}
=
\frac{E^2p(1-p)n_e}{B^2}
\left(
\|\mu_e\|_2^2+\tr\Sigma_e
\right),
\]
and
\[
V_e^{\rm DP}
=
\frac{E^2\sigma_e^2C_e^2d}{B^2}.
\]
Consequently, if $\bar G^2=\max_e\|\mu_e\|_2^2$, then
\[
\begin{aligned}
\frac{1}{E}\sum_{e=1}^E
\mathbb E\|\bar g_e-\nu_e\|_2^2
&\le
\frac{1}{E}\sum_{e=1}^E
\left(
V_e^{\rm samp}+V_e^{\rm DP}
\right) \\
&\quad+
2\bar G^2\chi^2(q\|u)
+
\frac{2L_g^2}{E}
\sum_{e=1}^E r_e^2 .
\end{aligned}
\]
\end{theorem}

\begin{proof}
Fix an expert $e$. All expectations below are taken over the Poisson
subsampling variables $\{S_i\}$ and the Gaussian noise $\xi_e$, conditional
on the dataset. By Definition 6,
\[
\bar g_e
=
\frac{E}{B}
\sum_{i\in\mathcal I_e}S_i h_i
+
\frac{E}{B}\xi_e .
\]
Since $S_i\sim{\rm Bernoulli}(p)$ and $\mathbb E[\xi_e]=0$,
\[
\begin{aligned}
\mathbb E[\bar g_e]
&=
\frac{E}{B}
\sum_{i\in\mathcal I_e}\mathbb E[S_i]h_i \\
&=
\frac{Ep}{B}
\sum_{i\in\mathcal I_e}h_i .
\end{aligned}
\]
Using $B=pn$ and $\sum_{i\in\mathcal I_e}h_i=n_e\mu_e$, we obtain
\[
\mathbb E[\bar g_e]
=
\frac{En_e}{n}\mu_e
=
Eq_e\mu_e .
\]
Hence the bias relative to the cell mean is
\[
\mathbb E[\bar g_e]-\mu_e
=
(Eq_e-1)\mu_e .
\]

We next compute the variance. Independence of the Bernoulli variables gives
\[
\begin{aligned}
{\rm Var}
\left(
\frac{E}{B}
\sum_{i\in\mathcal I_e}S_i h_i
\right)
&=
\frac{E^2}{B^2}
\sum_{i\in\mathcal I_e}
{\rm Var}(S_i h_i).
\end{aligned}
\]
Since $h_i$ is fixed conditional on the dataset,
\[
{\rm Var}(S_i h_i)
=
p(1-p)h_i h_i^\top .
\]
Therefore,
\[
\begin{aligned}
{\rm tr}\,
{\rm Var}
\left(
\frac{E}{B}
\sum_{i\in\mathcal I_e}S_i h_i
\right)
&=
\frac{E^2p(1-p)}{B^2}
\sum_{i\in\mathcal I_e}\|h_i\|_2^2 .
\end{aligned}
\]
By the empirical second-moment identity,
\[
\frac{1}{n_e}
\sum_{i\in\mathcal I_e}\|h_i\|_2^2
=
\|\mu_e\|_2^2+{\rm tr}\,\Sigma_e .
\]
Thus
\[
{\rm tr}\,
{\rm Var}
\left(
\frac{E}{B}
\sum_{i\in\mathcal I_e}S_i h_i
\right)
=
V_e^{\rm samp}.
\]
The DP noise is independent of the sampling term, and
\[
\begin{aligned}
{\rm tr}\,
{\rm Var}
\left(
\frac{E}{B}\xi_e
\right)
&=
\frac{E^2}{B^2}
{\rm tr}
\left(
\sigma_e^2C_e^2I_d
\right) \\
&=
\frac{E^2\sigma_e^2C_e^2d}{B^2}
=
V_e^{\rm DP}.
\end{aligned}
\]
Therefore,
\[
\mathbb E
\left\|
\bar g_e-\mathbb E[\bar g_e]
\right\|_2^2
=
V_e^{\rm samp}+V_e^{\rm DP}.
\]

Now decompose the error relative to the prototype gradient $\nu_e$:
\[
\begin{aligned}
\mathbb E\|\bar g_e-\nu_e\|_2^2
&=
\mathbb E
\left\|
\bar g_e-\mathbb E[\bar g_e]
\right\|_2^2
+
\left\|
\mathbb E[\bar g_e]-\nu_e
\right\|_2^2 .
\end{aligned}
\]
Using $\mathbb E[\bar g_e]=Eq_e\mu_e$,
\[
\mathbb E[\bar g_e]-\nu_e
=
(Eq_e-1)\mu_e+(\mu_e-\nu_e).
\]
Hence, by $\|a+b\|_2^2\le 2\|a\|_2^2+2\|b\|_2^2$,
\[
\begin{aligned}
\left\|
\mathbb E[\bar g_e]-\nu_e
\right\|_2^2
&\le
2(Eq_e-1)^2\|\mu_e\|_2^2
+
2\|\mu_e-\nu_e\|_2^2 .
\end{aligned}
\]

It remains to control $\|\mu_e-\nu_e\|_2^2$. Since $h_i=\psi(Z_i)$ and
$\nu_e=\psi(c_e)$,
\[
\mu_e-\nu_e
=
\frac{1}{n_e}
\sum_{i\in\mathcal I_e}
\left(
\psi(Z_i)-\psi(c_e)
\right).
\]
By Jensen's inequality,
\[
\begin{aligned}
\|\mu_e-\nu_e\|_2^2
&\le
\frac{1}{n_e}
\sum_{i\in\mathcal I_e}
\|\psi(Z_i)-\psi(c_e)\|_2^2 .
\end{aligned}
\]
Since $\psi$ is $L_g$-Lipschitz,
\[
\|\psi(Z_i)-\psi(c_e)\|_2^2
\le
L_g^2\|Z_i-c_e\|_2^2 .
\]
Therefore,
\[
\begin{aligned}
\|\mu_e-\nu_e\|_2^2
&\le
\frac{L_g^2}{n_e}
\sum_{i\in\mathcal I_e}
\|Z_i-c_e\|_2^2 \\
&=
L_g^2r_e^2 .
\end{aligned}
\]
Combining the preceding bounds gives
\[
\mathbb E\|\bar g_e-\nu_e\|_2^2
\le
V_e^{\rm samp}
+
V_e^{\rm DP}
+
2(Eq_e-1)^2\|\mu_e\|_2^2
+
2L_g^2r_e^2 .
\]

We now prove the aggregate bound. Let
$\bar G^2=\max_e\|\mu_e\|_2^2$. Then
\[
\begin{aligned}
\frac{1}{E}
\sum_{e=1}^E
(Eq_e-1)^2\|\mu_e\|_2^2
&\le
\frac{\bar G^2}{E}
\sum_{e=1}^E
(Eq_e-1)^2 .
\end{aligned}
\]
Since $u_e=1/E$,
\[
\begin{aligned}
\chi^2(q\|u)
&=
\sum_{e=1}^E
\frac{(q_e-u_e)^2}{u_e}  \\
&=
E\sum_{e=1}^E
\left(q_e-\frac{1}{E}\right)^2 .
\end{aligned}
\]
Moreover,
\[
\begin{aligned}
\frac{1}{E}
\sum_{e=1}^E
(Eq_e-1)^2
&=
E
\sum_{e=1}^E
\left(q_e-\frac{1}{E}\right)^2  \\
&=
\chi^2(q\|u).
\end{aligned}
\]
Therefore,
\[
\frac{1}{E}
\sum_{e=1}^E
(Eq_e-1)^2\|\mu_e\|_2^2
\le
\bar G^2\chi^2(q\|u).
\]
Averaging the per-expert inequality over $e$ yields
\[
\begin{aligned}
\frac{1}{E}
\sum_{e=1}^E
\mathbb E\|\bar g_e-\nu_e\|_2^2
&\le
\frac{1}{E}
\sum_{e=1}^E
\left(
V_e^{\rm samp}+V_e^{\rm DP}
\right) \\
&\quad+
2\bar G^2\chi^2(q\|u)
+
\frac{2L_g^2}{E}
\sum_{e=1}^E r_e^2 .
\end{aligned}
\]
This proves the theorem.
\end{proof}

\subsection{Scope of the Bias-SNR Connection (Remark 2 details)}
\label{app:remark2-detail}
This appendix expands the two claims of
Remark~\ref{rem:adam-invariance}.

\textbf{Non-monotonicity.} $\beta_e^2=(Eq_e-1)^2\|\mu_e\|_2^2$ is
symmetric around $q_e=1/E$: an expert at $q_e=2/E$ has the same
$\beta_e^2$ as one at $q_e=0$, though these are opposite outcomes
for training. $\beta_e$ measures deviation from a normalization
convention, not training quality.

\textbf{Impossibility of an entropy bound on SNR.} No function $f$
satisfies $1/\mathrm{SNR}_e\le f(H(q))$ for all valid $(q,e)$. Since
$\chi^2(q\|u)\in[0,E{-}1]$ for every valid $q$, take $q_e=\delta$ for
one expert and $q_{e'}=(1-\delta)/(E{-}1)$ for the rest. As
$\delta\to0^+$, $\chi^2(q\|u)\to 1/(E{-}1)$ - near its own
\emph{minimum} - while $H(q)\to\log(E{-}1)$, near-maximal, and
$\mathrm{SNR}_e\propto q_e^2\to0$. A starved expert's contribution to
the aggregate divergence is only $O(1/E)$, so it can coexist with
near-maximal entropy while its own SNR vanishes; no finite $f$ can
bound $1/\mathrm{SNR}_e$ by $H(q)$.

\onecolumn
\section{Experiment Details}
\label{app:exp_detail}

\paragraph{Code and reproducibility.}
All experiments use PyTorch~\cite{DBLP:conf/nips/PaszkeGMLBCKLGA19}
with PRV accounting~\cite{gopi2021numerical} via
Opacus~\cite{DBLP:journals/corr/abs-2109-12298} v1.6.0
(Python 3.12, CUDA 12.6, PyTorch 2.9).

\paragraph{Hardware and training.}
Each trial runs on a single B200 (196\,GB).
At each optimiser step, every record
$i \in \mathcal{D}_{\mathrm{priv}}$ is independently included
in $\mathcal{B}_t^{\mathrm{exp}}$ with probability $p = B/n$
, giving an \emph{expected}
batch size of $B = pn$.
Because the realised $|\mathcal{B}_t^{\mathrm{exp}}|$ may be large, we accumulate
per-record clipped gradients across memory-efficient microbatches
\emph{within} a single optimiser step; no gradients are accumulated
\emph{across} steps.
The denominators $B$ (shared stream) and $D_e = B/E$
(expert stream) are fixed expected values independent of
the realised $|\mathcal{B}_t^{\mathrm{exp}}|$.
Privacy budgets are computed via the \textsc{prv} accountant
\citep{gopi2021numerical} with Poisson rate $p$ and
add/remove adjacency.

\paragraph{Hyperparameters and tuning.}
Primary hyperparameters are epochs, batch size $B$, learning
rate $\eta$, clipping threshold $C$, and noise multiplier
$\sigma_{\mathrm{DP}}$ (set by the PRV accountant).
Each method-including Baselines~A, B, and~C-is tuned
\emph{independently} on SST-2 using the same search grid
(Tab.~\ref{tab:search_grid_cv}) and the same budget of
$N_{\mathrm{total}}=100$ trials; the best-performing
configuration per method is then frozen for MNLI, QNLI,
and QQP.
SST-2 results therefore reflect in-task tuning performance,
while the remaining three tasks measure cross-task transfer
under a fixed, task-agnostic configuration.
Tab.~\ref{tab:search_grid_cv} lists the search grid;
selected values are bolded.

Specifically, Baseline~A treats $C$ and $\sigma_{\mathrm{DP}}$
as free variables over the full grid and does \emph{not}
inherit $C_s$ from the proposed method; the PRV accountant
targets the full budget $\varepsilon$ (i.e.\ $\rho=1$) and
solves for $\sigma_{\mathrm{DP}}$ jointly with $C$ via binary
search, matching the accounting procedure of the proposed
method.
Baselines~B and~C likewise search over all grid values of
$\eta$, $C$, $B$, and epochs independently.
The proposed method applies LoRA to the shared
stream~$\theta_{\mathrm{sh}}$ and each expert~$\theta_{l^\star,e}$; the role-aware mechanism
(ownership map, public denominator~$D_e$, residual objective)
is independent of this parameterisation choice.
\begin{table}[h]

\centering
\small
\begin{tabular}{lc}
\toprule
Hyperparameter & Search grid \\
\midrule
Epochs          & $\{10,\mathbf{20},30\}$ \\
Batch size $B$  & $\{500,\mathbf{1000},2000,3000\}$ \\
Learning rate $\eta$ & $\{10,\mathbf{5},3,1,0.5\}\times10^{-4}$ \\
Clipping $C$    & $\{0.1,\mathbf{1.0},10\}$ \\
\midrule
Budget split $\rho$      & $\{0.7,0.8,\mathbf{0.9},0.95\}$ \\

\bottomrule

\end{tabular}
\caption{Hyperparameter search grid (SST-2). Selected values bolded;
configuration frozen for all other tasks.}
\label{tab:search_grid_cv}

\end{table}

\paragraph{Clipping and $\delta$.}
We use per-record $\ell_2$ clipping
$\mathrm{clip}(g,C)=g\cdot\min(1,C/\|g\|_2)$
and set $\delta=1/N^{1.1}$ for each task.

\onecolumn
\section{More Expriments on Other Models}
\label{app:more_results}
\label{sec:deepseek-vl2-results}
\begin{figure}[h]
\centering
\includegraphics[width=0.85\linewidth]{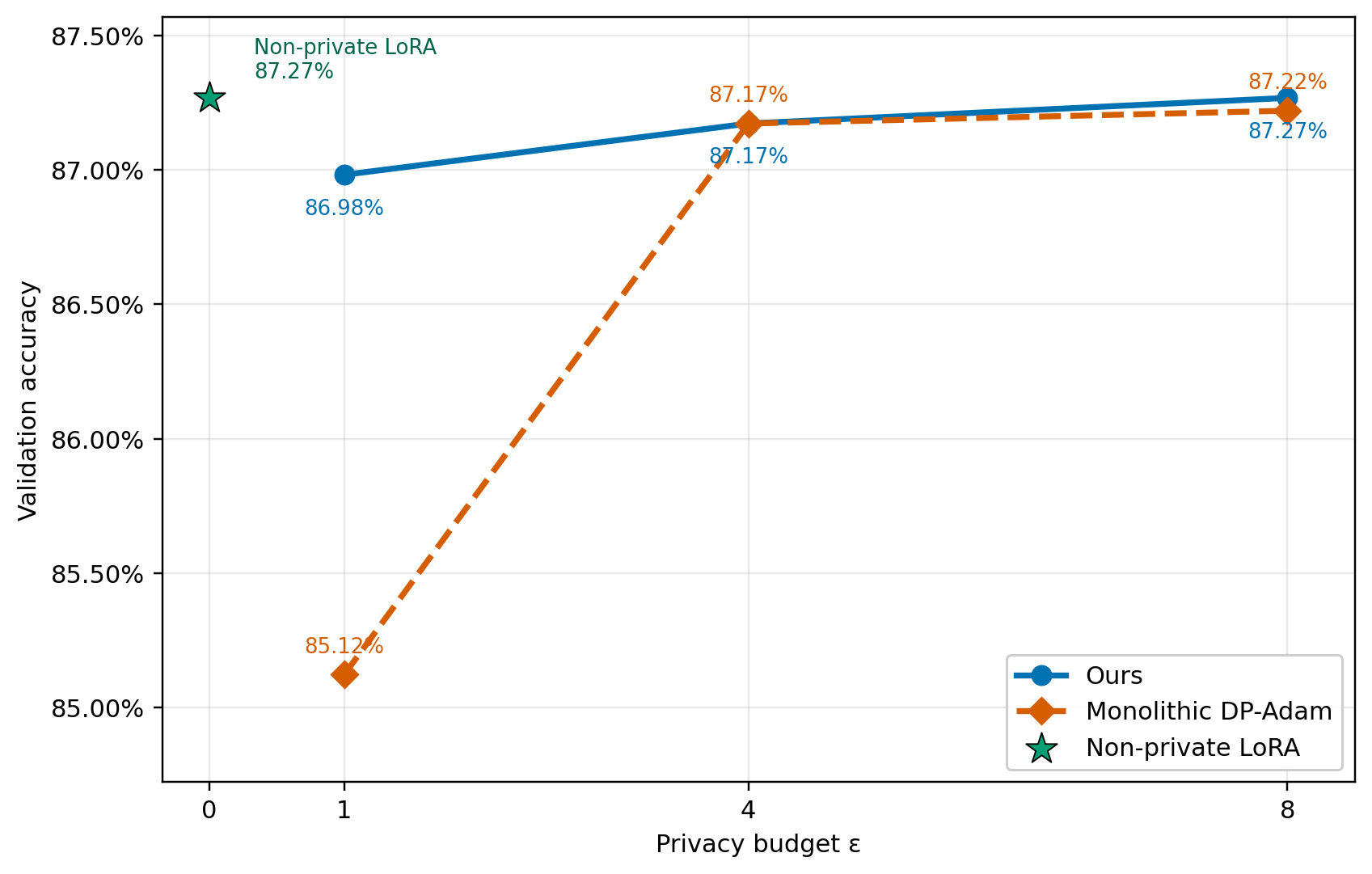}
\caption{Validation accuracy on ScienceQA for \textsc{DeepSeek-VL2-Tiny}
across $\varepsilon\in\{1,4,8\}$. Single run per $\varepsilon$.}
\label{fig:deepseek_vl2}
\end{figure}
As a secondary check of generalization beyond text-only encoder MoEs,
we fine-tune \texttt{DeepSeek-VL2-Tiny} \cite{wu2024deepseek} on ScienceQA \cite{saikh2022scienceqa}, a vision-language
task. \texttt{DeepSeek-VL2-Tiny} has both always-active shared experts
and routed experts (Sec.~\ref{sec:preliminaries}); shared experts
fold into the shared stream, and \Raptor\,  is
applied to the last routed-expert layer as in the main experiments.
No disjoint, format-matched public corpus was available for this
setting, so we protect the last sparse layer by fixed default rather
than via entropy-based selection (Sec.~\ref{sec:alternating}).
Fig.~\ref{fig:deepseek_vl2} reports
validation accuracy at $\varepsilon\in\{1,4,8\}$. At $\varepsilon=1$,
Ours reaches $86.98\%$ vs.\ $85.12\%$ for Monolithic DP-Adam LoRA
(\textcolor{custom_darkgreen}{$+1.86$}), consistent with the high-noise pattern on Switch and
OLMoE; the gap closes with $\varepsilon$, tying at $4$ ($87.17\%$) and
marginally exceeding non-private LoRA at $8$ ($87.27\%$ vs.\
$87.22\%$) - a single-run artifact (Reporting,
Sec.~\ref{sec:exp-setup}) rather than evidence that private
training exceeds its non-private upper bound. That gains persist
under a fixed-layer fallback, without entropy-based selection,
suggests \Raptor's core mechanism does not
strictly depend on that selection step (Sec.~\ref{sec:public-layer-selection}).

\onecolumn
\section{Ablation Study}
\label{app:ablation}
\subsection{Component Ablation}
\label{sec:component-ablation}
Tab.~\ref{tab:switch_main_results} attributes \textbf{Ours}'s overall gains to
the update rule as a whole, but does not show which of the four
components introduced in Sec.~\ref{sec:method} -
expert-specific clipping, the public denominator $D_e{=}B/E$, the
alternating shared/expert schedule, and the residual objective - is
doing the work. To isolate each one's contribution, we ablate \textbf{Ours} on
\textsc{Switch-base-8} by removing or replacing exactly one component
at a time while holding the entropy-selected sparse layer fixed:
\emph{w/ $D_e{=}B$} replaces the public denominator with the
realized-count-free but naive full-batch denominator;
\emph{w/ full-model clip} replaces separate shared/expert clipping
with a single global bound $C$; \emph{w/o alternating} updates shared
and expert parameters jointly in one step rather than in alternation
(Sec.~\ref{sec:alternating}); and \emph{w/o residual obj.} drops
the residual weighting $w_i$ (Eq.~\eqref{eq:residual-weight}), training
the expert stream on the raw classification loss instead. Tab.~\ref{tab:ablation_combined}
reports this ablation at $\varepsilon\approx 8$ and $\varepsilon=1$
side by side, so that each component's contribution can be compared
directly across privacy budgets.

The \emph{alternating schedule} is the largest contributor: removing it
costs up to $1.97$ pts at $\varepsilon\approx 8$ (QNLI) and $2.09$ pts at
$\varepsilon=1$ (QQP), and on QNLI/QQP at $\varepsilon\approx 8$ the
no-alternating variant falls \emph{below} the matched-scope global
baseline (Tab.~\ref{tab:switch_main_results}), showing that role-aware
clipping and normalization alone are not sufficient without the
two-stream schedule. The public denominator $D_e=B/E$ contributes
$0.05$-$0.98$ pts, with the largest single gain on QNLI; consistent
with Corollary~\ref{cor:balanced-error}, this gain is proportionally
larger relative to the other components at $\varepsilon=1$, where the
bias from $D_e=B$ under high noise is more damaging. Expert-specific
clipping contributes $0.06$-$0.62$ pts, and the residual objective
contributes $0.04$-$0.66$ pts, smallest on SST-2, where the shared
stream alone suffices for sentiment classification. Relative to
$\varepsilon\approx 8$, component contributions grow at $\varepsilon=1$
on SST-2, MNLI, and QQP but shrink on QNLI: the high-noise
amplification suggested by Corollary~\ref{cor:balanced-error} is
task-dependent, not uniform. Together, all components yield
$0.23$-$3.24$ pts over Monolithic DP-Adam LoRA
(Tab.~\ref{tab:switch_main_results}).

\begin{table}[h]
\centering
\small
\begin{tabular}{lcccccccc}
\toprule
& \multicolumn{4}{c}{$\varepsilon\approx 8$} & \multicolumn{4}{c}{$\varepsilon=1$} \\
\cmidrule(lr){2-5}\cmidrule(lr){6-9}
Method & SST-2 & MNLI & QNLI & QQP & SST-2 & MNLI & QNLI & QQP \\
\midrule
Monolithic DP-Adam LoRA   & $93.92$ & $79.36$ & $82.48$ & $82.68$ & $90.94$ & $76.91$ & $80.38$ & $81.19$ \\

\rowcolor{custom_light_purple_2}
\Raptor\,  (full)         & $\mathbf{94.15}$ & $\mathbf{81.32}$ & $\mathbf{85.26}$ & $\mathbf{85.92}$ & $\mathbf{92.90}$ & $\mathbf{78.42}$ & $\mathbf{83.30}$ & $\mathbf{84.08}$ \\

\quad w/\ $D_e=B$         & $94.10$ & $80.96$ & $84.28$ & $85.55$ & $92.83$ & $77.98$ & $82.79$ & $83.55$ \\
\quad w/\ full-model clip & $94.09$ & $81.25$ & $84.64$ & $85.84$ & $92.82$ & $78.34$ & $83.22$ & $83.99$ \\
\quad w/o alternating     & $94.02$ & $80.21$ & $83.29$ & $84.09$ & $92.75$ & $77.15$ & $81.68$ & $81.99$ \\
\quad w/o residual obj.   & $94.07$ & $81.22$ & $84.60$ & $85.82$ & $92.86$ & $78.30$ & $83.18$ & $83.96$ \\
\bottomrule
\end{tabular}%
\caption{Component ablation on \textsc{Switch-base-8} at
$\varepsilon\approx 8$ and $\varepsilon=1$. Bold: full \Raptor.
Each row removes or replaces one component; all variants use the
entropy-selected sparse layer.}
\label{tab:ablation_combined}
\end{table}

\subsection{Matched-Scope Diagnostics}
\label{sec:matched-scope}
The gains reported in Tab.~\ref{tab:switch_main_results} compare
\Raptor\, against Monolithic DP-Adam and Monolithic
DP-Adam LoRA, both of which differ from \textbf{Ours} along \emph{two} axes at
once: they use a different update rule (one clipping bound, one noise
scale, one denominator, applied to the full model) \emph{and} a
different training scope (all trainable parameters, rather than the
shared stream plus one selected expert layer).
This leaves an open question: how much of Ours's advantage comes from
the role-aware update rule itself - separate clipping, the public
denominator $D_e{=}B/E$, and the alternating schedule
- versus simply from training a smaller, more targeted set of
parameters under a matched privacy budget? To isolate this, we
construct three matched-scope baselines that fix Ours's layer, total
privacy budget, and step count, and vary \emph{only} the update rule
applied within that scope: Baseline A trains the same shared and
expert parameters as Ours but with one global clipping bound, one
noise scale, and one denominator (i.e.\ monolithic DP applied to
Ours's reduced scope); Baseline B trains only the shared stream, with
the selected expert layer frozen; and Baseline C trains only the
selected expert layer, with the shared stream frozen. Comparing Ours
against these three, rather than against the full-scope monolithic
baselines alone, lets us attribute the accuracy gains to the update
rule specifically, rather than to scope or parameter count.

Tab.~\ref{tab:matched_scope} reports Baselines A-C on
\textsc{Switch-base-8} (same layer, scope, budget, and steps as Ours;
update rule only varies; Sec.~\ref{sec:exp-setup}) alongside
Monolithic DP-Adam LoRA and Ours (entropy-selected) for reference,
across all three privacy budgets. Baseline A (global DP) isolates
whether matching training scope to Ours alone explains the gains; it
trails Ours by $0.36$--$3.85$ points on MNLI, QNLI, QQP at
$\varepsilon=8$, and by comparable margins at $\varepsilon\in\{1,4\}$
- confirming the remaining gap is attributable to role-aware
clipping, normalization, and the alternating schedule
(Tab.~\ref{tab:ablation_combined}), not scope alone. Baseline B
(shared-only, expert parameters frozen) is competitive with Ours on
SST-2 but trails on QNLI and QQP at every budget, showing expert
updates contribute beyond what the shared stream alone achieves.
Baseline C (expert-only, shared stream frozen) collapses toward
random accuracy on MNLI ($32$--$36\%$) and QNLI ($\approx\!50\%$)
across all budgets, confirming the shared stream is necessary, not
merely helpful.

\begin{table}[h]
\centering
\small
\begin{tabular}{llcccc}
\toprule
$\varepsilon$ & Method & SST-2 & MNLI & QNLI & QQP \\
\midrule

\multirow{5}{*}{$1$}
& Monolith.\ DP-Adam LoRA  & 90.94 & 76.91 & 80.38 & 81.19 \\
& Baseline A (global)      & 92.55 & 76.80 & 82.90 & 83.10 \\
& Baseline B (shared-only) & 91.80 & 76.40 & 82.50 & 82.90 \\
& Baseline C (expert-only) & 67.78 & 32.75 & 49.72 & 63.66 \\
& \cellcolor{custom_light_purple_2}\textbf{Ours, entropy-sel.}
& \cellcolor{custom_light_purple_2}\textbf{92.90}
& \cellcolor{custom_light_purple_2}\textbf{78.42}
& \cellcolor{custom_light_purple_2}\textbf{83.30}
& \cellcolor{custom_light_purple_2}\textbf{84.08} \\

\midrule

\multirow{5}{*}{$4$}
& Monolith.\ DP-Adam LoRA  & 92.66 & 78.33 & 81.95 & 82.26 \\
& Baseline A (global)      & 93.32 & 77.10 & 84.68 & 84.40 \\
& Baseline B (shared-only) & 93.20 & 77.70 & 84.68 & 84.78 \\
& Baseline C (expert-only) & 68.69 & 34.91 & 49.75 & 64.20 \\
& \cellcolor{custom_light_purple_2}\textbf{Ours, entropy-sel.}
& \cellcolor{custom_light_purple_2}\textbf{94.05}
& \cellcolor{custom_light_purple_2}\textbf{79.03}
& \cellcolor{custom_light_purple_2}\textbf{84.70}
& \cellcolor{custom_light_purple_2}\textbf{84.78} \\

\midrule

\multirow{5}{*}{$8$}
& Monolith.\ DP-Adam LoRA  & 93.92 & 79.36 & 82.48 & 82.68 \\
& Baseline A (global)      & 93.46 & 77.47 & 84.90 & 85.20 \\
& Baseline B (shared-only) & 93.50 & 78.10 & 84.83 & 85.50 \\
& Baseline C (expert-only) & 68.81 & 35.57 & 49.73 & 64.47 \\
& \cellcolor{custom_light_purple_2}\textbf{Ours, entropy-sel.}
& \cellcolor{custom_light_purple_2}\textbf{94.15}
& \cellcolor{custom_light_purple_2}\textbf{81.32}
& \cellcolor{custom_light_purple_2}\textbf{85.26}
& \cellcolor{custom_light_purple_2}\textbf{85.92} \\

\bottomrule
\end{tabular}%
\caption{Matched-scope diagnostics on \textsc{Switch-base-8}: Baselines
A--C share Ours's layer, scope, budget, and step count, varying only
the update rule (Sec.~\ref{sec:exp-setup}). Monolithic DP-Adam LoRA
and Ours (entropy-selected) shown for reference; full results
including first-/last-sparse variants and non-private upper bounds
are in Tab.~\ref{tab:switch_main_results}. Single run per cell.}
\label{tab:matched_scope}
\end{table}

\subsection{Sensitivity to the Shared/Expert Budget Split $\rho$}
\label{sec:hparam}

\begin{figure}[t]
    \centering
    \includegraphics[width=0.6\textwidth]{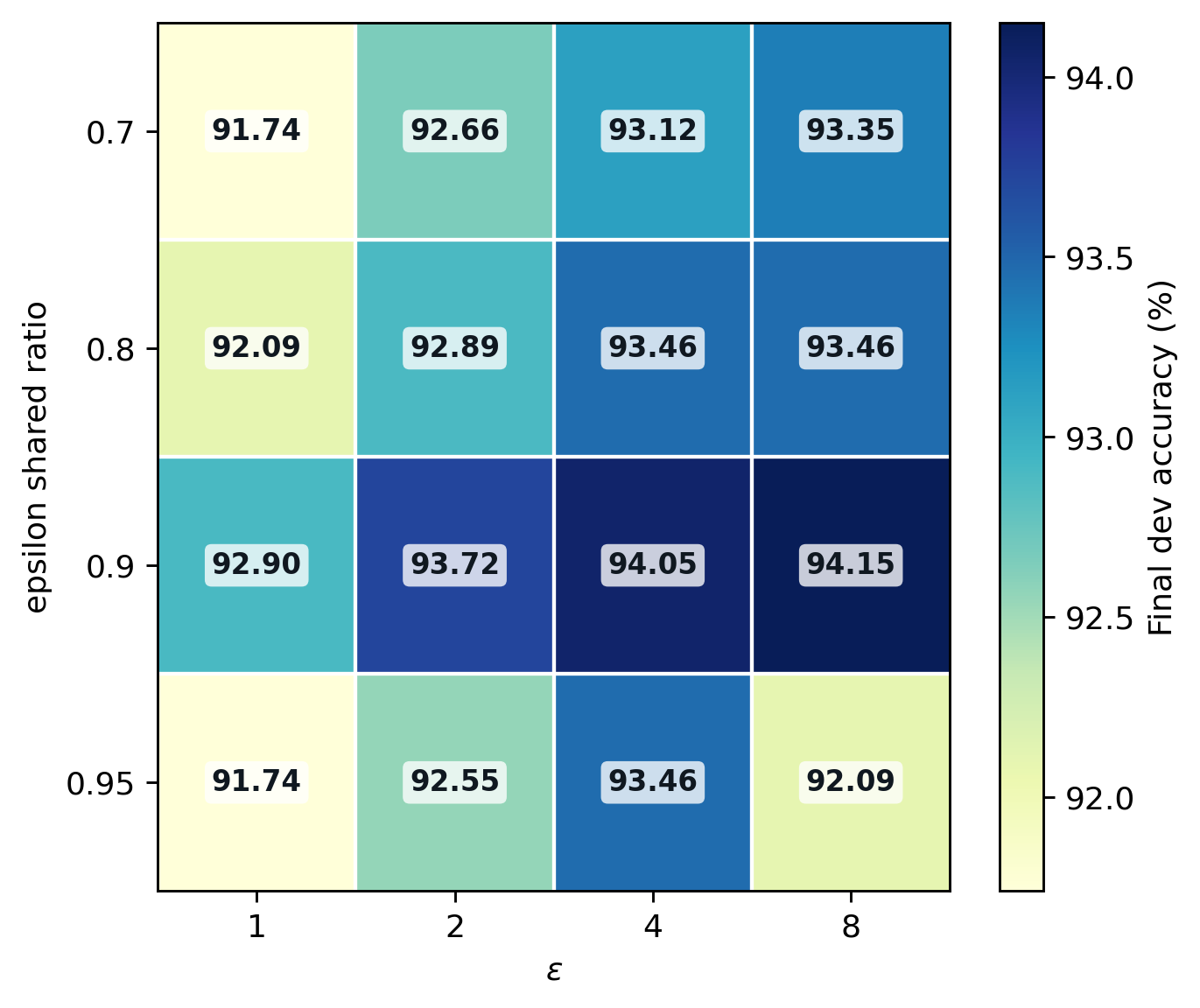}
    \caption{SST-2 accuracy (\%) of \Raptor\,  on
    \textsc{Switch-base-8} across $\varepsilon\in\{1,2,4,8\}$ and
    $\rho\in\{0.7,0.8,0.9,0.95\}$.}
    \label{fig:rho_sensitivity}
\end{figure}

\Raptor\, introduces one hyperparameter with no
analogue in monolithic DP fine-tuning: $\rho\in(0,1)$, which splits
the total privacy budget $\varepsilon$ between the shared stream
($\rho\varepsilon$) and the expert streams ($(1{-}\rho)\varepsilon$,
Sec.~\ref{sec:privacy}). Because the two streams have
different sensitivities and noise requirements - the expert stream's
noise variance carries an extra $E^2$ factor from the normalized
sensitivity $EC_e/B$ (Theorem~\ref{thm:bias-var}) - this split is not
a cosmetic tuning knob: an unfavorable choice of $\rho$ could starve
either stream of budget and erase the method's advantage entirely.
Proposition~\ref{prop:shared-dominant-rho} gives qualitative guidance
(shared-dominant $\rho$ under the surrogate analysis of
Appendix~\ref{app:budget-allocation}), but does not certify a specific
value; we therefore validate the choice of $\rho=0.9$ empirically here
by sweeping $\rho$ against $\varepsilon$ and checking that accuracy is
not unduly sensitive to this choice within a reasonable range.

Fig.~\ref{fig:rho_sensitivity} shows SST-2 accuracy across
$\rho\in\{0.7,0.8,0.9,0.95\}$ and $\varepsilon\in\{1,2,4,8\}$.
Accuracy increases monotonically with $\varepsilon$ for $\rho\le 0.9$;
at $\rho=0.95$ accuracy instead drops from $\varepsilon=4$ to
$\varepsilon=8$ ($93.46\to 92.09$), an inversion that only appears at
the most shared-dominant split tested. Sensitivity to $\rho$ is
highest at $\varepsilon=1$, where the tightest overall budget makes
misallocation between streams most costly, and shrinks at larger
$\varepsilon$, where both streams have enough budget that the split
matters less. Setting $\rho=0.95$ hurts at $\varepsilon=8$ ($92.09$
vs.\ $94.15$ for $\rho=0.9$), showing that under-budgeting the expert
stream degrades expert-specific learning even when the total budget
is generous. We use $\rho=0.9$ as the default throughout, as it is
uniformly at or near the best setting across all four values of
$\varepsilon$ tested.

\onecolumn
\section{Empirical Routing Entropy Analysis}
\label{app:entropy-analysis}

We verify whether routing entropy, estimated on external public corpora,
provides a reliable privacy-free signal for layer selection. Per
Sec.~\ref{sec:theory}, higher $H(q^{(l)})$ implies smaller
public-denominator bias at layer $l$.

\paragraph{Calibration corpora.}
For each task we use an external unlabeled corpus disjoint from all
GLUE splits and from $\mathcal{D}_{\mathrm{priv}}$: IMDb reviews for
SST-2~\cite{maas2011learning}, ANLI R2 pairs for
MNLI~\cite{nie2020adversarial}, SQuAD contexts for
QNLI~\cite{rajpurkar2016squad}. For QQP, we use PAWS-Wiki sentence pairs from the PAWS \texttt{labeled\_final} training split, with labels ignored. Exact
string-match verification finds zero pair overlap with GLUE QQP, QNLI, MNLI, SST-2 train.
\paragraph{Protocol.}
For each candidate layer $l\in\{0,\dots,5\}$ (encoder blocks
$\{1,3,5,7,9,11\}$), we apply the frozen \textsc{Switch-base-8} router to
$\mathcal{C}_{\mathrm{pub}}$ and compute the empirical owner distribution
\begin{equation}
q^{(l)}_e
=\frac{|\{x_i\in\mathcal{C}_{\mathrm{pub}}:a_l(x_i)=e\}|}
      {|\mathcal{C}_{\mathrm{pub}}|}.
\end{equation}
Tab.~\ref{tab:switch_public_calibration_all_sparse_layers} reports
$H(q)/\log E$, $\log E-H(q)$, and $\chi^2(q\|u)$ for all 24
layer-task combinations.

\paragraph{Results.}
Routing balance varies substantially across tasks and layers.
On SST-2, block~1 is by far the most balanced ($H/\log E=0.972$,
$\chi^2=0.118$), while block~9 collapses ($H/\log E=0.304$).
On MNLI, block~1 collapses ($\chi^2=4.946$) and block~7 is most
balanced ($H/\log E=0.682$, $\chi^2=1.262$).
On QNLI, block~11 leads ($H/\log E=0.902$, $\chi^2=0.393$).
On QQP, blocks~5 and~7 collapse ($\chi^2>2.9$) while block~9 is
most balanced ($H/\log E=0.837$, $\chi^2=0.683$).

Applying $l^\star=\arg\max_l H(q^{(l)})$ selects block~1 (SST-2),
block~7 (MNLI), block~11 (QNLI), and block~9 (QQP) - all
unambiguous, with margins of at least $0.029$ nats over the
runner-up. Within each task, entropy and $\chi^2$ criteria agree,
consistent with the Pinsker bound (Eq.~\eqref{eq:entropy-upper}).
A fixed ``last-layer'' heuristic would miss the optimal layer for
SST-2 (by $0.311$ nats) and QQP. These selections were fixed before
any private training; Sec.~\ref{sec:main-results} tests them
directly.

\begin{table*}[h]
\centering
\small
\begin{tabular}{llrccccc}
\toprule
Task & Calibration source & $n$
  & \makecell{Sparse\\layer}
  & \makecell{Encoder\\block}
  & $H(q)/\!\log E$ $\uparrow$
  & $\log E\!-\!H(q)$ $\downarrow$
  & $\chi^2(q\|u)$ $\downarrow$ \\
\midrule
MNLI  & ANLI R2       & 45{,}460  & 0 & 1  & 0.307 & 1.442 & 4.946 \\
MNLI  & ANLI R2       & 45{,}460  & 1 & 3  & 0.674 & 0.677 & 1.580 \\
MNLI  & ANLI R2       & 45{,}460  & 2 & 5  & 0.641 & 0.746 & 1.617 \\
MNLI  & ANLI R2       & 45{,}460  & 3 & 7  & \textbf{0.682} & \textbf{0.660} & \textbf{1.262} \\
MNLI  & ANLI R2       & 45{,}460  & 4 & 9  & 0.643 & 0.743 & 1.356 \\
MNLI  & ANLI R2       & 45{,}460  & 5 & 11 & 0.646 & 0.736 & 1.381 \\
\midrule
QNLI  & SQuAD context & 87{,}599  & 0 & 1  & 0.761 & 0.498 & 1.257 \\
QNLI  & SQuAD context & 87{,}599  & 1 & 3  & 0.881 & 0.246 & 0.387 \\
QNLI  & SQuAD context & 87{,}599  & 2 & 5  & 0.873 & 0.264 & 0.505 \\
QNLI  & SQuAD context & 87{,}599  & 3 & 7  & 0.852 & 0.307 & 0.543 \\
QNLI  & SQuAD context & 87{,}599  & 4 & 9  & 0.807 & 0.401 & 0.851 \\
QNLI  & SQuAD context & 87{,}599  & 5 & 11 & \textbf{0.902} & \textbf{0.204} & \textbf{0.393} \\
\midrule
QQP   &  PAWS-Wiki    & 49{,}289 & 0 & 1  & 0.727 & 0.568 & 1.481 \\
QQP   &  PAWS-Wiki    & 49{,}289 & 1 & 3  & 0.814 & 0.388 & 1.008 \\
QQP   &  PAWS-Wiki   & 49{,}289 & 2 & 5  & 0.574 & 0.885 & 2.922 \\
QQP   & PAWS-Wiki    & 49{,}289 & 3 & 7  & 0.498 & 1.044 & 3.379 \\
QQP   & PAWS-Wiki    & 49{,}289 & 4 & 9  & \textbf{0.837} & \textbf{0.338} & \textbf{0.683} \\
QQP   & PAWS-Wiki    & 49{,}289 & 5 & 11 & 0.825 & 0.364 & 0.721 \\
\midrule
SST-2 & IMDb unsup.   & 50{,}000  & 0 & 1  & \textbf{0.972} & \textbf{0.058} & \textbf{0.118} \\
SST-2 & IMDb unsup.   & 50{,}000  & 1 & 3  & 0.828 & 0.358 & 0.876 \\
SST-2 & IMDb unsup.   & 50{,}000  & 2 & 5  & 0.767 & 0.485 & 0.972 \\
SST-2 & IMDb unsup.   & 50{,}000  & 3 & 7  & 0.876 & 0.259 & 0.405 \\
SST-2 & IMDb unsup.   & 50{,}000  & 4 & 9  & 0.304 & 1.447 & 4.774 \\
SST-2 & IMDb unsup.   & 50{,}000  & 5 & 11 & 0.661 & 0.706 & 1.763 \\
\bottomrule
\end{tabular}
\caption{Routing statistics for all six sparse encoder layers of
\textsc{Switch-base-8} on external public corpora (labels ignored; disjoint
from GLUE splits and $\mathcal{D}_{\mathrm{priv}}$).
$H(q)/\log E$ ($\uparrow$), entropy deficit $\log E-H(q)$ in nats
($\downarrow$), and $\chi^2(q\|u)$ ($\downarrow$) for $E=8$ experts.
\textbf{Bold}: entropy-selected layer $l^\star$ per task.
All selections are unambiguous; no privacy budget is consumed.}
\label{tab:switch_public_calibration_all_sparse_layers}
\end{table*}
\paragraph{Corpus-size sensitivity.}
Fig.~\ref{fig:entropy_vs_n} and Fig.~\ref{fig:best_block_vs_n} show that
the minimum $n$ for stable layer selection correlates inversely with the
entropy margin between the top two layers
(Tab.~\ref{tab:switch_public_calibration_all_sparse_layers}): SST-2
stabilises at $n=10$ (margin $0.144$ nats), QNLI at $n\geq 100$
($0.021$ nats), MNLI at $n\geq 500$ ($0.008$ nats), and QQP at
$n\geq 1{,}000$ ($0.023$ nats, three-way near-tie).
All four tasks select the correct layer at $n=1{,}000$, under $2.2\%$
of the smallest calibration corpus used.
We recommend $n\geq 1{,}000$ as a conservative default; when the entropy
margin at that size exceeds $0.05$ nats, increasing $n$ further changes
neither the ranking nor the selected block.
\begin{figure*}[h]
  \centering
  \includegraphics[width=\linewidth]{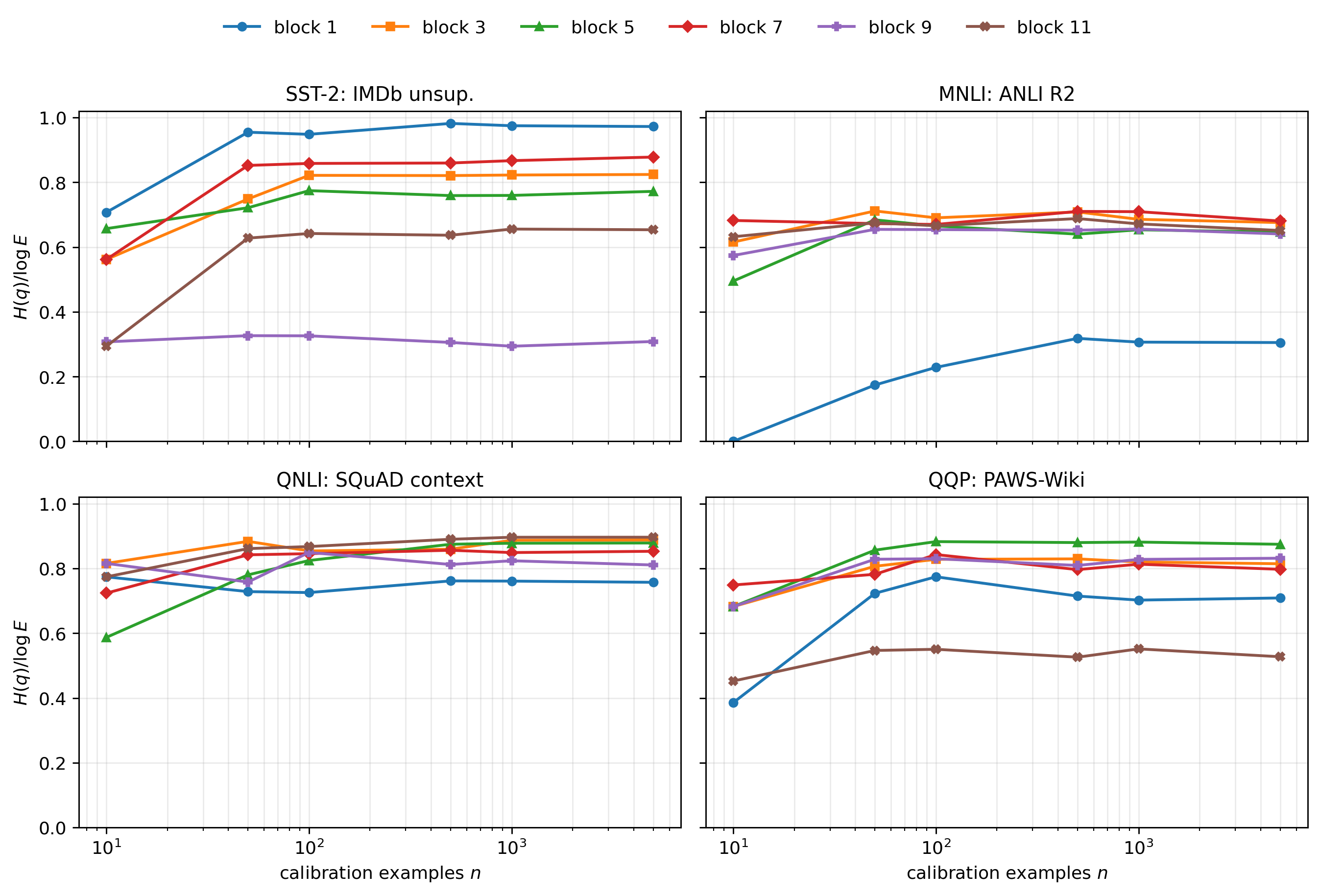}
  \caption{Normalised routing entropy $H(q^{(l)})/\log E$ as a
  function of calibration examples $n$ (log scale) for all six
  sparse encoder layers of \textsc{Switch-base-8}, across four tasks.
  Each line is one candidate layer; the entropy-selected layer at
  the full corpus size is marked in bold in
  Tab.~\ref{tab:switch_public_calibration_all_sparse_layers}.
  Entropy estimates stabilise by $n \approx 100$--$500$ for most
  layers; the main exception is MNLI block~1 (blue), which is
  heavily collapsed and grows slowly even at large $n$.
  The frozen router is applied to unlabelled inputs only;
  no private data are accessed.}
  \label{fig:entropy_vs_n}
\end{figure*}

\begin{figure}[h]
  \centering
  \includegraphics[width=0.9\linewidth]{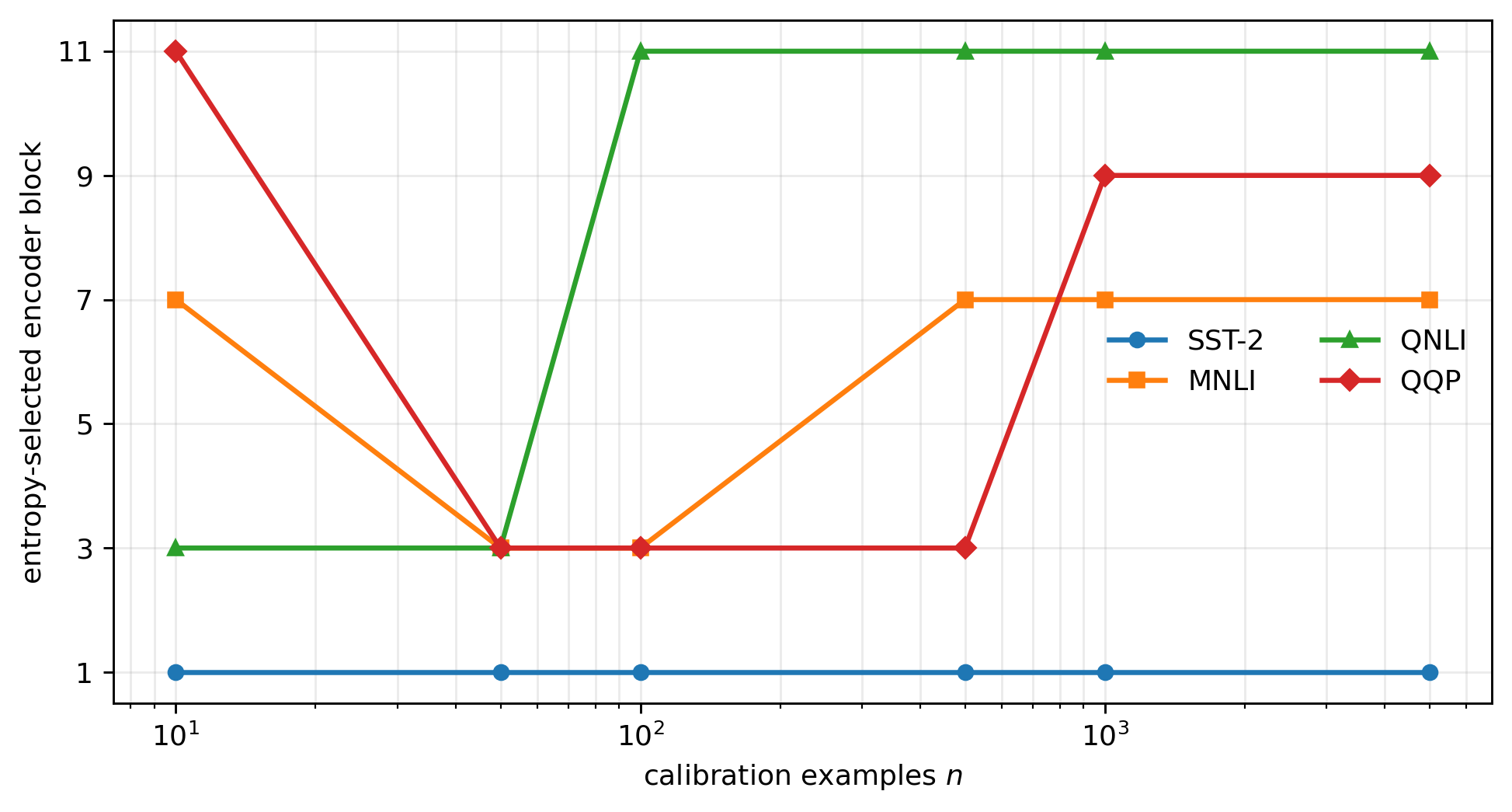}
  \caption{Public-calibration vs.\ GLUE-development normalised routing
entropy $H(q)/\log E$ for all 24 layer-task combinations (six sparse
encoder layers $\times$ four tasks). Each point is labelled with its
encoder block index. The dashed line is the diagonal $y=x$. Transfer
is closest for QNLI; MNLI shows a systematic gap, with all six layers
above the diagonal. The largest deviations are MNLI block~1 and SST-2
blocks~7 and~9, none of which is selected by the entropy criterion;
entropy-selected layers sit closer to the diagonal than collapsed,
non-selected ones.}
  \label{fig:best_block_vs_n}
\end{figure}
\paragraph{Public-to-private entropy transfer.}
The entropy-based layer-selection rule (Sec.~\ref{sec:public-layer-selection})
relies on routing statistics computed on a public corpus $\mathcal{C}_{\mathrm{pub}}$
to proxy the routing balance on the private training set $\mathcal{D}_{\mathrm{priv}}$.
Fig.~\ref{fig:entropy-scatter} checks this transfer empirically by
plotting, for each of the 24 layer-task combinations, the normalised
routing entropy $H(q)/\log E$ estimated on $\mathcal{C}_{\mathrm{pub}}$
against the corresponding entropy estimated on the GLUE development
set. Transfer is close for QNLI and, more loosely, for SST-2 and QQP;
MNLI is the exception, with all six layers sitting visibly above the
diagonal (dev-set entropy exceeding public-corpus entropy throughout),
a systematic gap rather than an isolated outlier. The largest single
deviation is MNLI block~1 (low public entropy but high dev-set
entropy), followed by SST-2 blocks~9 and~7; none of these three is
selected by $\arg\max_l H(q^{(l)})$ for its task. Entropy-selected
layers sit closer to the diagonal than collapsed, non-selected ones
across all four tasks - e.g.\ SST-2's selected block~1
($H/\log E = 0.972$ publicly) deviates far less than its own
block~9. On SST-2, QNLI, and QQP, the public-corpus argmax also
matches the dev-set ranking among top layers. On MNLI, blocks~3 and~7
are separated by only $0.008$ nats on $\mathcal{C}_{\mathrm{pub}}$
(Tab.~\ref{tab:switch_public_calibration_all_sparse_layers}), a margin narrow enough that the
dev-set ranking between these two layers may not agree with the
public-corpus one; Tab.~\ref{tab:switch_main_results} shows the
publicly-selected block~7 nonetheless yields the best downstream
accuracy on MNLI. We read this as evidence that public-corpus entropy
is a directionally reliable but imperfect proxy for private routing
balance, consistent with Remark~\ref{rem:adam-invariance}'s broader
caution that entropy is a validated heuristic rather than a certified
optimum.

\begin{figure}[h]
\centering
\includegraphics[width=0.6\textwidth]{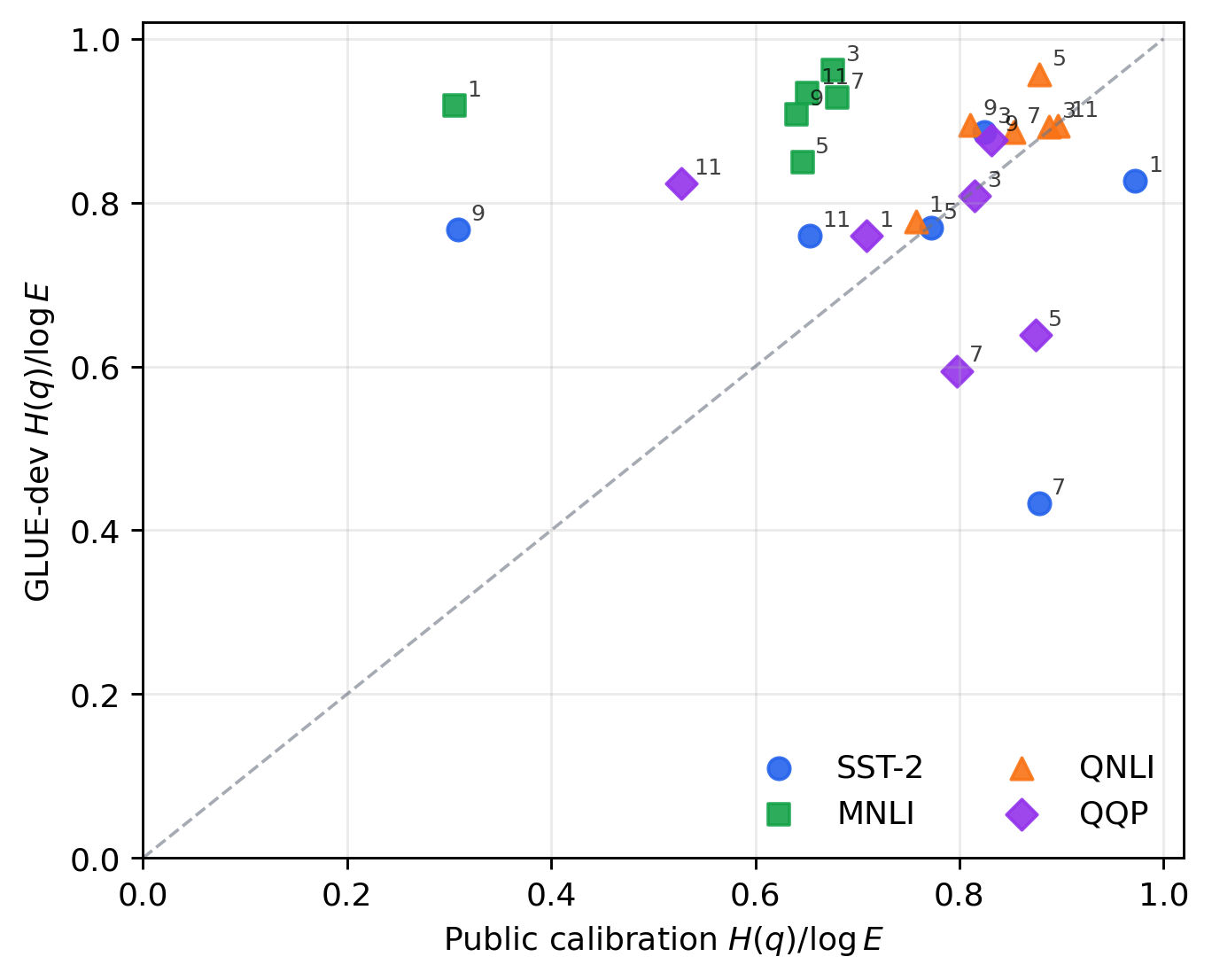}
\caption{Public-calibration vs.\ GLUE-development normalised routing entropy
  $H(q)/\log E$ for all 24 layer--task combinations (six sparse encoder
  layers $\times$ four tasks). Each point is labelled with its encoder
  block index. The dashed line is the diagonal $y=x$. Points near the
  diagonal indicate that entropy estimated on the public corpus transfers
  to the private training distribution. The two largest off-diagonal
  deviations (SST-2 blocks~7 and~9, blue circles) are collapsed layers
  not selected by the entropy criterion; the entropy-selected layer for
  each task lies near or above the diagonal.}
\label{fig:entropy-scatter}
\end{figure}

\onecolumn
\section{Mechanism Diagnostics}
\label{app:diagnostics}

The structural failures diagnosed in Sec.~\ref{sec:diagnosis} follow
from the algebraic form of monolithic private optimization: a single
clipping factor couples dense and sparse roles, the full-batch denominator
dilutes expert updates, and fixed privacy noise interacts unfavorably with
routing imbalance. These failures are properties of the mechanism definition
and do not require empirical validation to exist. Empirical diagnostics are
nevertheless useful for measuring how strongly they manifest under the
actual gradient distribution, routing frequencies, and noise realizations
of a trained model.

Tab.~\ref{tab:rq2_diagnostics_tasks} reports three such diagnostics at
$\varepsilon\approx 8$ across GLUE tasks. These results should be read as
mechanism-level evidence that is directionally informative but not
statistically powered: diagnostic runs were completed for a limited number
of configurations per task, and the table is intended to corroborate the
algebraic arguments of Sec.~\ref{sec:diagnosis} rather than to
constitute a statistically controlled comparison.

\begin{table*}[h]
\centering
\small
\begin{tabular}{llccc}
\toprule
Task & Method
& $M_{\mathrm{role}}\downarrow$
& $\overline{\mathrm{SNR}}\uparrow$
& $\mathrm{SNR}\text{-}\mathrm{CV}\downarrow$ \\
\midrule

\multirow{3}{*}{SST-2}
& Monolithic DP-Adam LoRA & $0.320$ & $1.08{\times}10^{-4}$ & $2.35$ \\
& Monolithic DP-Adam      & $0.591$ & $6.38{\times}10^{-6}$ & $1.20$ \\
& \cellcolor{custom_light_purple_2}\Raptor
& \cellcolor{custom_light_purple_2}$0.401$
& \cellcolor{custom_light_purple_2}$1.80{\times}10^{-5}$
& \cellcolor{custom_light_purple_2}$1.69$ \\

\midrule

\multirow{3}{*}{MNLI}
& Monolithic DP-Adam LoRA & $0.494$ & $9.65{\times}10^{-5}$ & $3.54$ \\
& Monolithic DP-Adam      & $0.560$ & $3.04{\times}10^{-5}$ & $1.24$ \\
& \cellcolor{custom_light_purple_2}\Raptor
& \cellcolor{custom_light_purple_2}$0.337$
& \cellcolor{custom_light_purple_2}$2.25{\times}10^{-4}$
& \cellcolor{custom_light_purple_2}$1.86$ \\

\midrule

\multirow{3}{*}{QNLI}
& Monolithic DP-Adam LoRA & $0.723$ & $2.64{\times}10^{-5}$ & $1.69$ \\
& Monolithic DP-Adam      & $0.619$ & $9.58{\times}10^{-6}$ & $1.17$ \\
& \cellcolor{custom_light_purple_2}\Raptor
& \cellcolor{custom_light_purple_2}$0.397$
& \cellcolor{custom_light_purple_2}$7.19{\times}10^{-5}$
& \cellcolor{custom_light_purple_2}$1.62$ \\

\midrule

\multirow{3}{*}{QQP}
& Monolithic DP-Adam LoRA & $0.701$ & $5.98{\times}10^{-5}$ & $2.03$ \\
& Monolithic DP-Adam      & $0.573$ & $1.81{\times}10^{-5}$ & $1.31$ \\
& \cellcolor{custom_light_purple_2}\Raptor
& \cellcolor{custom_light_purple_2}$0.324$
& \cellcolor{custom_light_purple_2}$1.14{\times}10^{-4}$
& \cellcolor{custom_light_purple_2}$2.53$ \\

\bottomrule
\end{tabular}%
\caption{Mechanism diagnostics on \textsc{Switch-base-8} at $\varepsilon\approx 8$.
$M_{\mathrm{role}}$: total variation distance between gradient-norm share
and parameter-count share across roles (lower = more proportionate signal
allocation). $\overline{\mathrm{SNR}}$: mean per-expert signal-to-noise
ratio (higher = stronger expert update signal). $\mathrm{SNR}$-CV:
coefficient of variation of per-expert SNR (lower = more uniform update
quality across experts).}
\label{tab:rq2_diagnostics_tasks}
\end{table*}

\paragraph{Role mismatch.}
For each top-level role $r$ (shared parameters, router, classifier, and
experts), let $s_r^{\mathrm{norm}}$ be the share of gradient norm assigned
to that role and $s_r^{\mathrm{param}}$ be its share of trainable
parameters. We define
\begin{equation}
M_{\mathrm{role}}=
\frac{1}{2}
\sum_r
\bigl|s_r^{\mathrm{norm}}-s_r^{\mathrm{param}}\bigr|.
\end{equation}
Lower values indicate that gradient signal is more proportionate to the
trainable parameter mass across roles. Across all tasks with available
Monolithic DP-Adam diagnostics, \Raptor\,  substantially reduces
$M_{\mathrm{role}}$ relative to Monolithic DP-Adam: from $0.591$ to $0.401$
on SST-2, from $0.560$ to $0.337$ on MNLI, and from $0.619$ to $0.397$
on QNLI. Relative to Monolithic DP-Adam LoRA, role mismatch is reduced on
MNLI, QNLI, and QQP, though not on SST-2. The role-aware decomposition
therefore does not uniformly minimize this diagnostic, but it consistently
avoids the pronounced role mismatch observed under full monolithic
DP-Adam.

\paragraph{Mean expert SNR.}
The column $\overline{\mathrm{SNR}}$ measures the average per-expert
signal-to-noise ratio, capturing the absolute strength of the expert
update signal after accounting for privacy noise. On MNLI, QNLI, and QQP,
\Raptor\, achieves the largest mean expert SNR among the available
configurations: $2.25{\times}10^{-4}$ on MNLI, $7.19{\times}10^{-5}$ on
QNLI, and $1.14{\times}10^{-4}$ on QQP. On SST-2, Monolithic DP-Adam LoRA
records the highest mean SNR, while our framework still improves over
full Monolithic DP-Adam. This pattern is consistent with the intended effect
of role-aware expert updates: the method is most beneficial when monolithic
optimization gives experts weak or diluted signal, whereas simpler tasks
such as SST-2 can already be served adequately by a strong LoRA baseline.

\paragraph{Expert SNR variation.}
The SNR-CV column reports the coefficient of variation of per-expert SNR;
lower values indicate more uniform update quality across experts. \Raptor\, reduces SNR-CV relative to Monolithic DP-Adam LoRA on SST-2, MNLI,
and QNLI, but does not always achieve the lowest SNR-CV overall: full
Monolithic DP-Adam has lower SNR-CV on SST-2, MNLI, and QNLI, and
Monolithic DP-Adam LoRA has lower SNR-CV than our framework on QQP.
This does not contradict the utility results: a low SNR-CV can arise from
uniformly \emph{weak} expert updates rather than uniformly useful ones. Full
Monolithic DP-Adam exhibits low SNR-CV on SST-2 and QNLI precisely because
its mean expert SNR is far lower than that of our framework. We
therefore interpret SNR-CV jointly with $\overline{\mathrm{SNR}}$:
role-aware training primarily improves absolute expert signal strength and
role proportionality, while only partially reducing cross-expert SNR
variation.

\paragraph{Summary.}
Tab.~\ref{tab:rq2_diagnostics_tasks} provides directional support for the
mechanism account developed in Sec.~\ref{sec:diagnosis}. \Raptor\, is not uniformly best on every diagnostic, and the limited number
of completed runs means that the reported values should be treated as
indicative rather than conclusive. Nevertheless, the method consistently
reduces the large role mismatch of full monolithic DP-Adam and substantially
increases mean expert SNR on the harder GLUE tasks. These trends are
directionally consistent with the downstream accuracy gains in
Tab.~\ref{tab:switch_main_results}: the advantage of role-aware private
training stems less from eliminating all expert imbalance and more from
preventing expert updates from being suppressed or diluted by a dense
monolithic private optimizer.

\paragraph{Corpus-Level Routing Diagnostics}

Replacing the training batch $\mathcal{B}_t^{\mathrm{exp}}$ with the full calibration
corpus $\mathcal{C}_{\mathrm{pub}}$ (Sec.~\ref{sec:public-layer-selection})
gives corpus-level estimates $\widehat{\mathrm{CV}}_l$ and
$\widehat{H}_l$, computed once from the frozen router before any private
training begins.
Tab.~\ref{tab:routing_load_diagnostics} reports these for the
entropy-selected and most-collapsed layers across all four tasks.
In every case the selected layer has lower $\widehat{\mathrm{CV}}_l$
and higher $\widehat{H}_l$ than the collapsed layer, confirming that
the two statistics are consistent and that high routing entropy
coincides with more uniform load.

\begin{table}[h]
\centering
\small
\begin{tabular}{lllcc}
\toprule
Task & Public corpus & Layer (block)
  & $\widehat{\mathrm{CV}}_{l}$
  & $\widehat{H}_{l}/\log E$ \\
\midrule
SST-2 & IMDb unsup.
  & selected (1)    & 0.344 & 0.972 \\
  &
  & collapsed (9)   & 2.185 & 0.304 \\
\midrule
MNLI  & ANLI R2
  & selected (7)    & 1.123 & 0.682 \\
  &
  & collapsed (1)   & 2.224 & 0.307 \\
\midrule
QNLI  & SQuAD context
  & selected (11)   & 0.627 & 0.902 \\
  &
  & collapsed (1)   & 1.121 & 0.761 \\
\midrule
QQP   & PAWS-Wiki
  & selected (9)    & 0.700 & 0.837 \\
  &
  & collapsed (7)  & 1.763 & 0.498 \\
\bottomrule
\end{tabular}
\caption{Corpus-level routing diagnostics for the entropy-selected and
  most-collapsed sparse layers of \textsc{Switch-base-8}, computed on external
  public calibration corpora with labels ignored
  (see Appendix~\ref{app:entropy-analysis} for corpus details).
  Higher $\widehat{\mathrm{CV}}_{l}$ and lower $\widehat{H}_{l}$
  indicate stronger routing imbalance; the two statistics are
  consistent across all tasks.
  Values are identical across private training methods because the
  router is frozen.}
\label{tab:routing_load_diagnostics}
\end{table}

\onecolumn
\section{Budget Allocation: Surrogate Analysis and Empirical Sensitivity}
\label{app:budget-allocation}

\paragraph{Surrogate analysis.}
The bias-variance decomposition of Theorem~\ref{thm:bias-var} shows
that privacy noise enters the shared and expert streams through separate
variance terms, $V^{\mathrm{DP}}_s \propto C_s^2/(\rho\varepsilon)^2$ and
$V^{\mathrm{DP}}_e \propto E^2C_e^2/((1-\rho)\varepsilon)^2$, whose
relative magnitudes depend on the budget split $\rho$. To obtain a
tractable criterion for choosing $\rho$, we weight these terms by the
downstream sensitivity of the task loss to each stream.

\begin{proposition}[Budget allocation under shared-dominant routing]
\label{prop:shared-dominant-rho}
Under balanced routing ($q_e = 1/E$), define the single-step surrogate
\begin{equation}
\mathcal{L}(\rho)
= \frac{\lambda_s C_s^2}{\rho^2}
+ \frac{\lambda_e E^2 C_e^2}{(1-\rho)^2},
\qquad \rho\in(0,1),
\label{eq:shared-dominant-loss}
\end{equation}
where $\lambda_s, \lambda_e > 0$ weight the task loss sensitivity to
shared- and expert-stream perturbations. The unique minimizer is
\begin{equation}
\rho^* = \frac{1}{1 + \left(
  \dfrac{\lambda_e E^2 C_e^2}{\lambda_s C_s^2}
\right)^{1/3}}.
\label{eq:rho-star}
\end{equation}
\end{proposition}

\begin{proof}
Setting $A = \lambda_s C_s^2$ and $B_c = \lambda_e E^2 C_e^2$,
differentiating $\mathcal{L}(\rho) = A/\rho^2 + B_c/(1-\rho)^2$ and
equating to zero gives $((1-\rho)/\rho)^3 = B_c/A$, yielding
Eq.~\eqref{eq:rho-star}. Strict convexity ($d^2\mathcal{L}/d\rho^2 =
6A/\rho^4 + 6B_c/(1-\rho)^4 > 0$) confirms uniqueness.
\end{proof}

\paragraph{Interpretation.}
Eq.~\eqref{eq:rho-star} gives two qualitative directions. First,
$\rho^*$ increases with $\lambda_s/\lambda_e$: higher shared-stream
sensitivity concentrates the budget on the shared stream. Second,
$\rho^*$ \emph{decreases} with $E$: the $E^2$ amplification of
expert-stream noise makes expert-side budget more valuable per unit
of noise reduction under the surrogate, pulling $\rho^*$ toward $0$
as $E$ grows. The shared-dominant regime $\rho^*\!\to\!1$ we adopt
therefore does not come from $E$ - it comes from the sensitivity
asymmetry: the empirical estimate
$\lambda_e/\lambda_s \lesssim 6.3\times10^{-3}$
(Tab.~\ref{tab:noise-sensitivity}) shows the task loss is
far more sensitive to shared-stream perturbations, which must
outweigh the $E^2$ factor for $\rho^*$ to sit near $1$.
Proposition~\ref{prop:shared-dominant-rho} provides qualitative
guidance, not a certified value of $\rho^*$: the sensitivity ratio
$\lambda_e/\lambda_s$ required to formally conclude $\rho^*\ge 0.9$
(Eq.~(41)) is stringent and, per Tab.~\ref{tab:noise-sensitivity},
does not hold for this model. We therefore treat $\rho=0.9$ as an
empirically validated shared-dominant choice
(Sec.~\ref{sec:experiments}), not a certified optimum.
\paragraph{Empirical sensitivity estimate.}
For each role $r \in \{\text{shared},\,\text{experts}\}$, we inject
Gaussian noise $\xi_r \sim \mathcal{N}(0,\,\sigma_r^2 I)$ matched to
the DP noise magnitude at $\varepsilon \approx 8$, apply it to $\theta_r$,
and record $\Delta\mathcal{L}_{\mathrm{val}}$. We define the one-sided
sensitivity estimate
\begin{equation}
\lambda_+ =
\frac{\max(\Delta\mathcal{L}_{\mathrm{val}},\;0)}
     {\mathbb{E}\|\xi_r\|_2^2},
\label{eq:lambda-plus}
\end{equation}
flooring non-positive values at zero: a negative or near-zero
$\Delta\mathcal{L}_{\mathrm{val}}$ indicates that the role lies below
the measurement noise floor.

\begin{table}[h]
\centering
\small
\begin{tabular}{lrrr}
\toprule
Role & Params & $\Delta\mathcal{L}_{\mathrm{val}}$ & $\lambda_+$ \\
\midrule
Shared  & $9{,}438{,}722$ & $\phantom{-}2.175\times10^{-4}$
        & $2.305\times10^{-5}$ \\
Experts & $7{,}864{,}320$ & $-1.361\times10^{-6}$
        & $0$ \\
\bottomrule
\end{tabular}
\caption{Noise-injection sensitivity on SST-2 at $\varepsilon \approx 8$.
  $\lambda_+$ floors non-positive $\Delta\mathcal{L}_{\mathrm{val}}$ at
  zero (Eq.~\eqref{eq:lambda-plus}). A non-positive expert value
  indicates sensitivity below the measurement noise floor.}
\label{tab:noise-sensitivity}
\end{table}

Tab.~\ref{tab:noise-sensitivity} shows clear shared-stream sensitivity
($\lambda_s = 2.305\times10^{-5}$) and negligible expert sensitivity.
Using the magnitude of the signed expert estimate as a conservative upper
bound,
\begin{equation}
\frac{\lambda_e}{\lambda_s}
\;\lesssim\;
\frac{1.361\times10^{-6}}{2.175\times10^{-4}}
\;\approx\; 6.3\times10^{-3}.
\label{eq:lambda-ratio-bound}
\end{equation}
For $E = 8$ and $C_s = C_e$, Eq.~\eqref{eq:lambda-ratio-bound} requires
$\lambda_e/\lambda_s \le 2.1\times10^{-5}$. The empirical bound exceeds
this threshold, so Proposition~\ref{prop:shared-dominant-rho} does not
formally certify $\rho^* = 0.9$ for this model. We therefore treat
$\rho = 0.9$ as a shared-dominant initialization: motivated by the
surrogate and consistent with the near-zero measured expert sensitivity,
but validated primarily by the empirical sweep in
Sec.~\ref{sec:hparam}.

\onecolumn
\section{Runtime and Memory}
\label{app:cost}
Tab.~\ref{tab:efficiency} reports implementation-level runtime and
memory for fine-tuning \textsc{Switch-base-8} on SST-2 at
$\varepsilon\approx 8$, comparing \Raptor\,  against the two
Monolithic DP baselines from Tab.~\ref{tab:switch_main_results}. We
include this because Role-Aware changes training scope as well as the
DP mechanism, so accuracy gains alone don't show whether they come at
a compute cost. 
Global DP-Adam trains the full model ($307.85$M params, $100\%$);
Global DP-Adam LoRA trains only adapters ($14.16$M, $4.60\%$);
our framework trains the shared stream plus one expert layer ($17.30$M,
$5.62\%$) - comparable in scope to LoRA, $\sim\!18\times$ smaller than
the full model. Our framework's gains over Monolithic DP-Adam
(Tab.~\ref{tab:switch_main_results}) are thus not explained by
training more of the model.

Peak memory is lowest for \Raptor\, ($112.16$GB), below both LoRA
($130.74$GB) and the full model ($165.51$GB), plausibly because
expert-stream clipping operates on owner subsets $\mathcal{B}_{t,l,e}$
rather than the full batch at once. Per-step time
and total GPU-hours are highest for the full model ($7.23$s,
$2.61$ GPU-hrs) and lowest for LoRA ($4.07$s, $1.47$); our framework sits
close to LoRA ($4.51$s, $1.65$ GPU-hrs) despite the extra forward pass
for the residual weight $w_i$ and the alternating schedule
(Sec.~\ref{sec:alternating}) - about $37\%$ cheaper in total
compute than the full-model baseline, at higher accuracy.

\begin{table}[h]
\centering
\small
\begin{tabular}{lccccc}
\toprule
Method & Trainable params & \% params & Peak GB & Sec./step & GPU-hours \\
\midrule
Global DP-Adam LoRA
& $14.16$M & $4.60\%$ & $130.74$ & $4.07$ & $1.47$ \\
Global DP-Adam
& $307.85$M & $100.00\%$ & $165.51$ & $7.23$ & $2.61$ \\
\cellcolor{custom_light_purple_2}\Raptor
& \cellcolor{custom_light_purple_2}$17.30$M
& \cellcolor{custom_light_purple_2}$5.62\%$
& \cellcolor{custom_light_purple_2}$\mathbf{112.16}$
& \cellcolor{custom_light_purple_2}$4.51$
& \cellcolor{custom_light_purple_2}$1.65$ \\
\bottomrule
\end{tabular}
\caption{Implementation footprint of fine-tuning \textsc{Switch-base-8}
on SST-2 at $\varepsilon\approx 8$. Raw implementation-level costs, not
a controlled same-scope benchmark---scope differs by design.
Percentages relative to the full trainable \textsc{Switch-base-8} model used by
Global DP-Adam. Bold: lowest peak memory.}
\label{tab:efficiency}
\end{table}


\end{document}